\documentclass{article}

    \PassOptionsToPackage{numbers}{natbib}
 \usepackage[preprint]{neurips_2026}

\usepackage[utf8]{inputenc} 
\usepackage[T1]{fontenc}    
\usepackage{xcolor}
\definecolor{mydarkblue}{rgb}{0,0.08,0.45}
\usepackage[
    colorlinks=true,
    linkcolor=mydarkblue,
    citecolor=mydarkblue,
    urlcolor=mydarkblue
]{hyperref}
\usepackage{url}            
\usepackage{booktabs}       
\usepackage{nicefrac}       
\usepackage{microtype}      
\usepackage{graphicx}
\usepackage{subcaption}

\usepackage{amsmath}
\usepackage{amssymb}
\usepackage{amsfonts}       
\usepackage{mathtools}
\usepackage{amsthm}

\usepackage{algorithm}
\usepackage{algorithmic}

\usepackage{todonotes}

\DeclareMathOperator*{\argmin}{arg\,min}
\DeclareMathOperator*{\argmax}{arg\,max}
\newcommand{\EI}{\operatorname{EI}}

\usepackage[capitalize,noabbrev]{cleveref}

\theoremstyle{plain}
\newtheorem{theorem}{Theorem}[section]
\newtheorem{proposition}[theorem]{Proposition}

\theoremstyle{definition}

\theoremstyle{remark}
\newtheorem{remark}[theorem]{Remark}

\renewcommand{\cite}[1]{\citep{#1}}

\usepackage{wrapfig}

\title{LAGO: A Local-Global Optimization Framework Combining Trust Region Methods and Bayesian Optimization}

\author{
  Eliott Van Dieren \\
  Institute of Mathematics\\
  EPFL\\
  Lausanne, Switzerland\\
  \texttt{eliott.vandieren@epfl.ch} \\
  \And
  Tommaso Vanzan\\
  Dipartimento di Scienze Matematiche\\
  Politecnico di Torino\\
  Turin, Italy\\
  \texttt{tommaso.vanzan@polito.it} \\
  \AND
  Fabio Nobile\\
  Institute of Mathematics\\
  EPFL\\
  Lausanne, Switzerland\\
  \texttt{fabio.nobile@epfl.ch} \\
}

\begin{document}

\maketitle

\begin{abstract}
We introduce LAGO, a \textbf{L}oc\textbf{A}l-\textbf{G}lobal
\textbf{O}ptimization framework coupling Bayesian
Optimization (BO) and gradient-based trust region local refinement through an
adaptive competition mechanism for smooth expensive-to-evaluate objective functions with available gradients.
At each iteration, global and local optimization strategies independently propose candidate points, and the next
evaluation is selected based on predicted improvement.
LAGO separates global exploration from local refinement at the proposal level: the BO acquisition
function is optimized outside the active trust region, while local candidates are proposed within the trust region.
Points in the vicinity of the accepted local step are incorporated in the global GP dataset only when satisfying a lengthscale-based minimum-distance
criterion, hence reducing the risk of numerical instability during local exploitation.
LAGO enhances BO with efficient local refinement when reaching promising
regions, and reverts to exploratory behavior when local steps are not competitive.
\end{abstract}

\section{Introduction}

\subsection{Problem setting and motivations}

We consider optimization problems in which evaluating the objective is computationally expensive,
for instance, in engineering design \cite{mdobook} when using PDE-based models, shape optimization in fluid- and aero-dynamics when using 
computational fluid dynamics models \cite{ferziger2002computational}, and more generally in PDE-constrained optimization \cite{hinze2008optimization}. 
In these contexts, accurate (noise-free) gradients are often available via
automatic differentiation \cite{baydin2018automatic} or adjoint-based methods \cite{plessix2006review}, at a cost comparable to, if not lower than, that of a single function evaluation.

A key challenge in this setting is whether to pursue a global or a local
optimization strategy. On the one hand, Bayesian Optimization (BO) \cite{garnett2023bayesian} is well known for its strong performance under limited
evaluation budgets. Given a compact set $\mathcal{X} \subseteq \mathbb{R}^d$, $d\in \mathbb{N}$, and a continuous function $f : \mathcal{X} \to \mathbb{R}$, BO addresses the problem of finding
\begin{equation*}
    \mathbf{x}^\star \in \argmin_{\mathbf{x} \in \mathcal{X}} f(\mathbf{x}),
\end{equation*}
by placing a Gaussian Process (GP) \cite{rasmussen2005} prior over $f$.
Posterior inference is used to improve the knowledge of $f$ and to guide
the selection of the next evaluation of $f$ through the maximization of an
acquisition function, which is inexpensive to
optimize over the design space relatively to $f$ \cite{frazier2018tutorial}. 

While BO is effective at identifying promising regions, it is less suited for
local refinement. Indeed, acquisition functions often induce clustering of evaluation
points \cite{qin2017improving}, which can lead to near-singular kernel matrices
and numerical instability in GP inference \cite{wendland2004scattered,
wynne2021convergence}. In practice, the ill-conditioning is mitigated by adding a regularization term
(nugget) to the diagonal of the kernel matrix.
This regularization improves numerical stability, but is equivalent to
introducing additional observation noise, thereby modifying the GP posterior
\cite{rasmussen2005} and potentially affecting the sequential selection of
candidates.
Additionally, it limits the accuracy of the recovered minimum due to the induced regularization floor \cite{wynne2021convergence}.
Beyond numerical issues, global stationary GPs are inherently limited for local
refinement, as they must trade off global fit against local fidelity \cite{rasmussen2005}.
In particular, resolving ill-conditioned minima typically requires
either dense sampling or localized modeling \cite{eriksson2019scalable, visser2025labcat}.
As a result, the GP surrogate is generally not accurate enough for
local optimization with tight tolerances.

On the other hand, local gradient-based methods \cite{nocedal2006numerical}, such as trust region, gradient
descent, or quasi-Newton methods, efficiently
exploit local information and often exhibit fast convergence. However, they are
sensitive to initialization and lack guarantees of global optimality in
non-convex settings. While multiple restarts can partially address this issue,
they usually do not leverage information across runs, leading to poor sample efficiency
when evaluations are expensive.

Hybrid strategies combining global exploration and local refinement have been widely studied \cite{brunzema2025bayesqp, diouane2023trego, eriksson2019scalable, jones1993lipschitzian, mcleod2018optimization, regis2007improved}.
Most approaches alternate between global and local phases based on predefined
criteria, such as sufficient decrease or acquisition thresholds. For instance,
\citet{mcleod2018optimization} use BO to minimize the overall regret, followed
by a local refinement phase triggered by an acquisition-based rule.
However, such sequential design introduces an allocation problem: how to
balance exploration and refinement without prior knowledge of the objective.
Misspecified switching rules can lead to premature exploitation or conversely
insufficient refinement, hence degrading performance.

We address this limitation by proposing a framework that uses an adaptive
decision rule to balance global and local optimization steps, allowing
exploration and refinement to interact continuously without predefined
switching heuristics.

\subsection{Contributions, scope and limitations}\label{sec:contributions}
We propose \emph{LAGO}, a unified framework that couples BO with gradient-based trust region local
refinement through an adaptive decision rule.
The method addresses the allocation problem between global exploration and local
optimization by adaptively interleaving both regimes.
The proposed framework is characterized by the following key components and properties:
\begin{enumerate}
    \item[(C$_1$)] \hypertarget{C_1}{} \textbf{Adaptive selection mechanism:} We introduce a decision
            rule that dynamically selects between global exploration
            and local refinement based on predicted
            improvements in each regime. This enables adaptive allocation of
            evaluations without predefined switching rules. 
    \item[(C$_2$)] \hypertarget{C_2}{} \textbf{Decoupled local-global structure:} LAGO enforces a strict
            separation between global search and local refinement. The BO
            component proposes candidates exclusively outside the active trust
            region, while the local step, acceptance rule, radius updates, and
            Hessian updates coincide with those of the underlying
            trust region method.
            Global evaluations can only affect the local routine by relocating
            the trust region center when they improve upon the current best
            point.
            Consequently, conditional on global recentering ceasing, LAGO
            reduces to a standard trust region method and inherits its
            local convergence guarantees (\Cref{prop:local_convergence}).
    \item[(C$_3$)] \hypertarget{C_3}{} \textbf{Local search ill-conditioning mitigation:} 
        During dense local refinement, the GP surrogate selectively incorporates points far enough
        from the trust region center to avoid near-duplicate samples that would
        otherwise induce ill-conditioning. This preserves
        numerical stability while retaining sufficient local function knowledge to inform
        global exploration. 
    \item[(C$_4$)] \hypertarget{C_4}{} \textbf{Adaptive robustness to problem structure:}
        LAGO naturally adjusts its behavior to the optimization landscape. In regimes
        where global exploration is favorable, such as highly multimodal problems,
        it effectively suppresses local refinement and reverts to BO-like behavior with
        negligible overhead. Conversely, in weakly multimodal or convex regimes, LAGO actively
        engages the local component, exploiting gradient and curvature information to
        rapidly refine solutions and recover fast local convergence.
\end{enumerate}
These properties are evaluated through Research Questions
(RQs) discussed in \Cref{sec:experiments}, while a comparison to related works on hybrid methods is 
discussed in \Cref{sec:related_works}.

\textbf{Scope and limitations.} 
Standard BO methods degrade in high dimensions due to approximation challenges
unless additional functional and model assumptions (e.g., lengthscale tuning,
low intrinsic dimensionality or additive structure) are imposed \cite{binois2022survey,
pmlr-v267-papenmeier25a}.
Indeed, accurately training the surrogate becomes increasingly
difficult due to the curse of dimensionality \cite{bellman1966dynamic}, leading to poor approximations of
the objective and hence unreliable acquisition functions.
In this work, we restrict to low-dimensional problems ($d \le 10$) where BO remains tractable \cite{binois2022survey} and we can clearly illustrate the behavior of the proposed method.
Applications to high-dimensional problems with structural assumptions are currently under investigation and will be reported in a future work.

We restrict attention to deterministic (noise-free) objectives, as commonly
encountered in simulation-based optimization, where repeated evaluations are
identical and the dominant error stems from misspecification of the underlying
simulation model (e.g., a PDE) rather than stochastic noise. In this regime,
LAGO operates naturally and its accurate local refinement is meaningful.
However, the achievable accuracy remains limited by simulation model misspecification, so high-precision local
solves should be interpreted relative to this model rather than the true physical system.
Extending LAGO to handle stochastic noise is left for
future work, e.g., using noise-tolerant trust regions \cite{sun2023trust}.

\section{Background}\label{sec:background}

This section describes the two main components used in our algorithm: BO \cite{jones1998efficient} and the Symmetric Rank-One (SR1) trust region method \cite{nocedal2006numerical}.

\textbf{Bayesian Optimization.} 
In BO, we usually model the objective function using a Gaussian process \cite{rasmussen2005}. 
This model is attractive due to the availability of closed-form formulas for conditioning.
Indeed, conditioning a GP $\mathcal{GP}(m,k)$, $m : \mathcal{X} \to \mathbb{R}$
and $k : \mathcal{X} \times \mathcal{X} \to \mathbb{R}$ being the prior mean
and covariance respectively, on (possibly) noisy observations
$\{(\mathbf{x}_i,y_i)\}_{i=1}^n$, with $y_i=f(\mathbf{x}_i)+\xi_i$ and $\xi_i
\sim \mathcal{N}(0,\sigma^2)$, yields a posterior GP with mean and covariance
\begin{equation*}
    \begin{split}
    &\bar{m}(\mathbf{x}) := m(\mathbf{x}) + k_{\mathbf{x},X}(K + \sigma^2 I_n)^{-1}(\mathbf{y} - m_{X}),\\ 
    &\bar{k}(\mathbf{x},\mathbf{x}') := k(\mathbf{x},\mathbf{x}') -  k_{\mathbf{x},X}(K + \sigma^2 I_n)^{-1}k_{X,\mathbf{x}'},
    \end{split}
\end{equation*}
where $m_{X}=[m(\mathbf{x}_1),\dots,m(\mathbf{x}_n)]^\top$, $K=(k(\mathbf{x}_i,\mathbf{x}_j))_{ij}$, $k_{\mathbf{x},X} = [k(\mathbf{x},\mathbf{x}_1),...,k(\mathbf{x},\mathbf{x}_n)]$ and $\mathbf{y} = [y_1,...,y_n]^\top$.
In noise-free settings, such as our case, $\sigma^2$ is often replaced by a regularization term $\lambda$ for numerical stability of the kernel matrix inversion.

At each iteration of BO, an acquisition function is maximized, and guides the sequential sampling decision.
This work focuses on the Expected Improvement (EI) \cite{jones1998efficient}, and we refer to \citet{frazier2018tutorial} for an overview of alternative choices. 
The EI is a classical acquisition function that quantifies the expected decrease in value of the objective function if we query it at $\mathbf{x} \in \mathcal{X}$. 
Mathematically, it reads
\begin{equation*}
    \EI(\mathbf{x}) := \mathbb{E} \left[ (f_{\text{best}} - f(\mathbf{x}))_+\right],
\end{equation*}
with $f_{\text{best}}$ being the lowest value of $f$ found so far, and $(\cdot)_+ := \max(\cdot,0)$. The closed form expression using the conditional Gaussian process distribution of $f \sim \mathcal{GP}(\bar{m},\bar{k})$ is 
\begin{equation}\label{eq:ei}
    \EI(\mathbf{x}) = z(\mathbf{x})\Phi\left(\frac{z(\mathbf{x})}{\bar{\sigma}(\mathbf{x})}\right) + \bar{\sigma}(\mathbf{x}) \phi \left(\frac{z(\mathbf{x})}{\bar{\sigma}(\mathbf{x})}\right),
\end{equation}
with $z(\mathbf{x}) = f_{\text{best}}-\bar{m}(\mathbf{x})$,
$\bar{\sigma}(\mathbf{x}) = \bar{k}(\mathbf{x},\mathbf{x})^{1/2}$, and $\Phi$
and $\phi$ the standard normal CDF and PDF, respectively. EI therefore solely
depends on the GP posterior mean and standard deviation, which are inexpensive
to evaluate.
As a result, the acquisition function can be optimized at negligible cost compared to the objective function if the latter is expensive to evaluate.

Lastly, the covariance function $k$ typically depends on hyperparameters (e.g.,
lengthscale $\ell$ and signal variance) that control the correlation and amplitude of
the model. These hyperparameters are estimated from data by maximizing the
marginal likelihood of the GP \cite{rasmussen2005}, and updated throughout BO
steps as new observations are collected.

\textbf{Trust region algorithm.} The local optimization algorithm used in LAGO consists of a quadratic surrogate model optimization inside a trust region \cite{nocedal2006numerical}. 
The algorithm places a region around the current best point, and trusts the local quadratic surrogate to approximate the objective function well within that region.
The local surrogate is optimized within this region, and the proposed candidate is then accepted or rejected. 
At iteration $k$, given the trust region center $\mathbf{x}_{k-1}$, the
method computes a trial step $\mathbf{s}_k$ and trial point
$\mathbf{x}^+_k=\mathbf{x}_{k-1}+\mathbf{s}_k$ solving the \emph{trust region subproblem}
\begin{equation}\label{eq:tr_subproblem}
    \mathbf{s}_k \in \argmin_{\|\mathbf{s}\|\le \Delta_{k-1}} m^q(\mathbf{s}),
\end{equation}
where $\Delta_{k-1}$ is the trust region radius, and $m^q : \mathbb{R}^d \to \mathbb{R}$ is the quadratic model
$$m^q(\mathbf{s}) := f(\mathbf{x}_{k-1}) + \nabla f(\mathbf{x}_{k-1})^\top \mathbf{s} + \frac{1}{2} \mathbf{s}^\top H_{k-1}\mathbf{s},$$
with $H_{k-1}$ the current approximated Hessian of $f$.
Theorem~$4.1$ in \citet{nocedal2006numerical} ensures that the trust region subproblem admits a global minimizer.
We solve \eqref{eq:tr_subproblem} using the standard iterative procedures for trust region subproblems described in \citet[Section~4.3]{nocedal2006numerical}.
To update $H_{k-1}$ throughout the iterations, we leverage the SR1 update method \citep{broyden1967quasi}
\begin{equation}\label{eq:update_SR1}
    H_{k} = H_{k-1} + \frac{(\mathbf{y}_k - H_{k-1} \mathbf{s}_k)(\mathbf{y}_k - H_{k-1} \mathbf{s}_k)^\top}{(\mathbf{y}_k-H_{k-1}\mathbf{s}_k)^\top \mathbf{s}_k},
\end{equation}
where $\mathbf{y}_k = \nabla f(\mathbf{x}^+_{k})-\nabla f(\mathbf{x}_{k-1})$.
This update is done only if
\begin{equation}\label{eq:condition_update_H}
|(\mathbf{y}_k- H_{k-1} \mathbf{s}_k)^\top\mathbf{s}_k| \ge r \|\mathbf{s}_k\|\|\mathbf{y}_k - H_{k-1} \mathbf{s}_k\|
\end{equation}
with $r \ll 1$, to avoid numerical instability in \eqref{eq:update_SR1}. The quality of the trial step is measured by the improvement ratio
\begin{equation}\label{eq:improvement_ratio}
\rho_k=\frac{f(\mathbf{x}_{k-1})-f(\mathbf{x}^+_k)}{f(\mathbf{x}_{k-1})-m^q(\mathbf{s}_k)},
\end{equation}
representing realized versus predicted function decrease.
The trial point $\mathbf{x}^+_k$ is accepted if $\rho_k$ exceeds a prescribed
threshold, in which case the trust region center is updated to $\mathbf{x}_k =
\mathbf{x}^+_k$; otherwise, the step is rejected and $\mathbf{x}_k =
\mathbf{x}_{k-1}$. The radius is adjusted accordingly, typically
increased for successful steps and decreased otherwise.
The SR1 trust region algorithm is given in \Cref{alg:TR_algorithm} in \Cref{appendix:SR1}.

Under the usual assumptions of trust region theory, the trust region algorithm used here guarantees convergence to first-order
stationary points, i.e.,
\begin{equation*}
\lim_{k\to\infty} \nabla f(\mathbf{x}_k) = 0,
\end{equation*}
where $\{\mathbf{x}_k\}_{k \ge 0}$ is the sequence of trust region centers throughout the iterations \citep[Theorem~4.6]{nocedal2006numerical}.
Furthermore, under additional
regularity assumptions in a neighborhood of a local minimizer
$\mathbf{x}^\star$ with \(\nabla^2 f(\mathbf{x}^\star)\succ 0\), the iterates of \Cref{alg:TR_algorithm} converge
superlinearly \citep[Theorem~6.7]{nocedal2006numerical}. These results form the basis for the local convergence guarantees of LAGO, presented
in \Cref{prop:local_convergence}.

\section{LAGO framework description}\label{sec:lago}
LAGO relies on the competition between the two optimization routines presented
in \Cref{sec:background}: a \emph{global} BO strategy
and a \emph{local} trust region method.
At each iteration, we follow a selection criterion which adaptively decides
whether the local or global step is accepted. 

\textbf{Global optimization.} The global component is implemented using
standard BO. At each iteration, the acquisition function is
optimized outside the active trust region, based on a
global GP surrogate of the objective. This promotes exploration of the
search space while avoiding interference with the local refinement phase.
In this work, we use the EI acquisition function,
as it is also employed in the selection criterion between local and global
candidates. Other acquisition functions could be used without modifying the
overall framework.

Since local steps compute gradient information, one could consider replacing
the standard BO in LAGO by gradient-enhanced BO (gradBO) \cite{wu2017bayesian} where
the available gradients are incorporated in the GP.
While beneficial in some settings (see \hyperlink{RQ_3}{\textbf{RQ3}}), these
models incur a computational cost scaling cubically with the dimension $d$ due to conditioning the GP on $d$ partial derivatives at each point, which can be prohibitive. A brief
overview of this approach is provided in \Cref{appendix:gradGPs}.
Importantly, LAGO is agnostic to the choice of GP model as its
contribution lies in the interleaving of BO and local
refinement, rather than in the GP modeling itself.

\textbf{Local optimization.} We use the SR1-TR method
(\Cref{alg:TR_algorithm}) for the local component.
The trust region is centered at the current best point $\mathbf{x}_c$, and
the function and gradient values at this point are used to construct the
quadratic model $m^q$.
The Hessian of the quadratic surrogate is initialized using the Hessian of the
GP posterior mean at $\mathbf{x}_c$ and updated throughout the local steps using
\eqref{eq:update_SR1}. 

\textbf{Selection criterion.} For expensive-to-evaluate functions, we wish to evaluate the function (and potentially its gradient) only once per iteration. 
We therefore compare at each step the global and local candidates without evaluating the function at either location. 
Let $\mathbf{x}_G$ and $\mathbf{x}_L := \mathbf{x}_c +\mathbf{s}_t$ denote the global and local candidates, respectively.
For the global candidate, we compute its expected improvement $\EI(\mathbf{x}_G)$ using \eqref{eq:ei}.
For the local candidate, we define the local (deterministic) improvement
$I_t := f_{\text{best}} - m^q(\mathbf{s}_t)$,
where $f_{\text{best}}$ denotes the objective value at the current trust region center (which always coincides with the current best minimizer).
We select the global candidate if its expected improvement exceeds the predicted local improvement, that is 
\begin{equation}\label{eq:global_condition}
    \EI(\mathbf{x}_G) > \gamma I_t,
\end{equation}
with $\gamma > 0$ a parameter tuning the local-global behavior of the algorithm (see sensitivity analysis in \Cref{appendix:sensitivity_gamma}), which we set to $1$ by default.
Both quantities estimate the decrease obtained by spending the next evaluation
globally or locally, and are therefore comparable at the level of predicted
objective improvement.
Intuitively, selecting the global candidate indicates a higher
expected improvement outside the current trust region, whereas selecting the
local candidate favors exploiting the current trust region.
We note that the selection is based on \emph{predicted} improvement and does not guarantee the largest \emph{realized} improvement.
\Cref{fig:iteration_lago} shows an iteration of LAGO.

\begin{wrapfigure}{r}{0.4\textwidth}
\centering
\vspace{-15pt}
\includegraphics[width=0.4\textwidth]{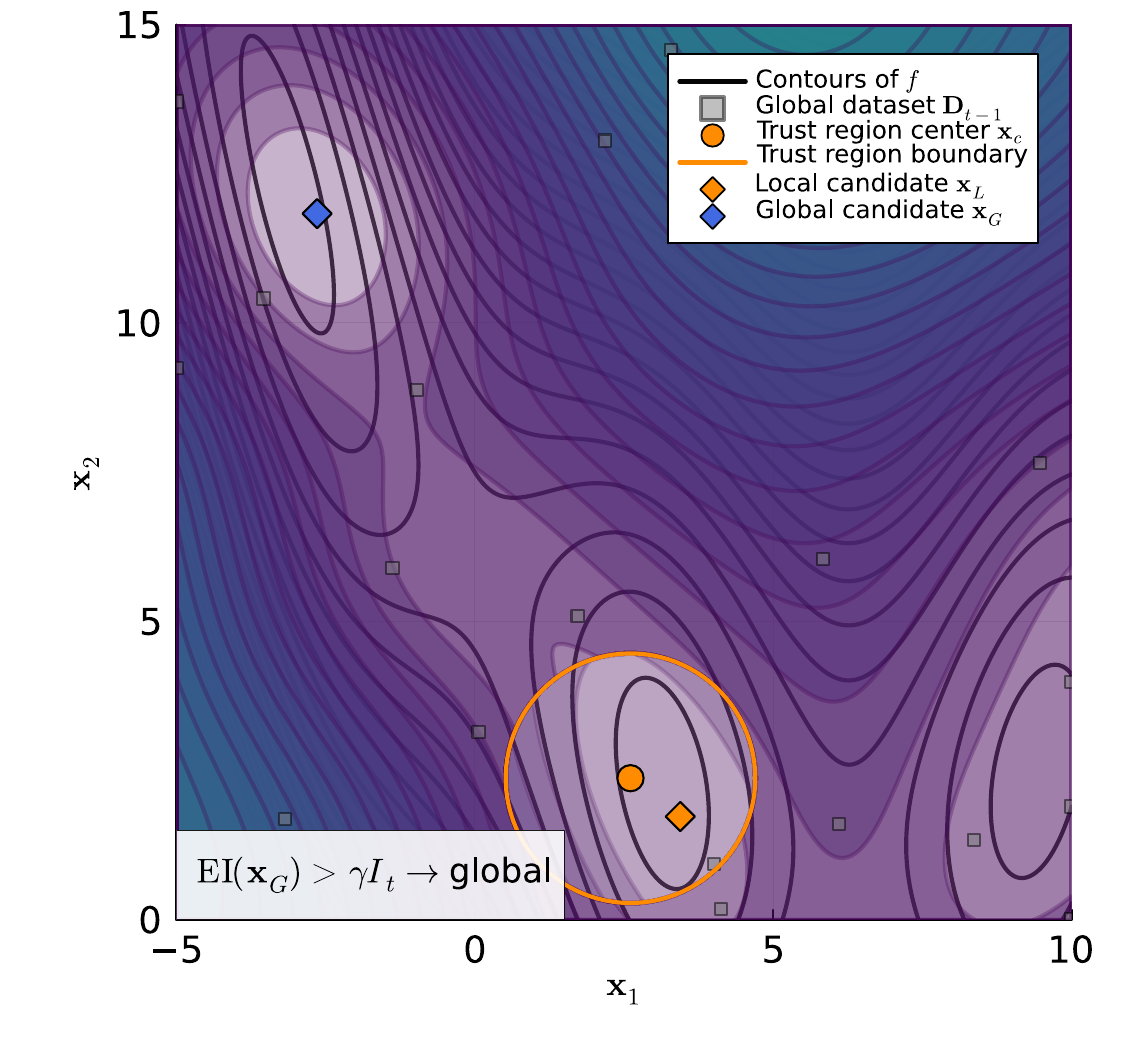}
\caption{\textbf{One LAGO iteration.} The TR is centered at the
current best minimizer. A local candidate is proposed inside the region,
while a global candidate is proposed outside. The shaded contours are the GP posterior mean outside of the TR, and the quadratic surrogate inside it.}
\vspace{-15pt}
\label{fig:iteration_lago}
\end{wrapfigure}

\textbf{Information flow.} A key feature of LAGO is its one-sided communication from the local to the global component.
In particular, updating the global GP does not modify the local optimization routine, except if the trust region is relocated to a better minimizer after a BO step. 
This preserves the convergence properties of the trust region method, yielding
\Cref{prop:local_convergence} (proof in \Cref{appendix:local_convergence}).

\begin{proposition}[Conditional reduction of LAGO to SR1-TR]\label[proposition]{prop:local_convergence}
Let $f \in C^2(\mathcal{X})$, and let $\{(\mathbf{x}_c^k, \Delta_k, H_k)\}_{k \ge 0}$ denote the sequence of
trust region centers, radii, and Hessian approximations generated by LAGO.
Assume that there exists $K \in \mathbb{N}$ such that, for all $k \ge K$, no
BO recentering occurs. Then, for all $k \ge K$,
\[
(\mathbf{x}_c^{k+1}, \Delta_{k+1}, H_{k+1})
=
\operatorname{SR1-TR}(\mathbf{x}_c^k, \Delta_k, H_k).
\]
Under the assumptions in \citep[Theorem~ 6.7]{nocedal2006numerical}, $\{\mathbf{x}_c^k\}_{k \ge K}$ converges superlinearly to a stationary point 
$\mathbf{x}^\star$.
\end{proposition}

Conversely, to inform the global search about the local landscape without
densely conditioning on nearby local iterates, the GP is conditioned on a filtered dataset that retains
the trust region center and points satisfying a lengthscale-based criterion, that is
\begin{equation}\label{eq:lengthscale_criterion}
    \mathbf{D}_t
    =
    \{(\mathbf{x}_c, f(\mathbf{x}_c))\}
    \cup
    \{(\mathbf{x}, f(\mathbf{x})) \in \mathbf{D}_{t-1} :
    \|\mathbf{x} - \mathbf{x}_c\|_2 > \nu \ell\},
\end{equation}
where $\mathbf{D}_{t-1}$ is the GP dataset at iteration $t-1$, $\ell$ is
the GP lengthscale, and $\nu$ is a hyperparameter.
In the terminology of scattered data approximation,
\eqref{eq:lengthscale_criterion} enforces a lower bound on the separation distance 
from the trust region center: it removes samples within a
lengthscale-dependent ball around $\mathbf{x}_c$, thereby preventing dense local
conditioning and mitigating kernel matrix ill-conditioning~\cite{wendland2004scattered}.

\textbf{Trust region termination.} 
We use a local termination criterion to detect when the trust-region refinement
has locally converged, indicating that no further significant improvement is
expected within the current trust region.
We check here a step-size condition, declaring the trust region as terminated if $\|\mathbf{s}_t\| \le \varepsilon_{\text{step}}$, (set to $10^{-7}$ across all our experiments).
The trust region radius $\Delta$ is then reset to $\min\left(\Delta, \ell/2\right)$\footnote{This choice reduces the risk of the global GP
to repeatedly propose candidates near the trust region boundary,
which may occur when its radius is too large relative to the GP lengthscale.}, and the BO routine is allowed to pursue further exploration, if necessary.

\textbf{Early stopping criterion.} 
The algorithm terminates early, i.e., before exhausting the allocated evaluation budget, when the following two conditions are \emph{simultaneously} satisfied:
\begin{itemize}
    \item[(S$_1$)] \hypertarget{termination_1}{} a sequence of $N$ global candidates yields EI values below a threshold $\varepsilon_{\text{T}}$.
    \item[(S$_2$)] \hypertarget{termination_2}{} the predicted local improvement $I_t$ is below $\varepsilon_{\text{T}}$.
\end{itemize}
Condition \hyperlink{termination_1}{(S$_1$)} reflects the absence of meaningful
global exploration, while \hyperlink{termination_2}{(S$_2$)} indicates no further predicted
local improvement. 
Only when both conditions are satisfied do we conclude that further improvement
by either component is unlikely. The stopping tolerance $\varepsilon_T$ (set to $10^{-12}$ in all experiments)
provides a direct control over the trade-off between solution accuracy and
evaluation cost, with smaller values leading to stricter convergence and higher
evaluation counts (see \Cref{appendix:sensitivity_eps}).

\begin{algorithm}
\caption{LAGO: Local-Global Framework Combining BO and SR1-TR.}
\label{alg:LAGO_algorithm}
\begin{algorithmic}[1]
    \STATE \textbf{Input:} objective $f$, gradient $\nabla f$, design space $\mathcal{X}$, parameters $\nu, \gamma, \varepsilon_T$, a prior $\mathcal{GP}(m,k)$, 
initial dataset $\mathbf{D}_0 := \{(\mathbf{x}_i,f(\mathbf{x}_i))\}_{i=1}^{N_0}$ of size $N_0$, initial trust region radius $\Delta_0$.
\STATE Obtain the posterior of the GP conditioning on $\mathbf{D}_0$, and fit the GP hyperparameters using MLE. 
\STATE Compute $\mathbf{x}^{\text{I}} \leftarrow \argmin_{\mathbf{x} \in \mathcal{X}} \bar{m}(\mathbf{x})$, a first informed candidate based on the GP.
\STATE Set $\mathbf{D}_1 \leftarrow \mathbf{D}_0  \cup (\mathbf{x}^{\operatorname{I}},f(\mathbf{x}^{\operatorname{I}}))$.
\STATE Initialize the trust region (TR) with center $\mathbf{x}_c\leftarrow
\arg\min_{\mathbf{x} \in \mathbf{D}_1} f(\mathbf{x})$, gradient $\nabla
f(\mathbf{x}_c)$, radius $\Delta \leftarrow \Delta_0$ and approximate Hessian
$H \leftarrow \nabla^2 \bar{m}(\mathbf{x}_c)$. 
\FOR{\(t=2,3,\dots\) until budget exhaustion or early stopping criterion is met}
    \STATE Refit GP hyperparameters (scale and lengthscale $\ell$) periodically (e.g., every $10$ iterations).
    \STATE \textbf{Local proposal:} Get local candidate $\mathbf{x}_L \leftarrow \mathbf{x}_c + \mathbf{s}_t$ where $\mathbf{s}_t$ solves \eqref{eq:tr_subproblem}.
    \STATE Check whether the TR is terminated, and if so, set $\Delta \leftarrow \min(\Delta, \frac{1}{2}\,\ell)$. 
    \STATE \textbf{Global proposal:} Obtain the global candidate 
            $$\mathbf{x}_G \leftarrow \argmax_{\mathbf{x}\in \mathcal{X}\setminus B(\mathbf{x}_c,\Delta)} \EI(\mathbf{x}).$$
    \IF{TR is terminated}
        \STATE Evaluate $f(\mathbf{x}_G)$ and set $\mathbf{D}_t \leftarrow \mathbf{D}_{t-1} \cup \{(\mathbf{x}_G,f(\mathbf{x}_G))\}$.
        \STATE Reinitialize the TR if new minimizer found (after evaluating $\nabla f(\mathbf{x}_G)$).
        \STATE Condition the GP on $\mathbf{D}_t$.
        \STATE \textbf{Continue to next iteration.}
    \ENDIF

    \STATE Compute the improvement $I_t \leftarrow f(\mathbf{x}_c) - m^q(\mathbf{s}_t)$.
    \IF{$\EI(\mathbf{x}_G) > \gamma I_t$}
        \STATE Evaluate $f(\mathbf{x}_G)$ and set $\mathbf{D}_t \leftarrow \mathbf{D}_{t-1} \cup \{(\mathbf{x}_G,f(\mathbf{x}_G))\}$.
        \STATE Reinitialize the TR if new minimizer found (after evaluating $\nabla f(\mathbf{x}_G)$).

    \ELSE
        \STATE Evaluate $(f(\mathbf{x}_L), \nabla f(\mathbf{x}_L))$.
        \STATE Perform one SR1-TR step (Algorithm~\ref{alg:TR_algorithm}) to accept/reject $\mathbf{x}_L$ and update $(\mathbf{x}_c,\Delta,H)$.
        \IF{new trust region center $\mathbf{x}_c$ found}
        \STATE Set $\mathbf{D}_t$ according to the filtering rule~\eqref{eq:lengthscale_criterion}.
        \ELSE
        \STATE $\mathbf{D}_t \leftarrow \mathbf{D}_{t-1}$.
        \ENDIF
    \ENDIF
    \STATE Condition the GP on $\mathbf{D}_t$. 
\ENDFOR
\STATE \textbf{Return:} the best minimizer found $\mathbf{x}_c$.
\end{algorithmic}
\end{algorithm}

\textbf{Design rationale.} 
Compared to classical local optimizers, the global surrogate reduces the risk of
stagnation in suboptimal regions by continuously proposing candidates in
unexplored or uncertain areas using all accumulated global information.
This global knowledge also informs the trust region placement and warm-starts its
initialization: the quadratic model inherits curvature information from the GP,
avoiding repeated identity-Hessian initializations.

The candidate-level competition is the central design choice: local refinement can
begin as soon as its predicted decrease is competitive, while global exploration
can still interrupt local refinement when the GP suggests a more promising
evaluation; the parameter $\gamma$ tunes this local-global trade-off.
\Cref{alg:LAGO_algorithm} outlines the pseudo-code for LAGO.

\section{Experimental results}\label{sec:experiments}
We start by comparing LAGO against related local-global methods discussed in \Cref{sec:related_works} and other representative baselines.
Then, we evaluate LAGO through a set of RQs designed to isolate its key components (\hyperlink{C_1}{C$_1$}-\hyperlink{C_4}{C$_4$}) and assess the properties introduced in \Cref{sec:contributions,sec:lago}.
Unless specified otherwise, LAGO's hyperparameters are fixed across benchmarks.
Gradient evaluations are charged as $d$ function evaluations on synthetic
benchmarks, corresponding to a conservative finite-difference-like cost model,
whereas the adjoint gradient in the PDE problem has unit cost.
See \Cref{appendix:benchmarks_pde_opt_problem,appendix:experimental_setting} for the problem definitions and the experimental setting, respectively.
Convergence plots show the median with Inter-Quartile Range (IQR), i.e. $25$-$75^{\text{th}}$ percentiles over $50$ random seeds.

\textbf{\textcolor{mydarkblue}{How does LAGO compare against other optimization methods?}}
We start by comparing LAGO\emph{-BO} (LAGO with BO) and LAGO\emph{-grad} (LAGO with gradBO), against a representative set of optimization algorithms,
including BO (EI) \cite{jones1998efficient}, gradBO (EI with derivative information)
\cite{wu2017bayesian}, BLOSSOM \cite{mcleod2018optimization}, trust region BO
variants namely TuRBO with $M$ regions (TurBO-$M$) \cite{eriksson2019scalable}, TREGO
\cite{diouane2023trego} and LABCAT \cite{visser2025labcat} with restarts, mesh-based optimization like BADS
\cite{acerbi2017practical} and local methods such as L-BFGS
\cite{liu1989limited} with restarts. These methods are evaluated on a set of representative synthetic benchmarks and a PDE-constrained optimization problem (see
\Cref{appendix:synthetic_problems,appendix:pde_constrained_problem} for details).

Across benchmarks, \Cref{fig:comparison} 
shows that LAGO is competitive with, and often outperforms, the strongest baselines in terms of convergence speed, final accuracy, and reliability (narrow IQR).
On some problems such as Styblinski-Tang (ST) ($2$-$5$D), LAGO outperforms all
methods with lower variance, and retains faster convergence than (grad)BO. On moderate multimodal functions such
as Lévy, ST in $5$D and the PDE-problem, LAGO-BO demonstrates improved robustness, steadily refining solutions
where several hybrid approaches stagnate or exhibit high variance. In ill-conditioned settings such as the Rosenbrock which
features a narrow valley with steep gradients, LAGO-grad provides
significant gains over both BO and most hybrid baselines (see also \hyperlink{RQ_3}{\textbf{RQ3}}).

\begin{figure*}[ht]
\centering
\includegraphics[width=0.32\textwidth]{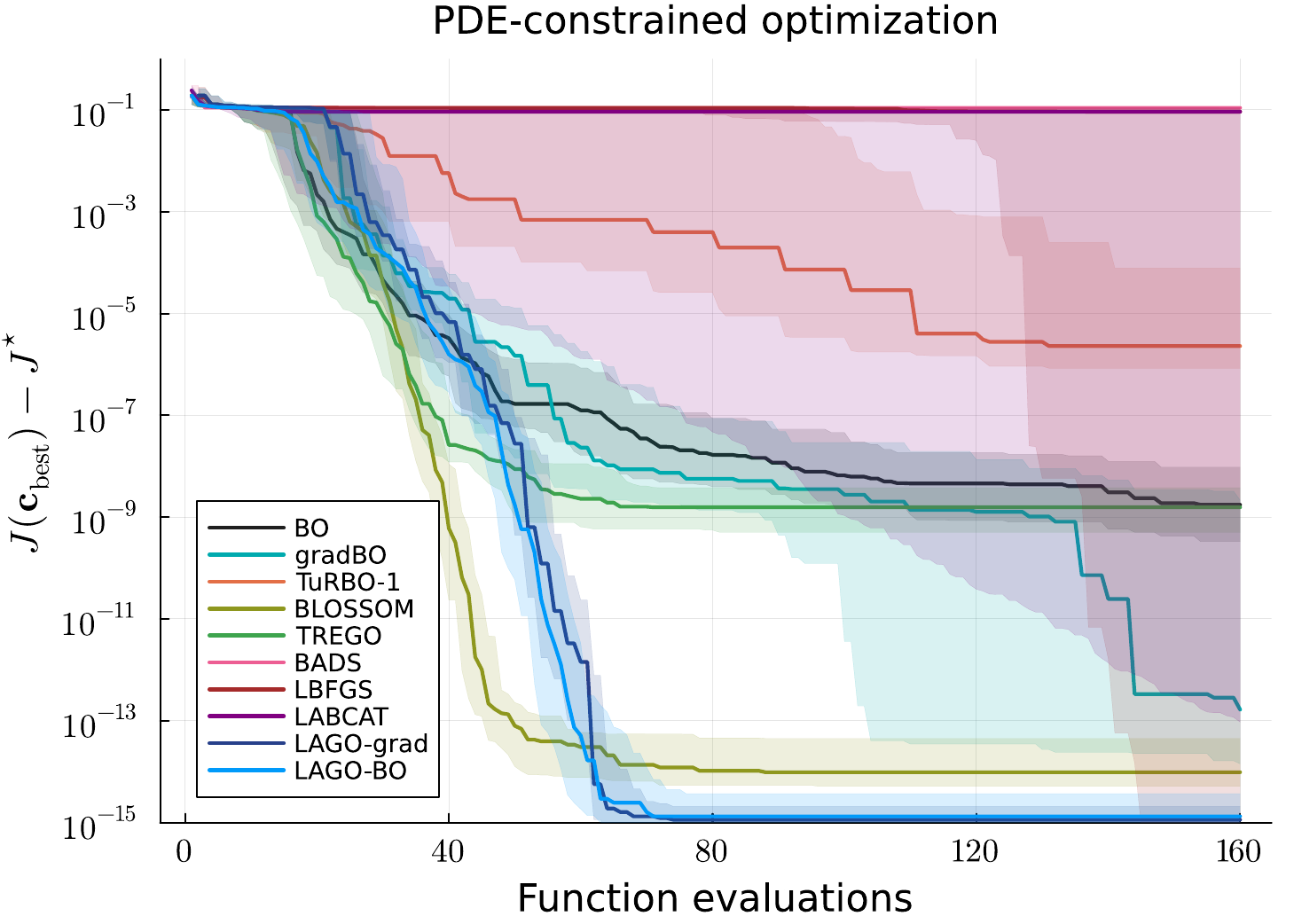}
\includegraphics[width=0.32\textwidth]{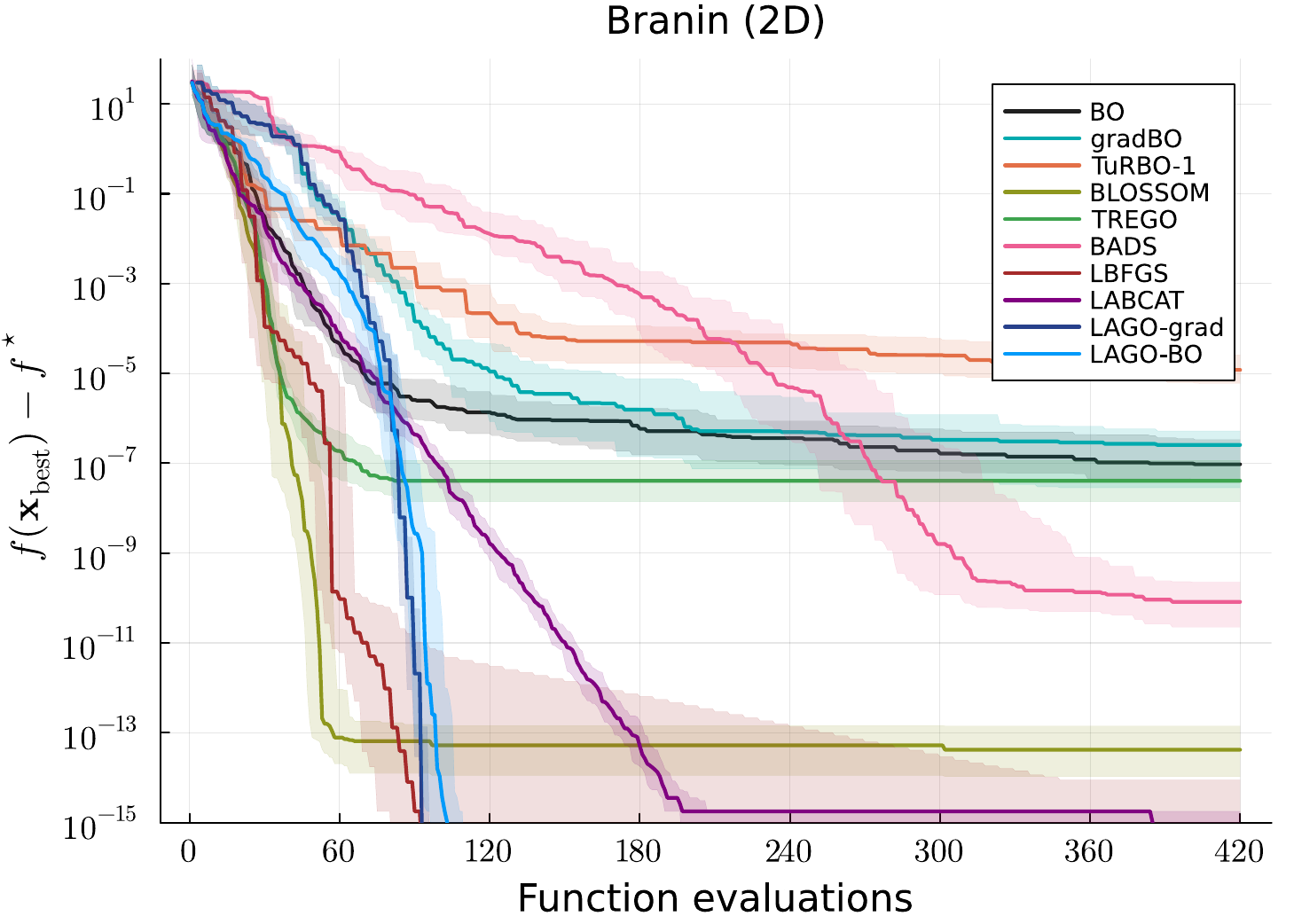}
\includegraphics[width=0.32\textwidth]{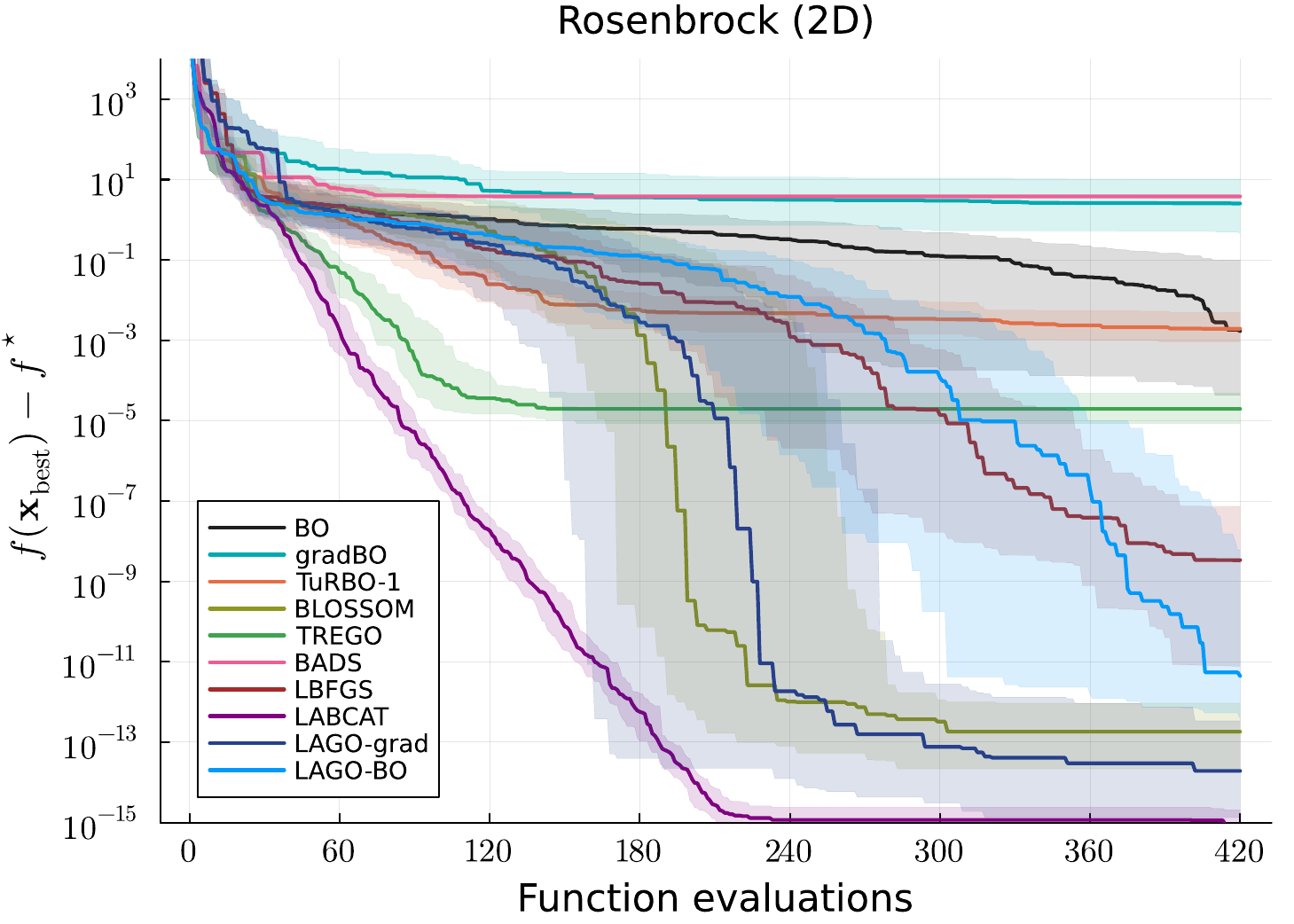}
\includegraphics[width=0.32\textwidth]{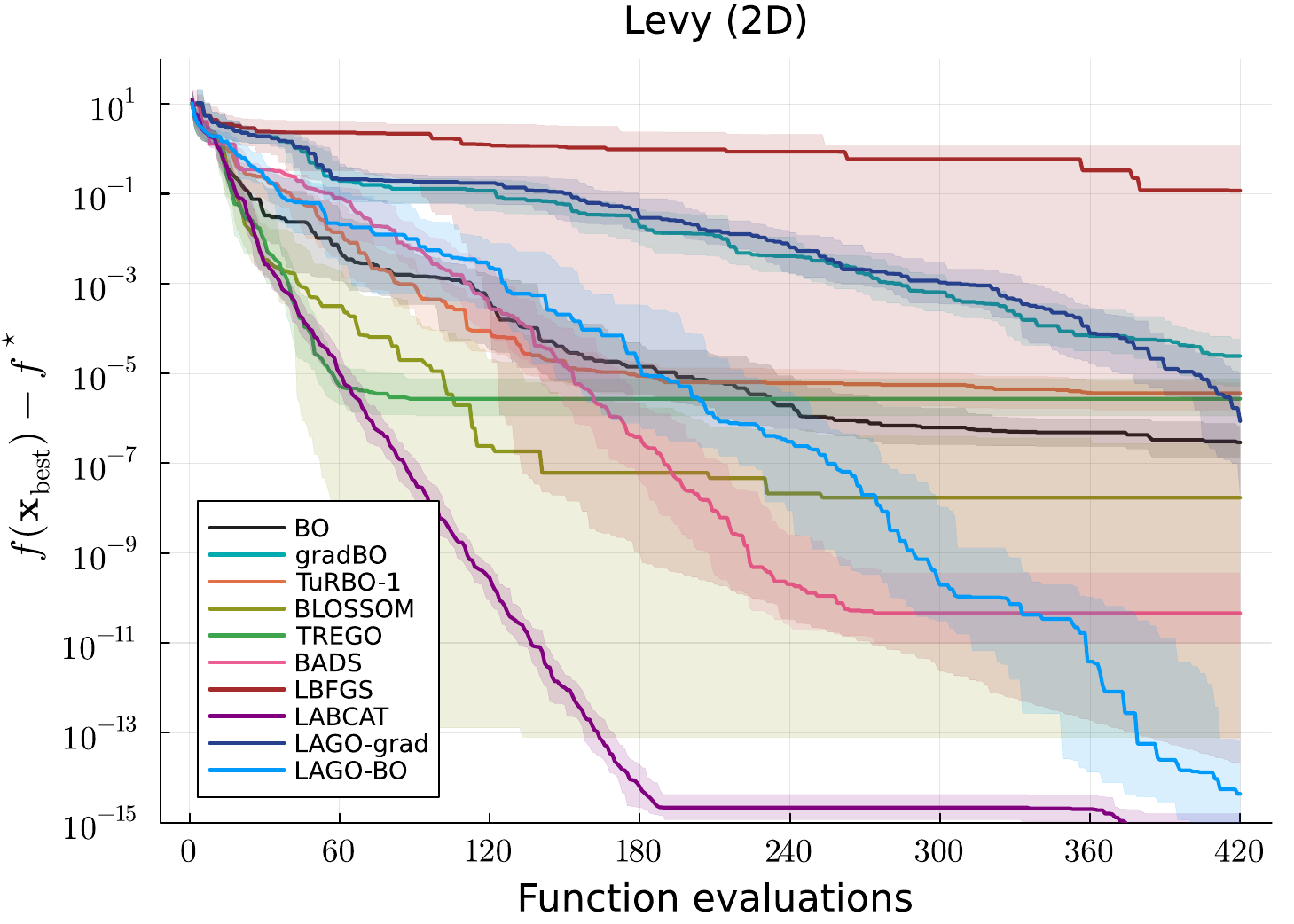}
\includegraphics[width=0.32\textwidth]{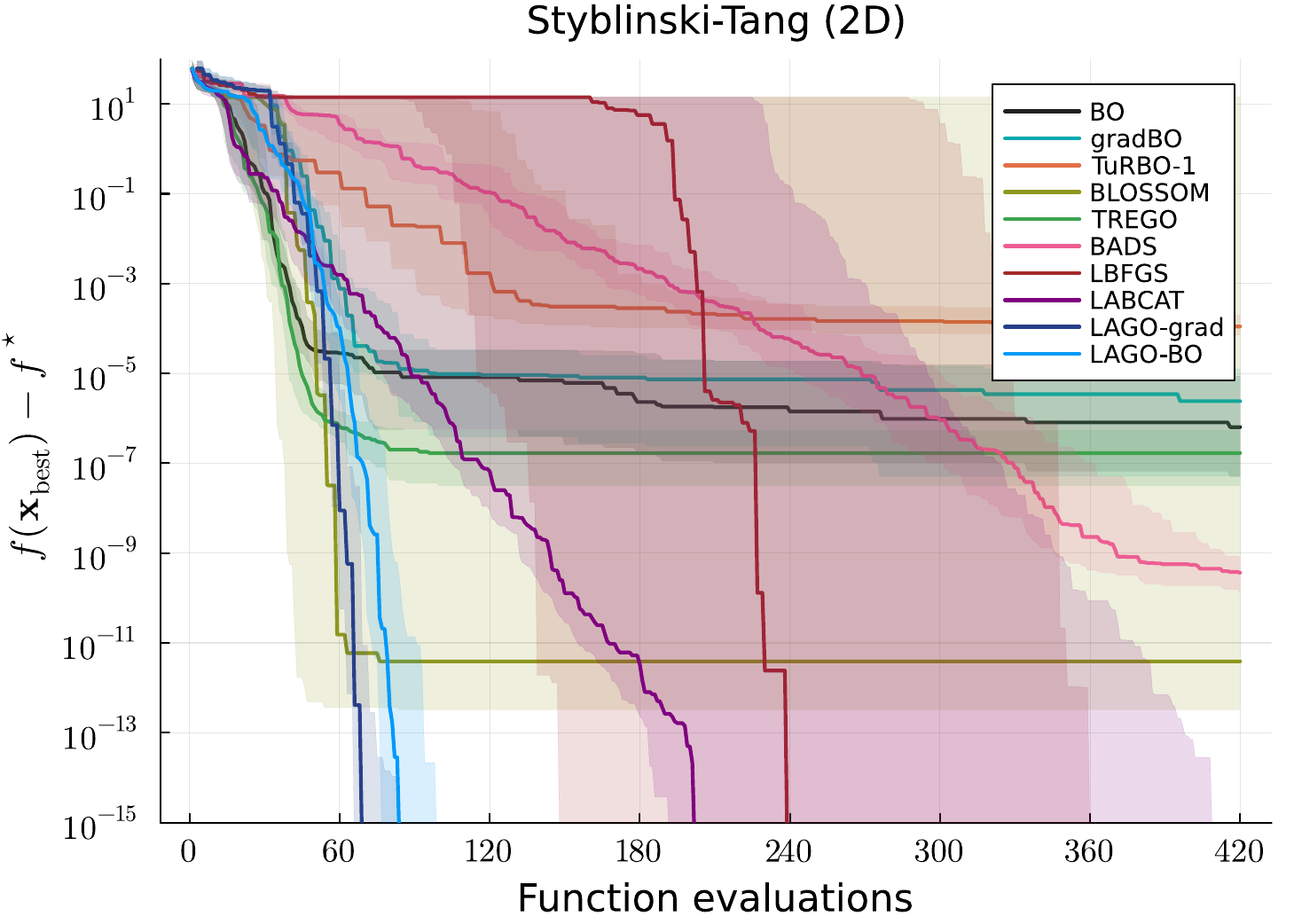}
\includegraphics[width=0.32\textwidth]{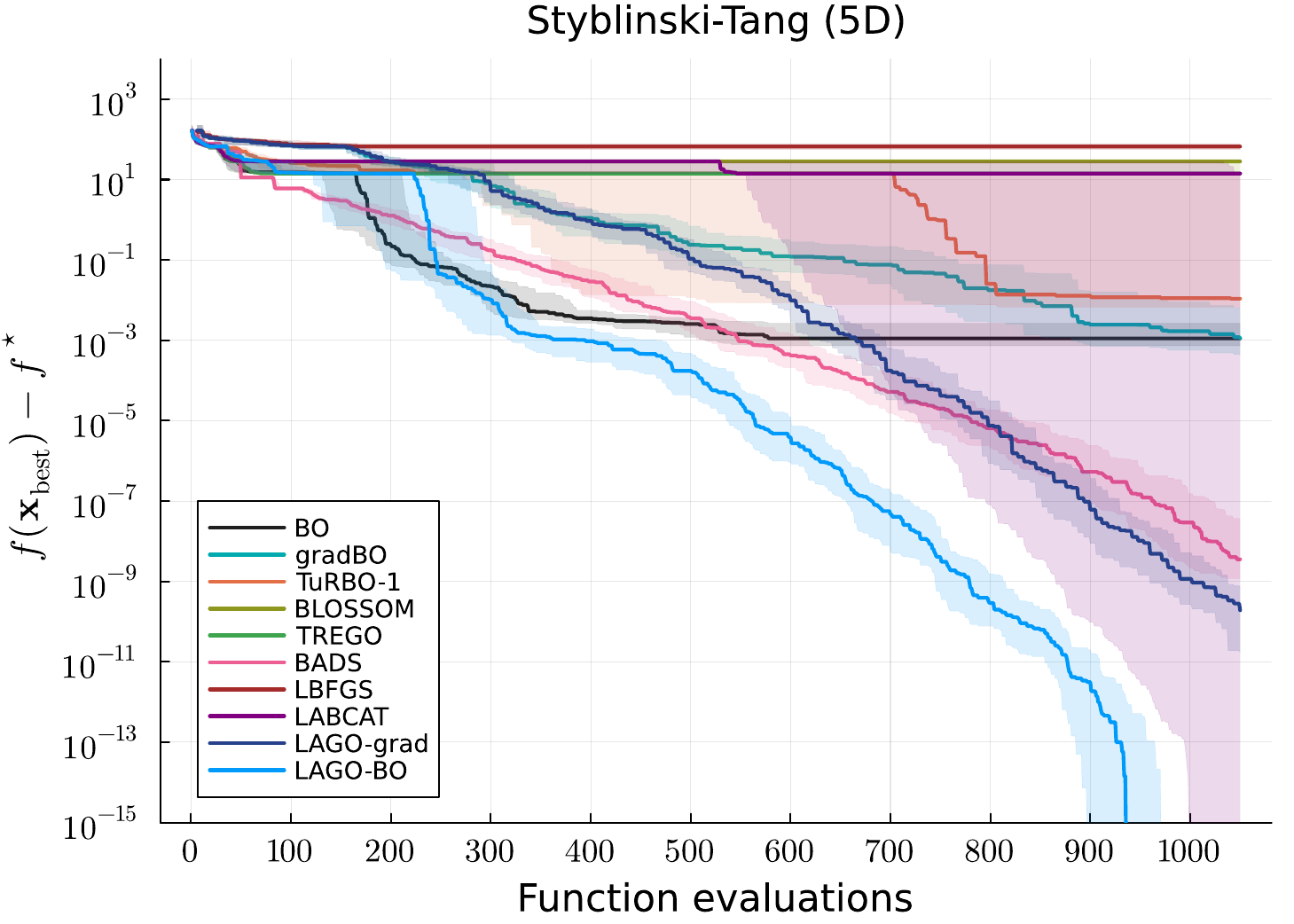}
\caption{\textbf{Convergence comparison.} Convergence of optimization routines on a PDE-constrained optimization problem and synthetic benchmarks (Median $\pm$ IQR). 
    LAGO performs consistently well across benchmarks, both convex and multimodal ones.
    It also performs strongly in higher dimensional settings, where other hybrid methods start struggling.
    Gradient evaluations cost $d$ function evaluations, except for the PDE example where gradients have unit cost (leveraging adjoints).}
\label{fig:comparison}
\end{figure*}

\hypertarget{RQ_1}{}\textbf{\textcolor{mydarkblue}{RQ1 (\hyperlink{C_1}{C$_1$}) : Does interleaving improve robustness relative to phase-based switching?}}
We evaluate whether, without requiring a predefined switching rule, LAGO provides
a competitive and robust alternative to sequential BO followed by local
refinement (SR1-TR from the best point found), as well as to other
phase-switching methods such as BLOSSOM \cite{mcleod2018optimization} and
TREGO \cite{diouane2023trego}.
For BO+local baselines, switching is triggered either when
the last five EI values are below $10^{-3}$ (BO-LOCAL-ACQ), or via a
$50/50$ allocation of the function evaluation budget (BO-LOCAL-SPLIT).
\begin{figure*}
\centering
\includegraphics[width=0.32\textwidth]{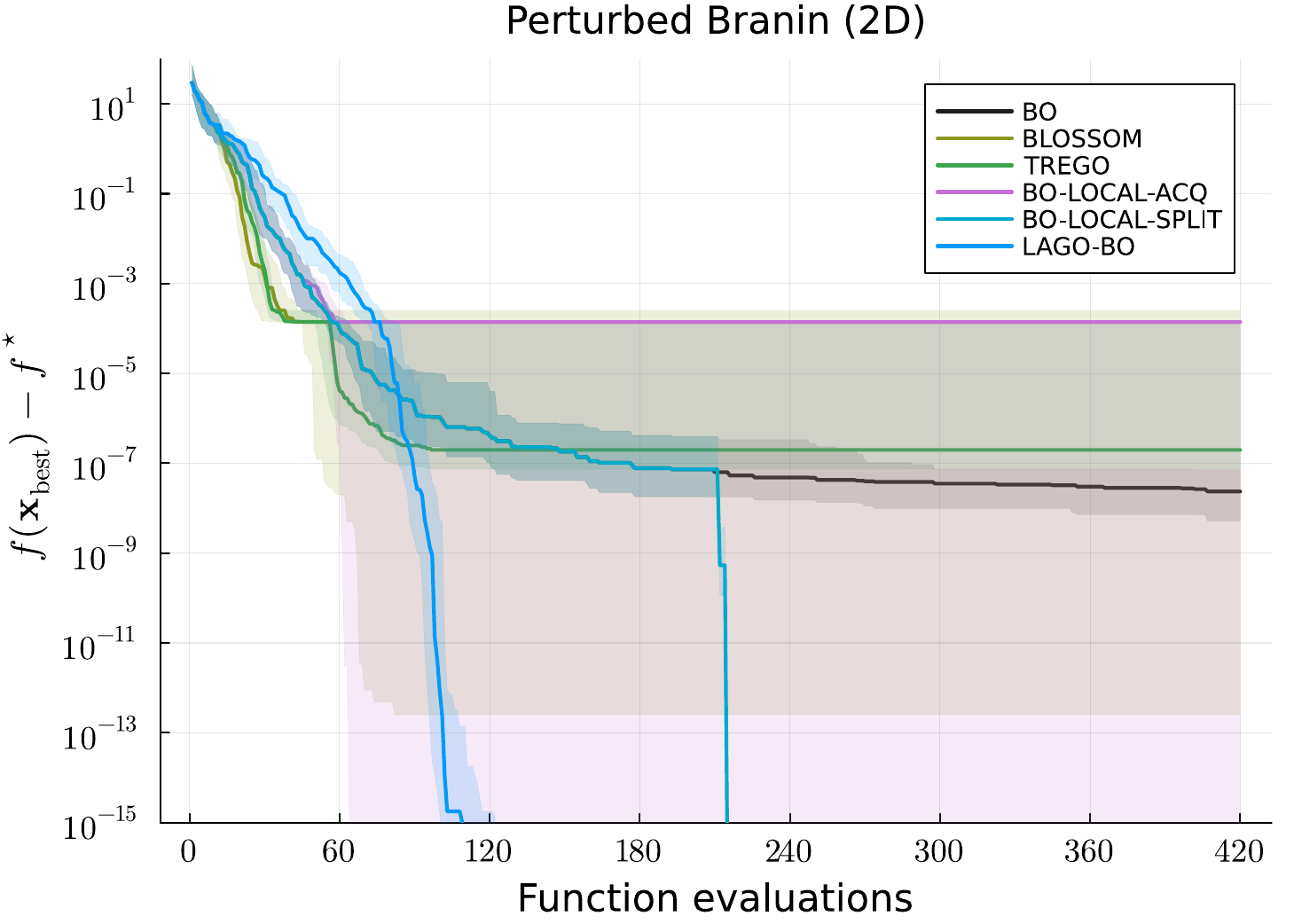}
\includegraphics[width=0.32\textwidth]{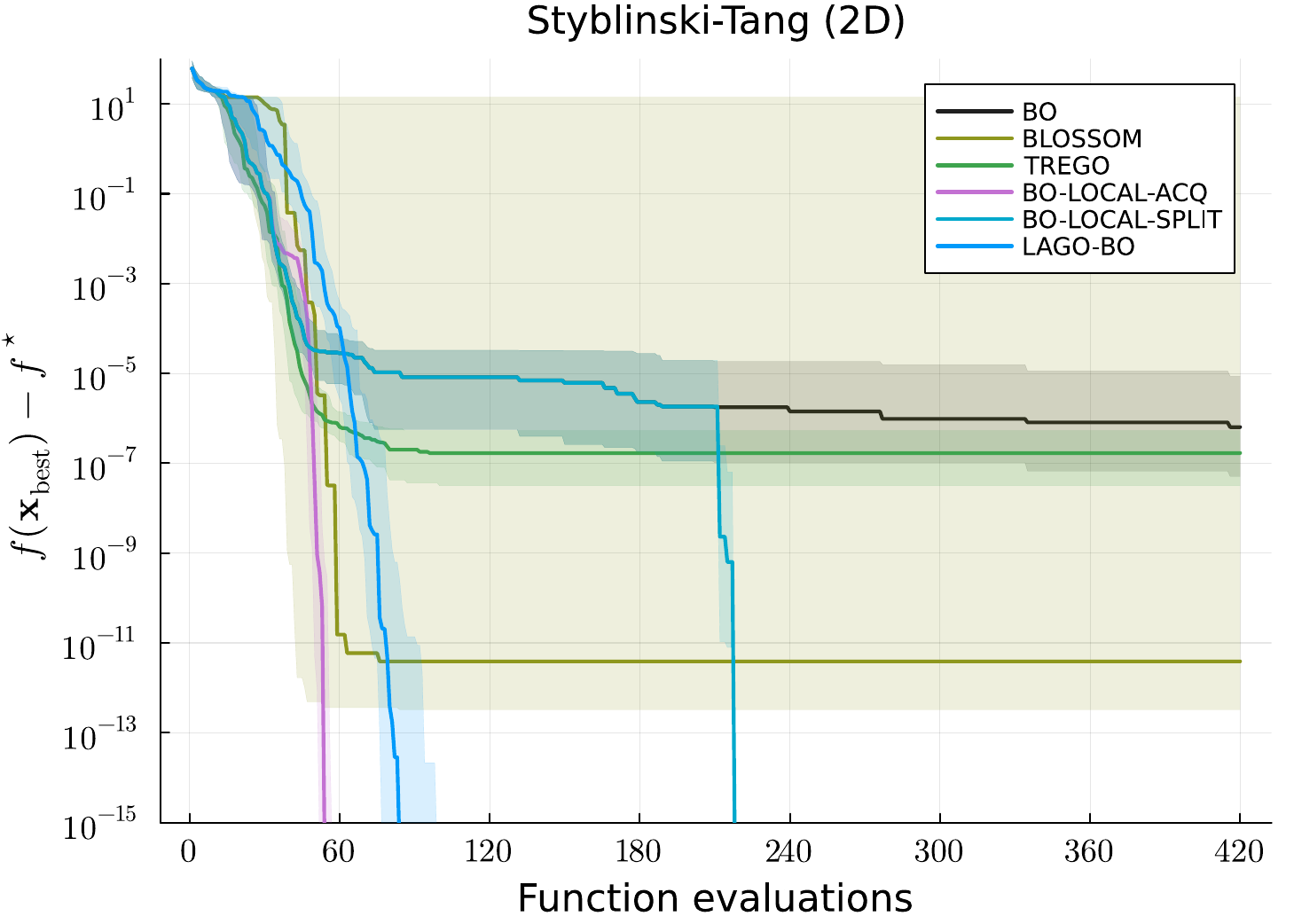}
\caption{\textbf{RQ1.} Performance comparison of LAGO and phase-switching
methods on synthetic benchmarks (Median $\pm$ IQR error). LAGO exhibits
more robust performance, particularly on benchmarks where phase-based
switching fails to identify the correct basin, leading to high variance.}
\label{fig:RQ1}
\end{figure*}

To showcase the interleaving significance, we tested it on two functions: a perturbed Branin function which
features a favorable perturbation around one of the three Branin minimizers and
the ST function. 
In these two settings, LAGO outperforms or is on par with
BO+local search variants, TREGO and BLOSSOM. 
Sequential approaches based on a single global-to-local switch, such as
BO-LOCAL-$\ast$ and BLOSSOM, rely on selecting a problem-dependent switching
point, unknown a priori, and may commit to the wrong basin.
This is particularly evident on the
perturbed Branin function (\Cref{fig:RQ1}, left), where BO-LOCAL-ACQ frequently
fails to identify the correct basin, resulting in a low success rate (38\% of runs satisfy
$|f^T_{\text{best}} - f^\star| < 10^{-12}$, compared to 100\% for
LAGO). A similar effect is observed on the ST function for BLOSSOM (\Cref{fig:RQ1}, right). In both settings, the $50/50$ split delays convergence, while
TREGO identifies the true basin but suffers from BO local refinement limitations.

In contrast, LAGO allocates each evaluation by comparing the current global and
local candidates, rather than committing to predefined global or local phases.
It achieves competitive performance across problems (see also comparison above) while consistently identifying high-quality solutions, without
requiring a problem-dependent switching rule. Rather than aiming to outperform
problem-tuned sequential strategies in all cases, LAGO offers an adaptive
approach that eliminates the need for such tuning, and leverages the two
components whenever these become competitive.
Further analysis of the interleaving patterns, including the proportion of
local steps, is provided in \Cref{appendix:interleaving}.

\hypertarget{RQ_2}{}\textbf{\textcolor{mydarkblue}{RQ2 (\hyperlink{C_3}{C$_3$}): What is the impact of the data assimilation filter?}}
We here quantify the role of selective local data assimilation \eqref{eq:lengthscale_criterion} by assessing (i) its impact on convergence speed and (ii) its effect on GP conditioning on the Lévy function.
Setting the parameter $\nu$ too high delays convergence to the true minimizer (\Cref{fig:RQ2}, left) due to filtering out too many points from the GP dataset near the trust region center.
However, if $\nu = 0.0$, meaning that \eqref{eq:lengthscale_criterion} never triggers, the kernel matrix condition number sharply increases due to uncontrolled local conditioning (\Cref{fig:RQ2}, right).
These results highlight a trade-off: the filter is necessary to prevent numerical instability of the global surrogate, but overly aggressive filtering slows down convergence.
Setting $\nu = 0.1$ as default performs well across all benchmarks, mitigating ill-conditioning and performing equally well as lower values of $\nu$.
\begin{figure*}
\centering
\includegraphics[width=0.32\textwidth]{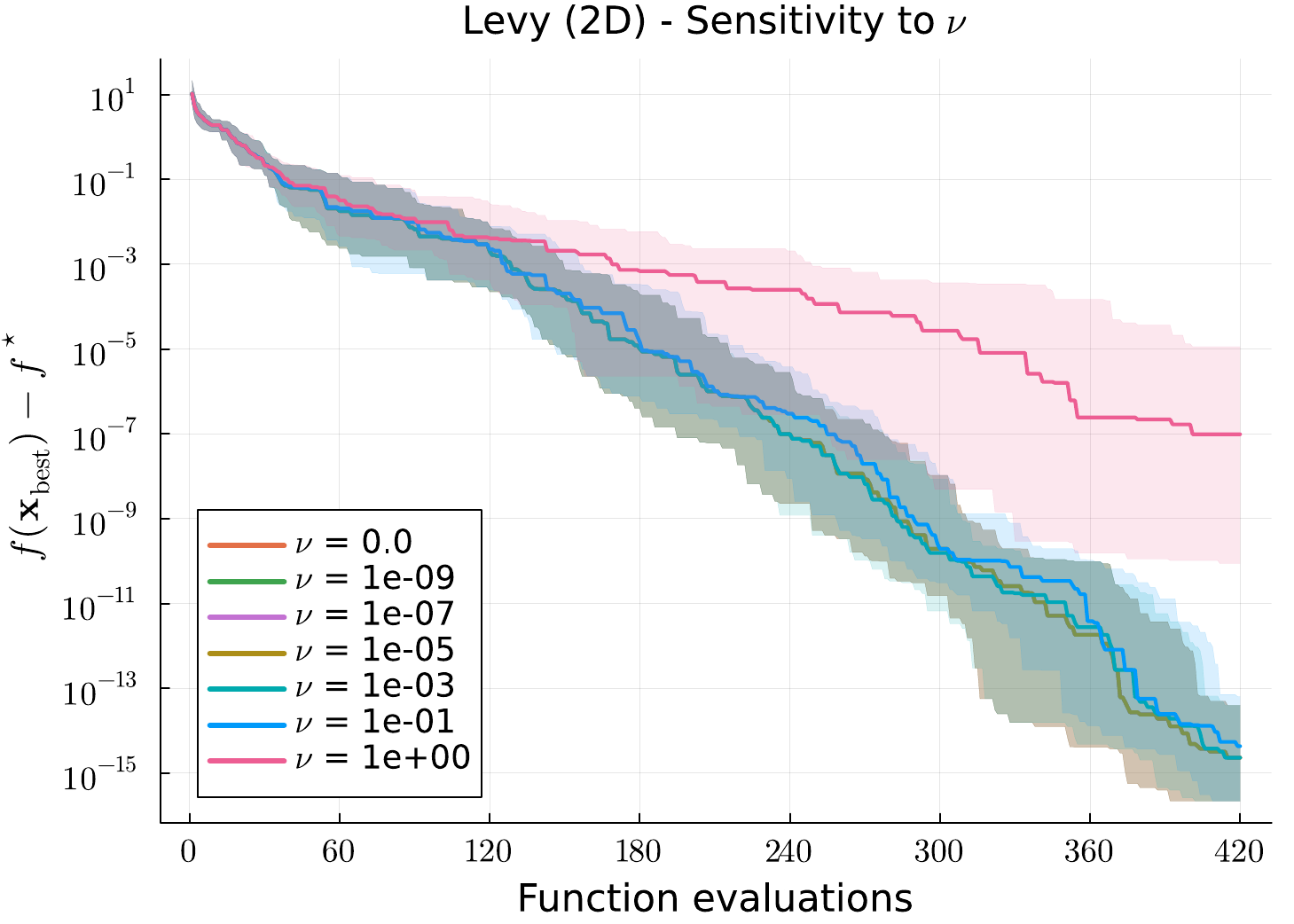}
\includegraphics[width=0.32\textwidth]{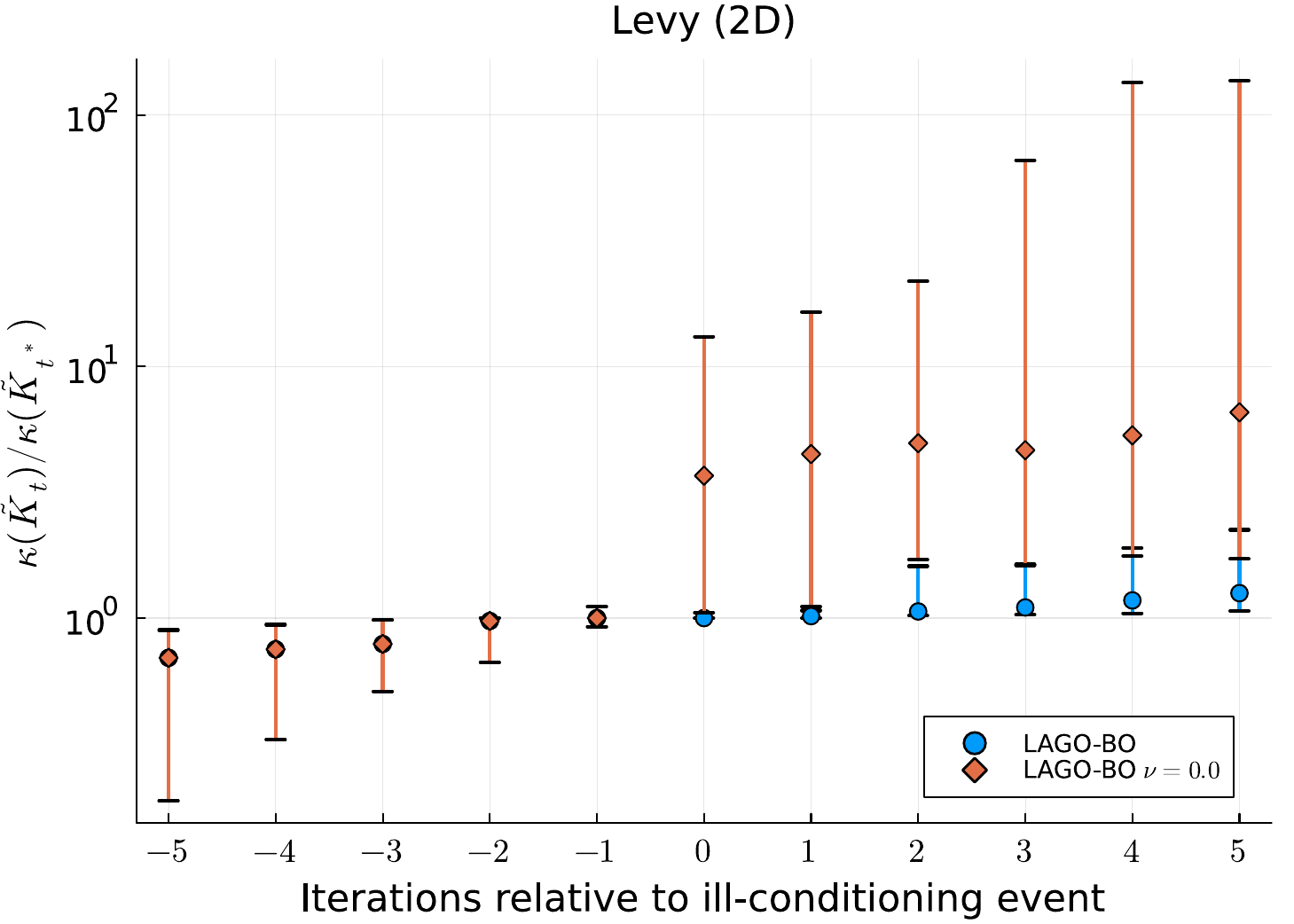}
\caption{\textbf{RQ2.} Convergence of LAGO (Median $\pm$ IQR) for different values of $\nu$ on Lévy (left), and increase in the kernel matrix condition number ratio (Median $\pm$ IQR)
when $\nu = 0$, shown over a window of $\pm 5$ iterations around the iteration where the first trigger of \eqref{eq:lengthscale_criterion} happens (right).
Setting $\nu$ too high delays convergence, while setting it to $0$ increases the ill-conditioning of the GP.}
\label{fig:RQ2}
\end{figure*}

\hypertarget{RQ_3}{}\textbf{\textcolor{mydarkblue}{RQ3: To what extent should gradient information be used globally?}} 
We study whether incorporating gradients into the global surrogate via gradGPs in gradBO
provides benefits over using them solely for local refinement. \Cref{fig:RQ3} shows that compared to LAGO-BO, LAGO-grad can either (i) perform similarly (Branin), (ii) provide additional
gains in ill-conditioned settings (Rosenbrock), or (iii) converge slower under a fixed evaluation budget (Lévy). 
The latter occurs because gradBO requires one gradient evaluation per iteration, reducing the allowed number of iterations when gradients cost $d$ function evaluations.
Overall, LAGO-grad is most useful over LAGO-BO in ill-conditioned settings where improved global modeling outweighs the gradient cost.
\begin{figure*}
\centering
\includegraphics[width=0.32\textwidth]{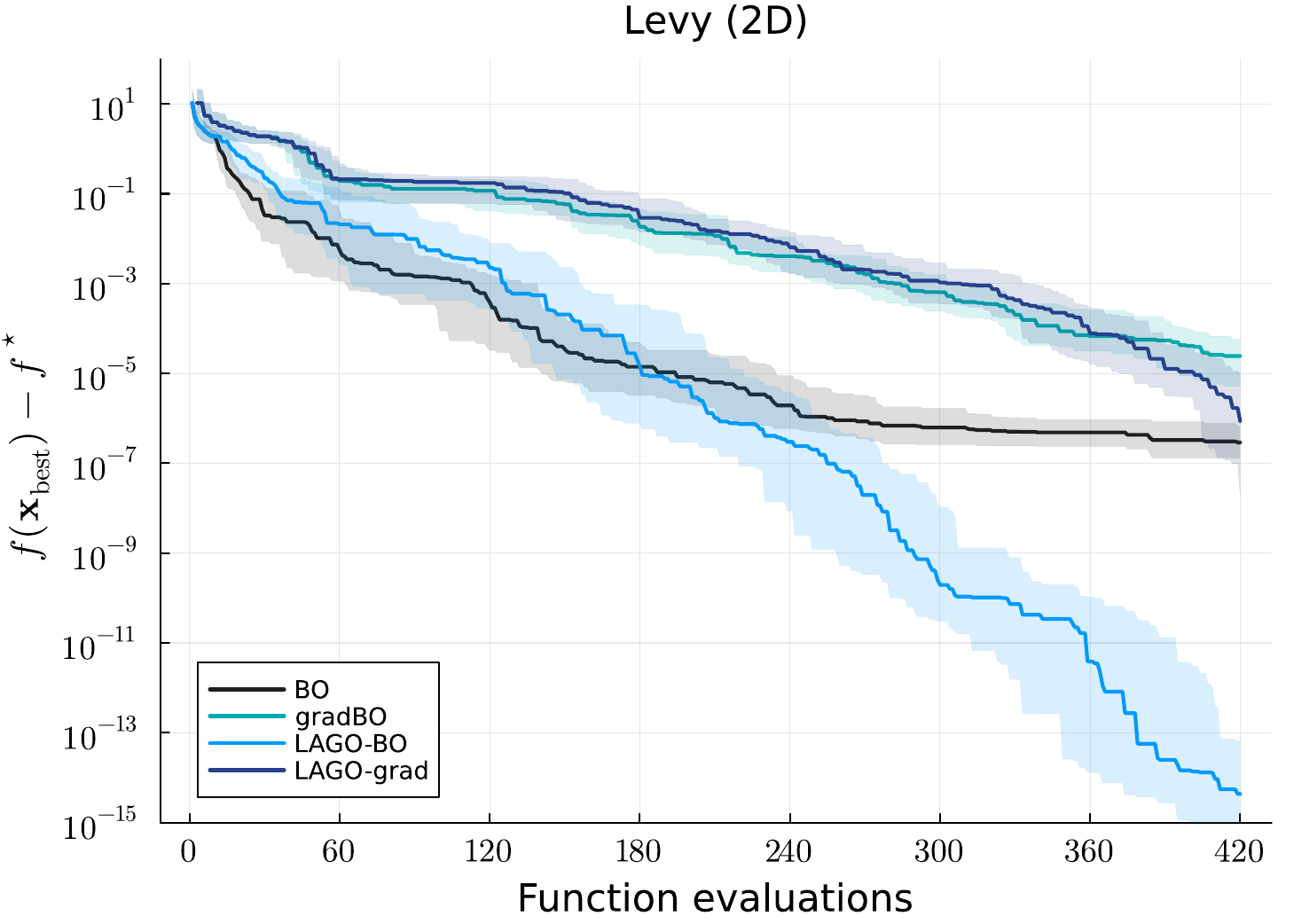}
\includegraphics[width=0.32\textwidth]{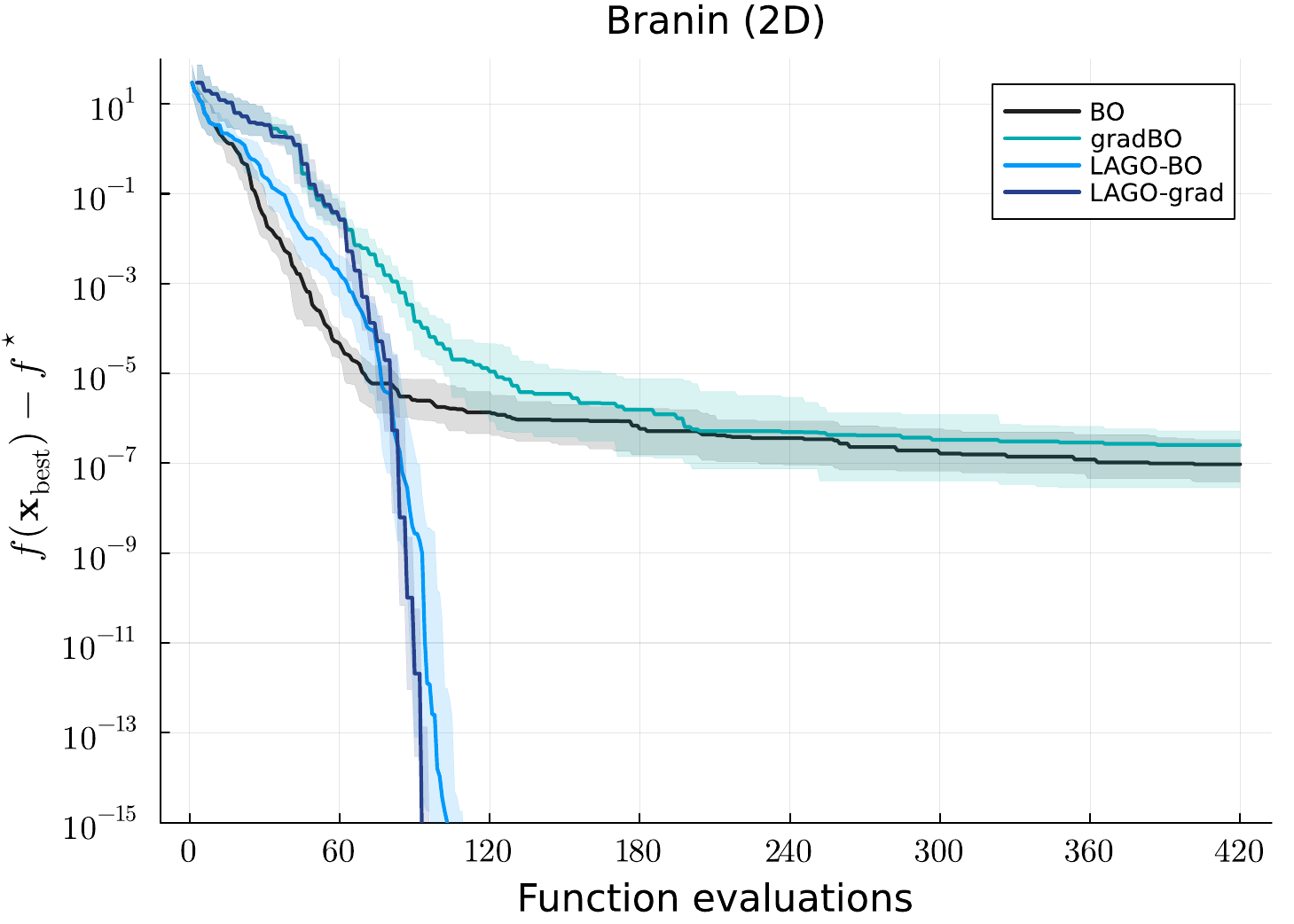}
\includegraphics[width=0.32\textwidth]{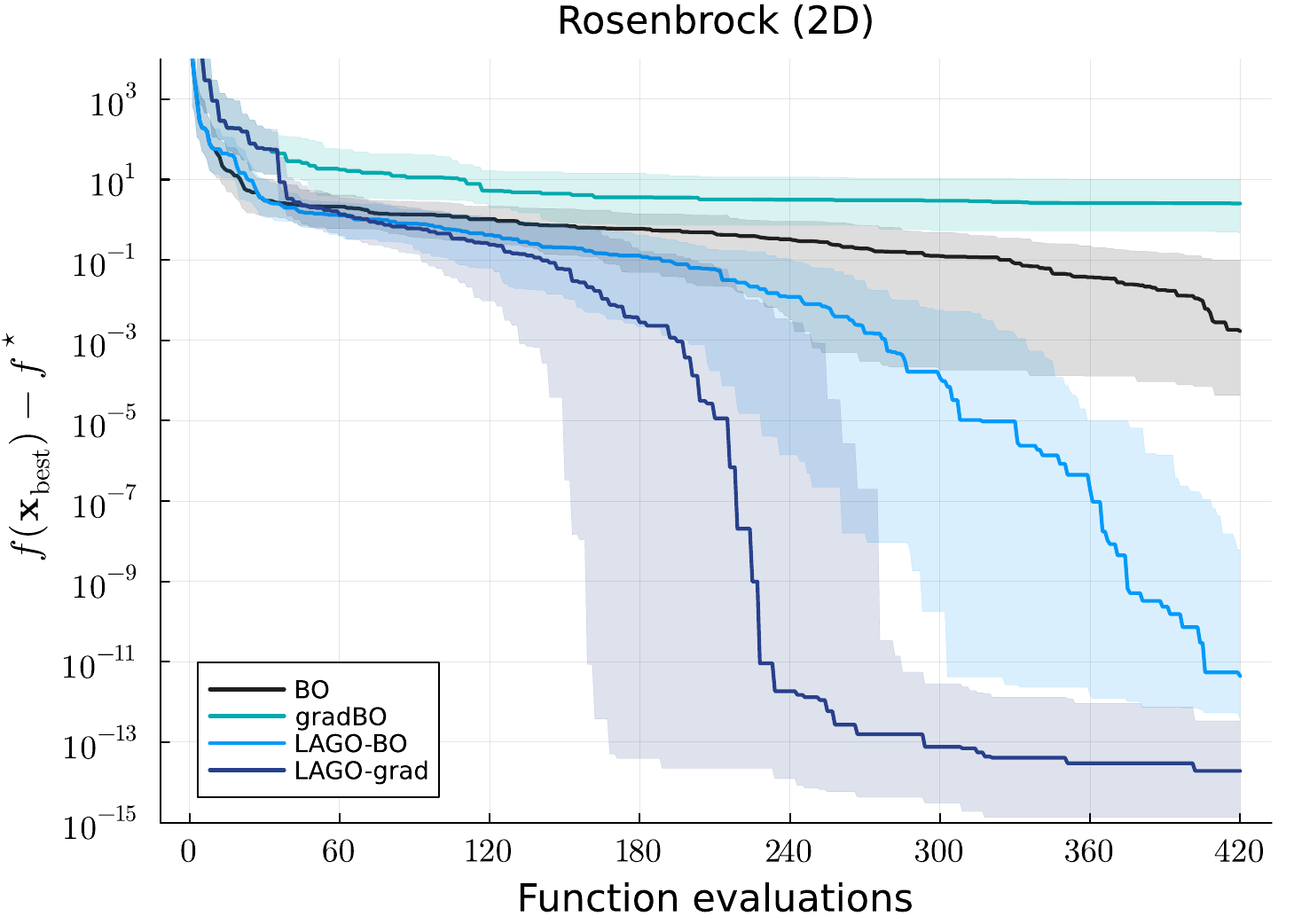}
\caption{\textbf{RQ3.} Comparison of LAGO and LAGO-grad, together with
their BO and gradBO counterparts (median $\pm$ IQR, shaded).
LAGO-grad improves over LAGO primarily on ill-conditioned problems such as
Rosenbrock but can underperform in multimodal settings such as Lévy.}
\label{fig:RQ3}
\end{figure*}

\hypertarget{RQ_4}{}\textbf{\textcolor{mydarkblue}{RQ4 (\hyperlink{C_4}{C$_4$}) : Does LAGO adapt to the function structure?}}
A key property of LAGO is its ability to remain competitive in regimes
where global exploration dominates and local refinement is less informative,
without incurring significant overhead. We evaluate whether LAGO
naturally suppresses local steps in highly multimodal settings such as
Rastrigin and Griewank in $2$D, and ST in $10$D.
As \Cref{fig:RQ4} shows in these regimes, LAGO predominantly selects global steps and closely matches
its BO component, resulting in comparable performance without
degradation nor strong overhead. This confirms that LAGO adapts to the problem structure,
leveraging local refinement when beneficial (e.g., convex or weakly multimodal problems) while reverting to global search
when it is not. \Cref{appendix:interleaving} shows LAGO trajectories showcasing the low use of local steps for these highly multimodal regimes.
\begin{figure*}
\centering
\includegraphics[width=0.32\textwidth]{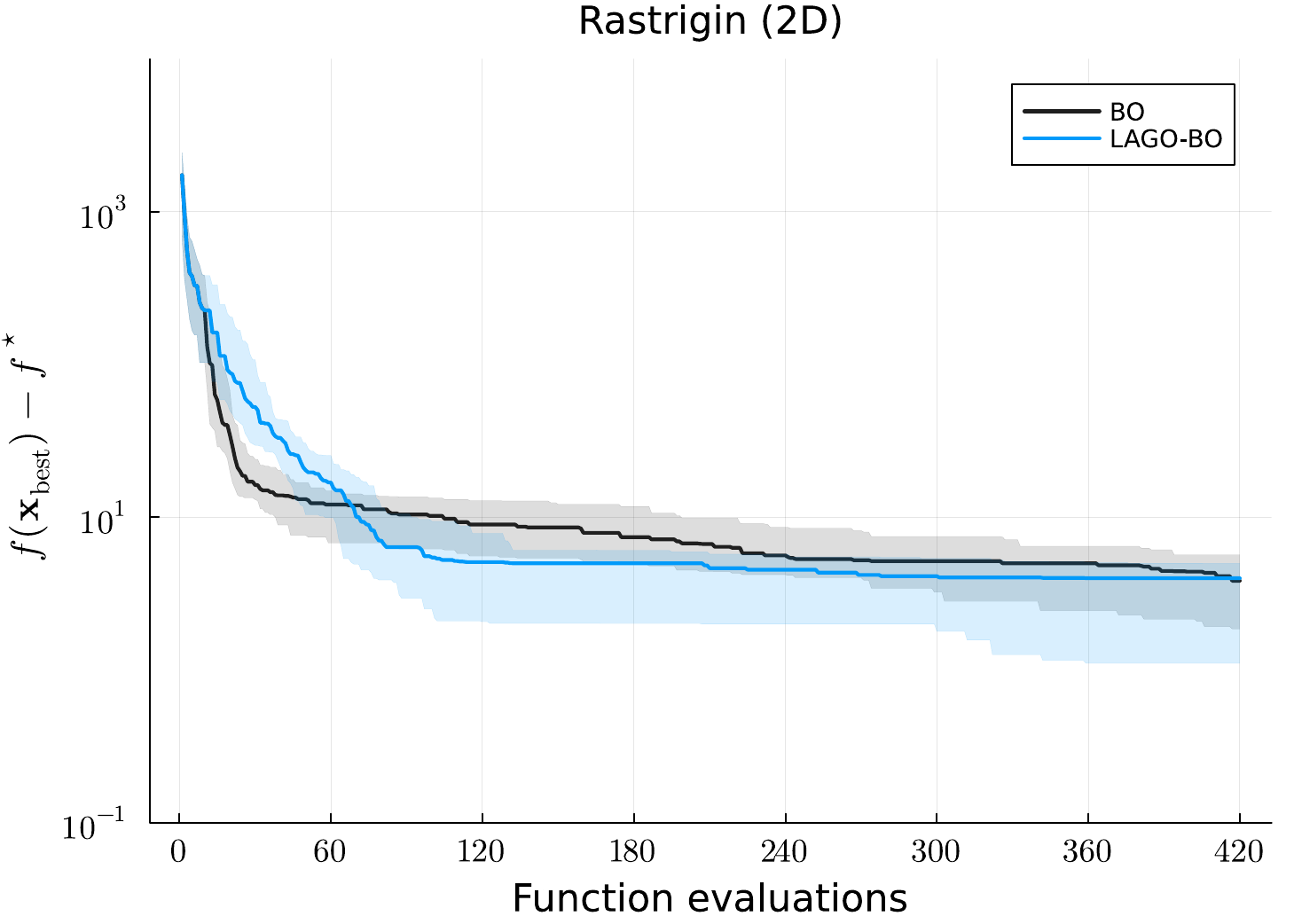}
\includegraphics[width=0.32\textwidth]{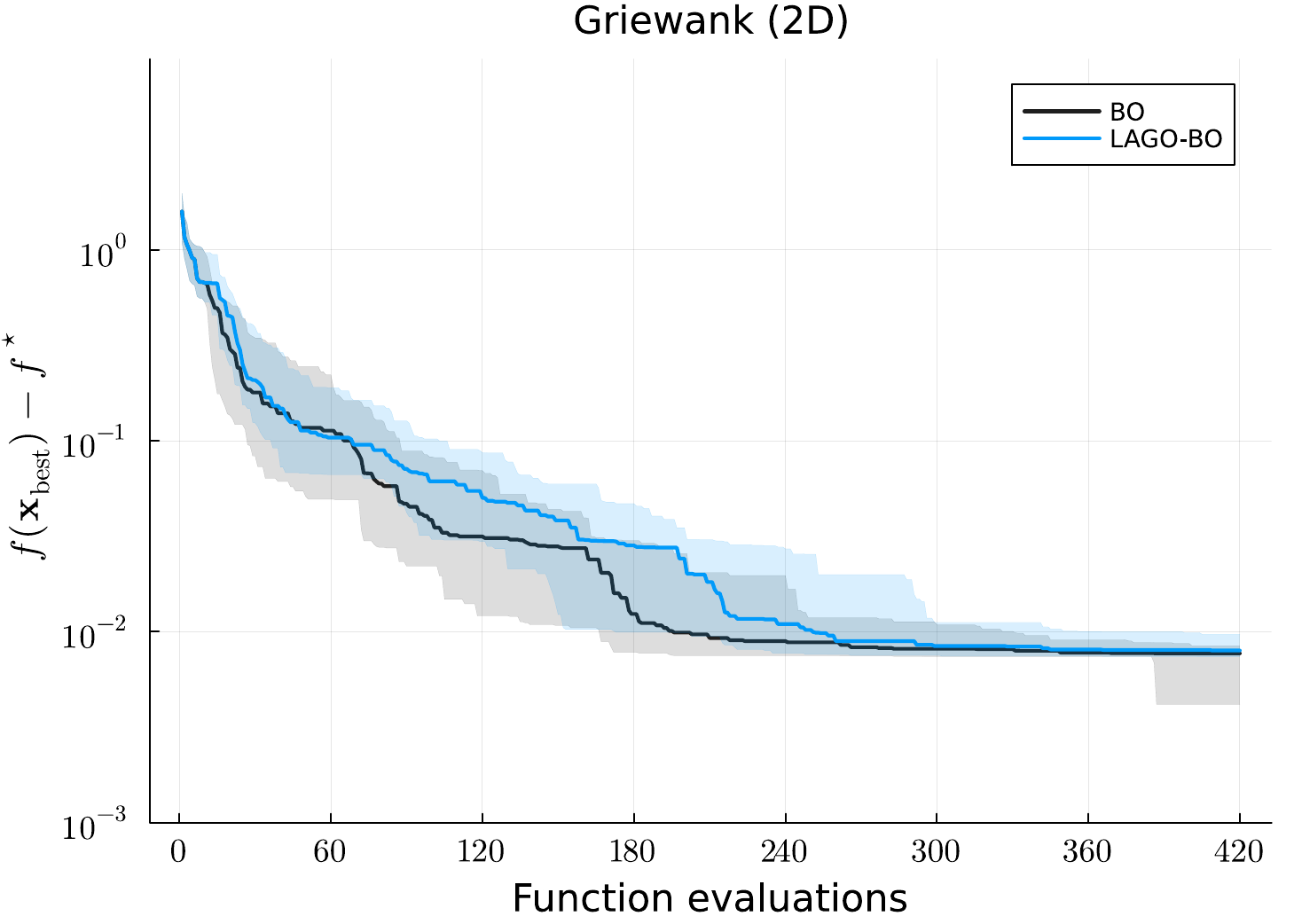}
\includegraphics[width=0.32\textwidth]{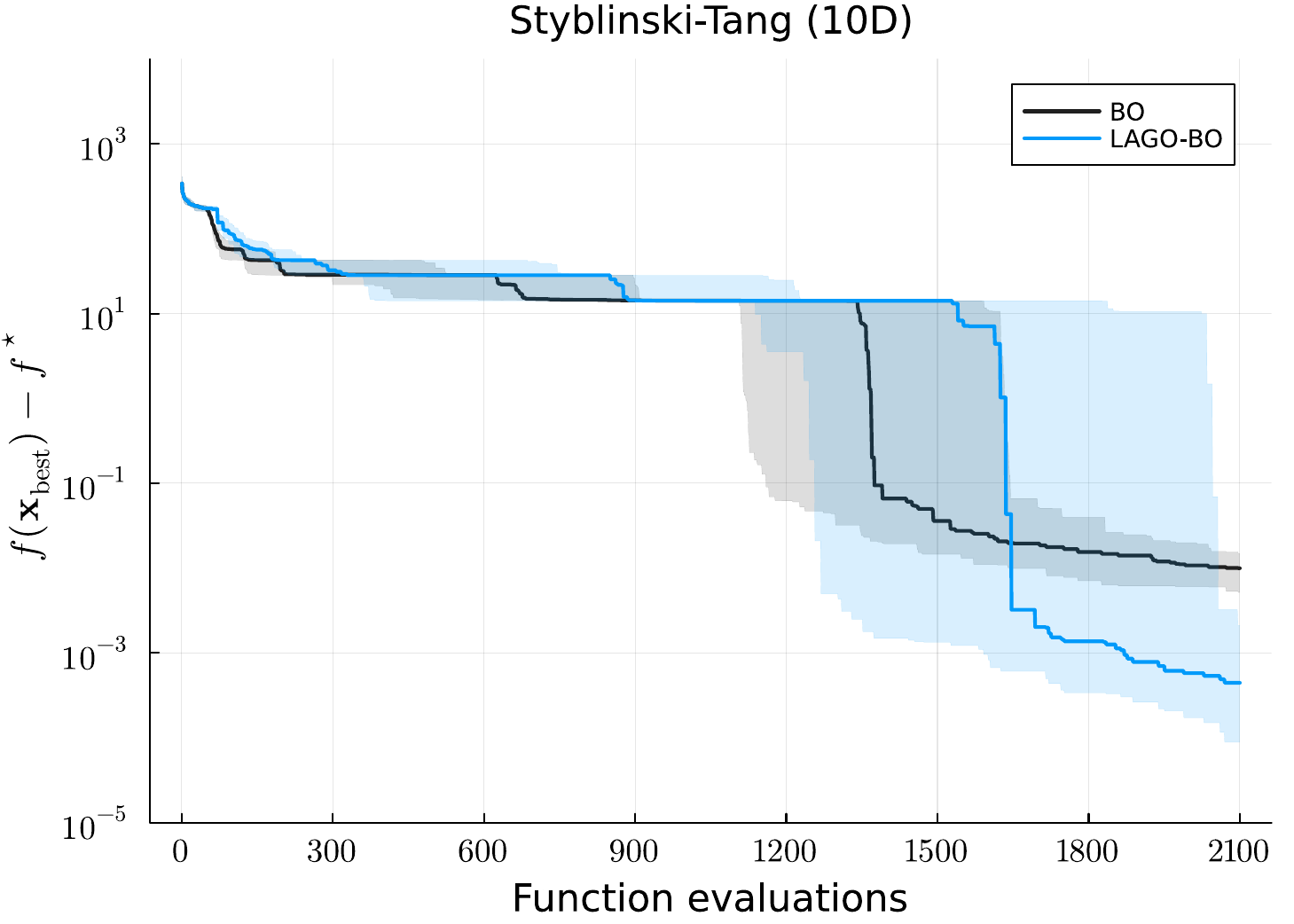}
\caption{\textbf{RQ4.} Median optimization error ($\pm$ IQR) on Rastrigin and Griewank (2D), and Styblinski-Tang (10D). 
LAGO-BO closely matches BO, indicating negligible overhead in highly multimodal
regimes, and outperforms BO in 10D on this example. This confirms that LAGO adapts to the problem structure and
reverts to BO-like behavior when exploration dominates.
}
\label{fig:RQ4}
\end{figure*}

\section{Conclusion}
We introduced \emph{LAGO}, a framework that couples Bayesian optimization with
gradient-based trust region local refinement through a predictive allocation rule, enabling
adaptive balance between global exploration and local exploitation without
predefined switching heuristics.

Our results support the key components of the method (\hyperlink{C_1}{C$_1$}-\hyperlink{C_4}{C$_4$}). The selection criterion
\eqref{eq:global_condition} enables effective interleaving of global and local
steps, providing a robust alternative to phase-based strategies without
problem-dependent tuning, while also adapting LAGO's behavior between strong local refinement in
convex or weakly multimodal regimes and BO-like behavior in highly multimodal
ones.
Furthermore, the strict local-global decoupling yields a conditional reduction to the
underlying trust-region method (\Cref{prop:local_convergence}).
Finally, the assimilation criterion \eqref{eq:lengthscale_criterion} mitigates ill-conditioning by avoiding
near-duplicate samples during local refinement. 
Across synthetic benchmarks and a PDE-constrained optimization application,
LAGO demonstrates consistent performance, being competitive with, or
outperforming, representative baselines while exhibiting lower variance.

Future work will extend LAGO beyond the low-dimensional, noise-free setting
considered here. In higher dimensions, this requires incorporating
functional structure to maintain surrogate fidelity, while stochastic objectives could be handled using noise-aware trust
regions \cite{sun2023trust}. 
\bibliography{biblio}

\appendix

\section{Extended related works}\label[appendix]{sec:related_works}

Several methods explicitly combine BO with trust regions, or other local optimization techniques.
TREGO \cite{diouane2023trego}, for instance, integrates a trust region algorithm with BO, and triggers local steps only when a global step fails to decrease the objective value sufficiently.
While effective, this failure-driven switching may delay local refinement, even
when a promising basin has already been identified, and ties local refinement to the inability of the global strategy to improve, rather than to an
explicit comparison of local and global improvement potential. 

LABCAT \cite{visser2025labcat} is another trust region BO technique that relies on rotation alignment of the box trust region based on principal components.
The method adaptively discards data points outside the active trust region, which reduces the numerical cost of GP conditioning.
However, as noted by the authors, LABCAT is best viewed as a robust local optimization routine rather than a truly global search.
Indeed, discarding points outside of the trust region prevents the global overview of the search space.

TuRBO \cite{eriksson2019scalable} addresses scalability by maintaining multiple
independent trust regions, each equipped with a local GP surrogate.
Although effective in high-dimensional settings, this approach fragments the global modeling of the objective function and relies on restarts to reintroduce global exploration.
Related extensions \cite{chen2025enhancing} incorporate Newton-type local solvers but similarly relocate trust regions only intermittently via global acquisition optimization.

BLOSSOM \cite{mcleod2018optimization} relies on global regret minimization, and combines global BO with local search. 
It enters the local phase once the regret condition is satisfied and then terminates after local refinement.

BADS \cite{acerbi2017practical} combines BO with mesh adaptive direct search \cite{audet2006mesh}, alternating between local BO searches and global exploration over a discretized mesh based on repeated local failures.

In contrast to these approaches, LAGO considers global and local candidates \emph{simultaneously} at every iteration and selects between them based on predicted improvement. 
This removes the need for explicit phase switching or failure detection and enables local exploitation as soon as it becomes competitive, including in early stages of the optimization. 
Moreover, LAGO explicitly decouples local refinement from global exploration. 
Local trust region steps are not influenced by the global GP surrogate, except
for the initialization of the trust region.
They are then allowed to proceed aggressively without causing ill-conditioning of
the global GP thanks to the selective incorporation of local information
which preserves a coherent global surrogate that continues to guide exploration.

\section{SR-1 Trust Region Algorithm} \label[appendix]{appendix:SR1}
The SR1 trust region algorithm is used as the local refinement
component in LAGO and follows \citet[Algorithm ~6.2]{nocedal2006numerical}.
The SR1-TR step at iteration $k$ is summarized in \Cref{alg:TR_algorithm}.
\begin{algorithm}
\caption{SR1 Trust Region algorithm step at iteration $k$}
  \label{alg:TR_algorithm}
  \begin{algorithmic}[1]
    \STATE {Input:} Radius $\Delta_{k-1}$, Maximum radius $\Delta_{\text{max}}$, Acceptance threshold $\eta \in (0, 0.001)$, the center $\mathbf{x}_{k-1}$.
    \STATE Compute $\mathbf{s}_k$ solving \eqref{eq:tr_subproblem}.
    \STATE Define $\mathbf{x}^+_k = \mathbf{x}_{k-1}+\mathbf{s}_k$.
    \STATE Evaluate $f(\mathbf{x}^+_k)$, and $\nabla f(\mathbf{x}^+_k)$.
    \STATE Compute the improvement ratio $\rho_k$ with \eqref{eq:improvement_ratio}.
    \IF{$\rho_k > \eta$}
        \STATE (accept) $\mathbf{x}_{k} \leftarrow \mathbf{x}^+_{k}$.    
    \ELSE
        \STATE (reject) $\mathbf{x}_{k} \leftarrow \mathbf{x}_{k-1}$.
    \ENDIF

    \IF{$\rho_k > 0.75$ and $\|\mathbf{s}_k\| > 0.8 \Delta_{k-1}$}
        \STATE $\Delta_{k} \leftarrow \min(\Delta_{\text{max}}, 2 \Delta_{k-1})$.
    \ELSE
        \IF{$\rho_k < 0.1$}
            \STATE $\Delta_{k} \leftarrow 0.5\Delta_{k-1}$.
        \ELSE
            \STATE $\Delta_k = \Delta_{k-1}$.
        \ENDIF
    \ENDIF
    \IF{\eqref{eq:condition_update_H} holds}
        \STATE Set $H_{k} \leftarrow$ SR1 update \eqref{eq:update_SR1}.
    \ELSE
        \STATE Set $H_{k} \leftarrow H_{k-1}$.
    \ENDIF
  \end{algorithmic}
\end{algorithm}

\section{Gradient-enhanced Gaussian Processes and gradBO} \label[appendix]{appendix:gradGPs}

One can extend the conditioning of the GP not only on the function evaluations but also on its gradients, resulting in gradient-enhanced GPs \cite{solak2002derivative}. 
Instead of modeling $f$ alone, we leverage the linearity of the gradient operator and place a joint GP prior over $[f,\nabla f]^\top$ \citep[Section~9.4]{rasmussen2005}
\begin{equation*}
\begin{aligned}
\begin{bmatrix}
    f \\
    \nabla f
\end{bmatrix}
&\sim \mathcal{GP}\Bigg(
\mathbf{m}_{\nabla}(\mathbf{x})
:= \begin{bmatrix}
    m(\mathbf{x}) \\
    \nabla m(\mathbf{x})
\end{bmatrix},
k_{\nabla}(\mathbf{x},\mathbf{x}')
:= \begin{bmatrix}
    k(\mathbf{x},\mathbf{x}') 
    & \nabla_{\mathbf{x}'} k(\mathbf{x},\mathbf{x}')^\top \\
    \nabla_{\mathbf{x}} k(\mathbf{x},\mathbf{x}')
    & \nabla_{\mathbf{x}} \nabla_{\mathbf{x}'}^\top k(\mathbf{x},\mathbf{x}')
\end{bmatrix}
\Bigg).
\end{aligned}
\end{equation*}

The posterior distribution of the vector $[f,\nabla f]^\top$ is again a GP when conditioned on $\{(\mathbf{x}_i,\mathbf{y}^\nabla_i)\}_{i=1}^n$ with
$\mathbf{y}^{\nabla}_i = [f(\mathbf{x}_i) , \nabla f(\mathbf{x}_i)]^\top + \boldsymbol{\xi}_i$ and $\boldsymbol{\xi}_i \sim
\mathcal{N}\!\left(\mathbf{0},\,\sigma^2 I_{d+1}\right).$
The posterior mean and covariance are, respectively,
\begin{equation*}
\begin{split}
    &\bar{\mathbf{m}}_{\nabla}(\mathbf{x}) := \mathbf{m}_{\nabla}(\mathbf{x}) + k_{\nabla,\mathbf{x},X}\tilde{K}_{\nabla}^{-1}(\mathbf{y}^\nabla-\mathbf{m}_{\nabla,X}),\\ 
    &\bar{k}_{\nabla}(\mathbf{x},\mathbf{x}') :=  k_{\nabla}(\mathbf{x},\mathbf{x}') - k_{\nabla,\mathbf{x},X}\tilde{K}_{\nabla}^{-1}k_{\nabla,X,\mathbf{x}'},
\end{split}
\end{equation*}
with $k_{\nabla,\mathbf{x},X} = [k_{\nabla}(\mathbf{x},\mathbf{x}_1),\ldots,k_{\nabla}(\mathbf{x},\mathbf{x}_n)]$, 
$\tilde{K}_\nabla := K_\nabla + \sigma^2 I_{n(d+1)}$ and the block-matrix $K_\nabla := (k_{\nabla}(\mathbf{x}_i,\mathbf{x}_j))_{i,j}$ is made of $n\times n$ blocks of size $d+1$.
Inverting this matrix hence costs $O(n^3(d+1)^3)$ which can be prohibitive in high dimensional settings.
 Gradient-enhanced GPs have been used in BO, and yield gradient-enhanced BO (gradBO) \cite{cheng2023gradient,wu2017bayesian}.

In particular, LAGO can be instantiated with standard GP
models (LAGO-\emph{BO}), or extended with gradBO (LAGO-\emph{grad}). 
The algorithmic structure, including the selection criterion and the
separation between global and local steps, remains similar.
Indeed, as gradBO leverages gradients at each query, one would therefore always perform function and gradient evaluations, no matter if we are doing a global or local step, compared to LAGO-BO where 
gradients are only queried when (i) performing a local step, or (ii) when recentering the trust region after a better solution is found using a global step.
Another possibility would be to condition gradient-enhanced GPs only on local
gradients queried during refinement, while conditioning solely on function
evaluations during BO steps. We do not explore this variant here.

In the context of LAGO, gradient-enhanced surrogates can improve performance in
ill-conditioned problems where gradient information significantly refines the
global model (see \hyperlink{RQ_3}{\textbf{RQ3}}). However, this is not required for the effectiveness of the
framework (LAGO-BO still converges without this information).

\section{Reduction of LAGO to SR1 Trust-Region Method}\label[appendix]{appendix:local_convergence}

This appendix provides the proof of \Cref{prop:local_convergence}.

\begin{proof}[Proof of \Cref{prop:local_convergence}]
By construction, LAGO alternates between global BO steps and local
trust region updates. A BO recentering step relocates the trust region center
if it found a better minimizer, while local steps follow the SR1 trust region
procedure.

By assumption, there exists an iteration $K$ such that no BO recentering occurs
for all $k \ge K$. Therefore, from iteration $K$ onward, all updates of the
trust region center are performed exclusively through the local refinement
mechanism $\operatorname{SR1-TR}$. In particular, the step computation, acceptance rule, trust region
radius updates, and Hessian updates coincide exactly with those of the SR1
trust region method.

It follows that, for all $k \ge K$, the sequence $\{\mathbf{x}_c^k\}_{k \ge K}$ generated
by LAGO is identical to that produced by the SR1 trust region method initialized
at $(\mathbf{x}_c^K,\Delta_K,H_K)$, namely $\operatorname{SR1-TR}(\mathbf{x}_c^K,\Delta_K, H_K)$.

The convergence result then follows from standard theory for SR1 trust region
methods (see, e.g., \citep[Theorems~4.2, 6.7]{nocedal2006numerical}). 
Particularly, \citep[Theorem~6.7]{nocedal2006numerical} ensures that $\{\mathbf{x}_c^k\}_{k \ge K}$ converges superlinearly to a stationary point 
$\mathbf{x}^\star$.
\end{proof}
\begin{remark}
The result above characterizes the asymptotic behavior of the algorithm in the
absence of stopping due to the early stopping criterion or budget restrictions, and therefore isolates the intrinsic convergence
properties of the local refinement mechanism.
\end{remark}

\section{Synthetic benchmarks and PDE-constrained optimization definitions}\label[appendix]{appendix:benchmarks_pde_opt_problem}

We discuss here the synthetic benchmarks used throughout the manuscript, as well as the PDE-constrained problem used in \Cref{sec:experiments}.

\subsection{Synthetic benchmarks}\label[appendix]{appendix:synthetic_problems}

We consider a set of standard synthetic benchmarks that capture different
optimization challenges. The \textbf{Branin} function is a classical
low-dimensional test problem with three minima (all global ones). The \textbf{perturbed Branin} variant
introduces a slight bias toward one minimizer, enabling the assessment of basin
identification and sensitivity to small variations in objective values.
The \textbf{Styblinski--Tang} (ST) function is multimodal with a growing number of local minima as $d$ grows,
and is considered in $2$-$5$-$10$D to evaluate scalability and
behavior in moderately multimodal settings. The \textbf{Rosenbrock} function
features a narrow valley with strong ill-conditioning, making it a
canonical test for local refinement and optimization in anisotropic landscapes.
The \textbf{Lévy} function combines multimodality with flat regions and sharp
basins, providing a challenging setting where both global exploration and local
refinement are required. To assess robustness in highly multimodal regimes, we include the
\textbf{Rastrigin} and \textbf{Griewank} functions, both characterized by a
very large number of regularly distributed local minima. These functions are known
to challenge global optimization methods, as the landscape is dominated by
frequent oscillations that can hinder convergence.

Together, these benchmarks allow us to evaluate LAGO across a range of regimes, from smooth and
ill-conditioned problems to highly multimodal landscapes, highlighting the
trade-offs between global exploration and local refinement. We report in \Cref{table:synthetic_problems} the function definitions, domains and global
minima of the functions, and \Cref{fig:2D_contours} shows contours and surfaces of the $2$D
functions.

\begin{table*}[ht]
    \tiny
\begin{center}
\caption{Information about the synthetic test functions used in \Cref{sec:experiments}.}
\label{table:synthetic_problems}
\begin{tabular}{c c c c} \toprule
    Objective & Definition $f: \mathbb{R}^d \to \mathbb{R}$ & Domain &  Global minimum \\ \midrule
    Branin        & $\left(\mathbf{x}_2-\frac{5.1}{4\pi^2}\mathbf{x}_1^2+\frac{5}{\pi}\mathbf{x}_1-6\right)^2
    +10\left(1-\frac{1}{8\pi}\right)\cos(\mathbf{x}_1)+10$ & $[-5, 10] \times [0, 15]$ & $0.39789$\\ 
    Styblinski-Tang & $\frac{1}{2}\sum_{i=1}^d \left(\mathbf{x}_i^4 - 16 \mathbf{x}_i^2 + 5 \mathbf{x}_i \right)$ &$[-5, 5]^d$ & $-39.16599d$\\ 
    Rosenbrock & $(1-\mathbf{x}_1)^2 + 100(\mathbf{x}_2 -\mathbf{x}_1^2)^2$& $[-5,10]^2$& $0.0$\\ 
    Lévy & $\sin^2(\pi w_1)
    + (w_1-1)^2\left[1+10\sin^2(\pi w_1+1)\right]$ & $[-10,10]^2$ & $0.0$ \\
& $+(w_2-1)^2\left[1+\sin^2(2\pi w_2)\right],\;
    w_i = 1+\frac{\mathbf{x}_i-1}{4}$ &&\\ 
    Perturbed Branin & $\operatorname{Branin}(\mathbf{x}) + 10^{-6}\|\mathbf{x} - [-\pi,12.275]\|^2$ & $[-5, 10] \times [0, 15]$ & $0.39789$\\
    Rastrigin & $10d + \sum_{i=1}^d \mathbf{x}_i^2 - 10 \cos (2\pi \mathbf{x}_i)$& $[-50, 50]^d$ & $0.0$\\
    Griewank & $\sum_{i=1}^d \frac{\mathbf{x}_i^2}{4000} - \prod_{i=1}^d \cos (\frac{\mathbf{x}_i}{\sqrt{i}})$& $[-50, 50]^d$ & $0.0$\\
\bottomrule
\end{tabular}
\end{center}
\end{table*}

\begin{figure*}[ht]
\centering
\includegraphics[width=0.32\textwidth]{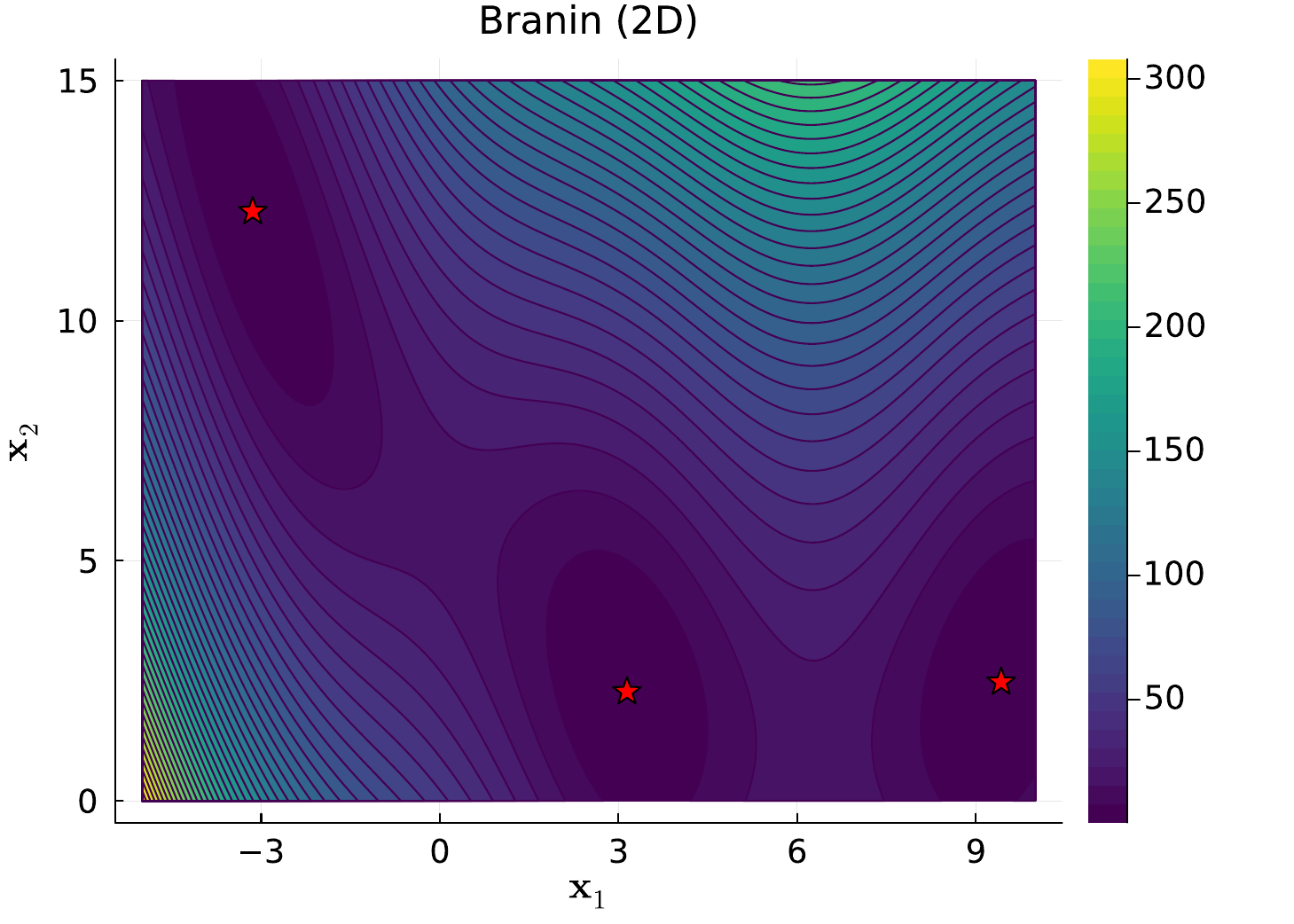}
\includegraphics[width=0.32\textwidth]{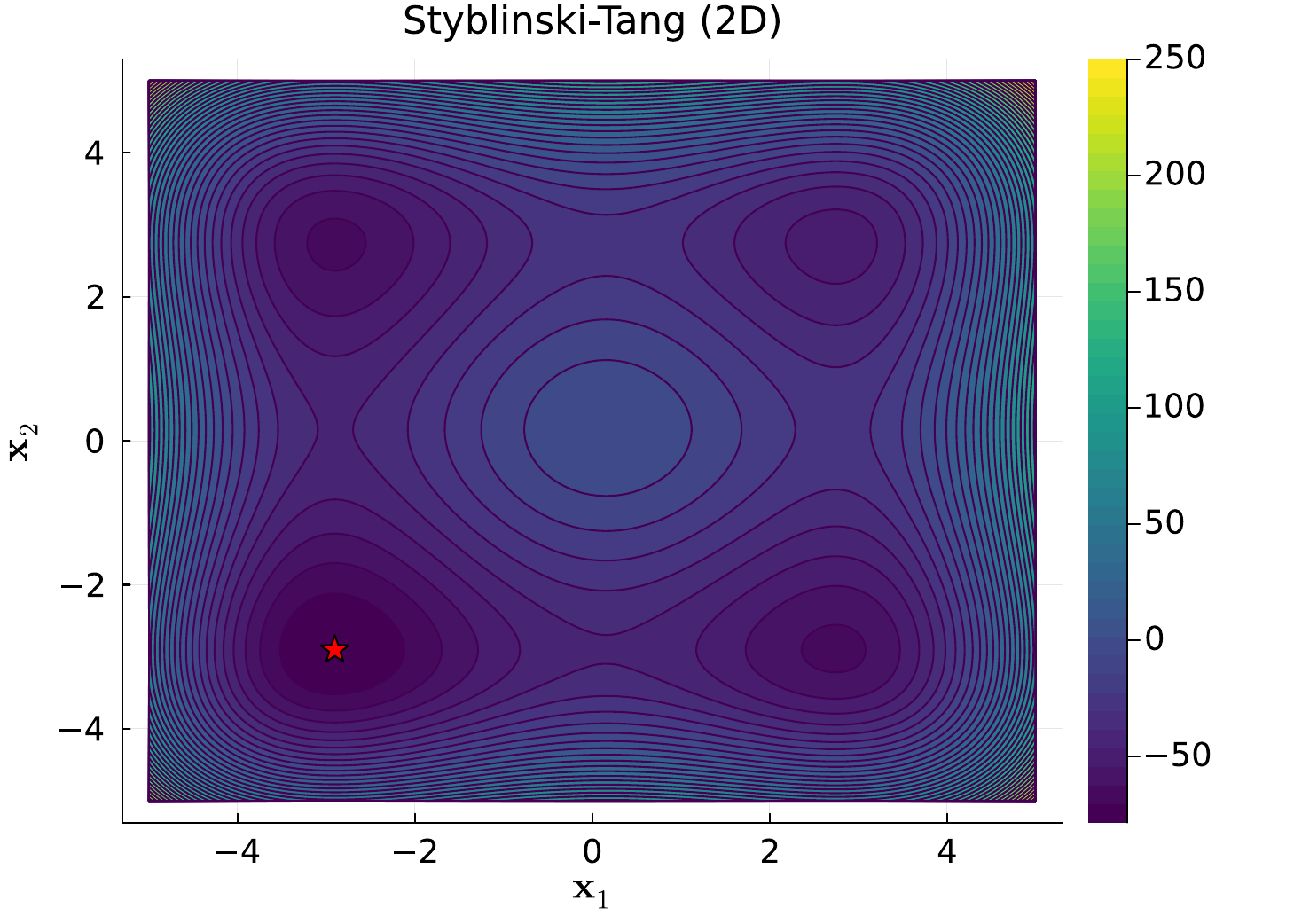}
\includegraphics[width=0.32\textwidth]{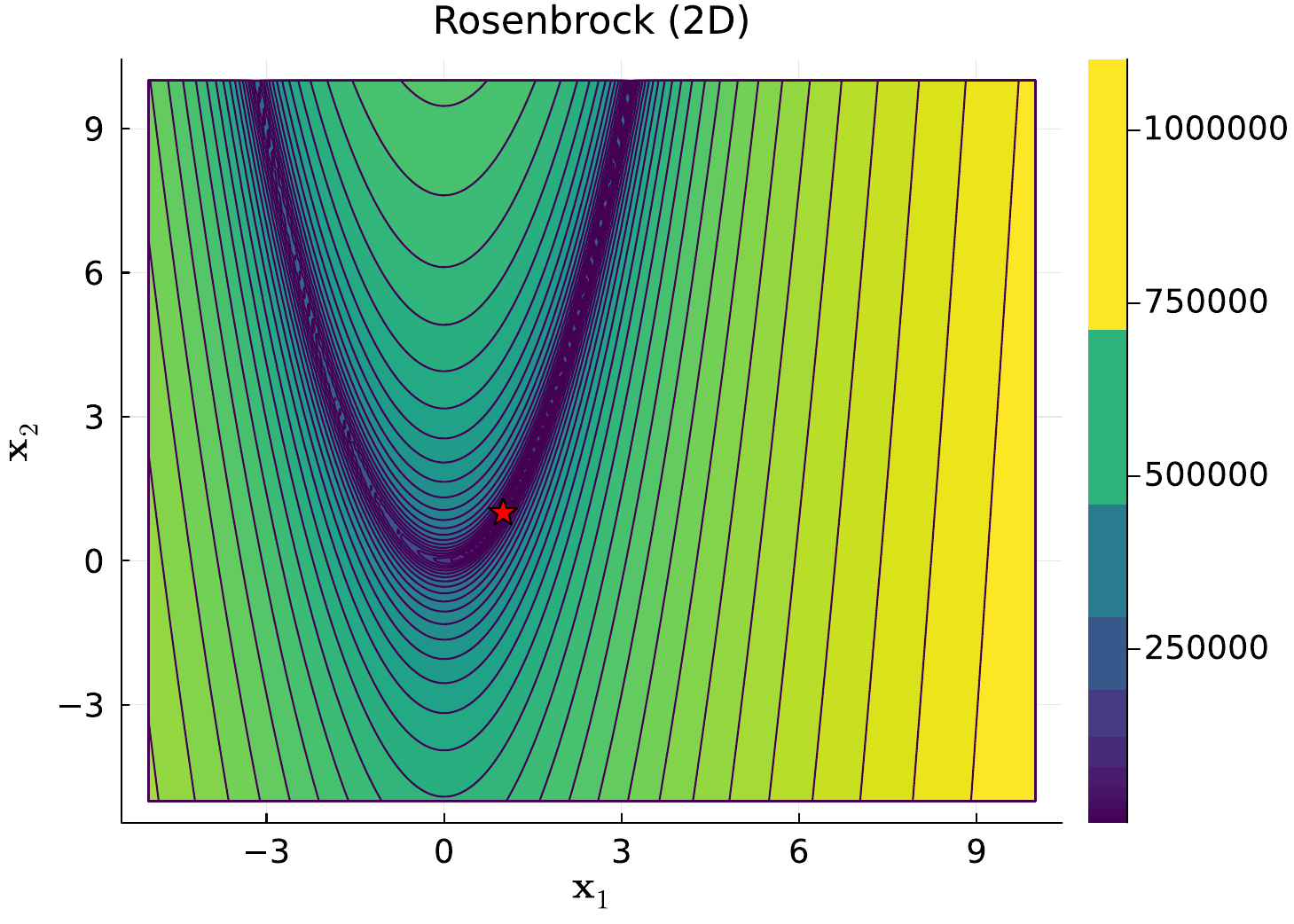}
\includegraphics[width=0.32\textwidth]{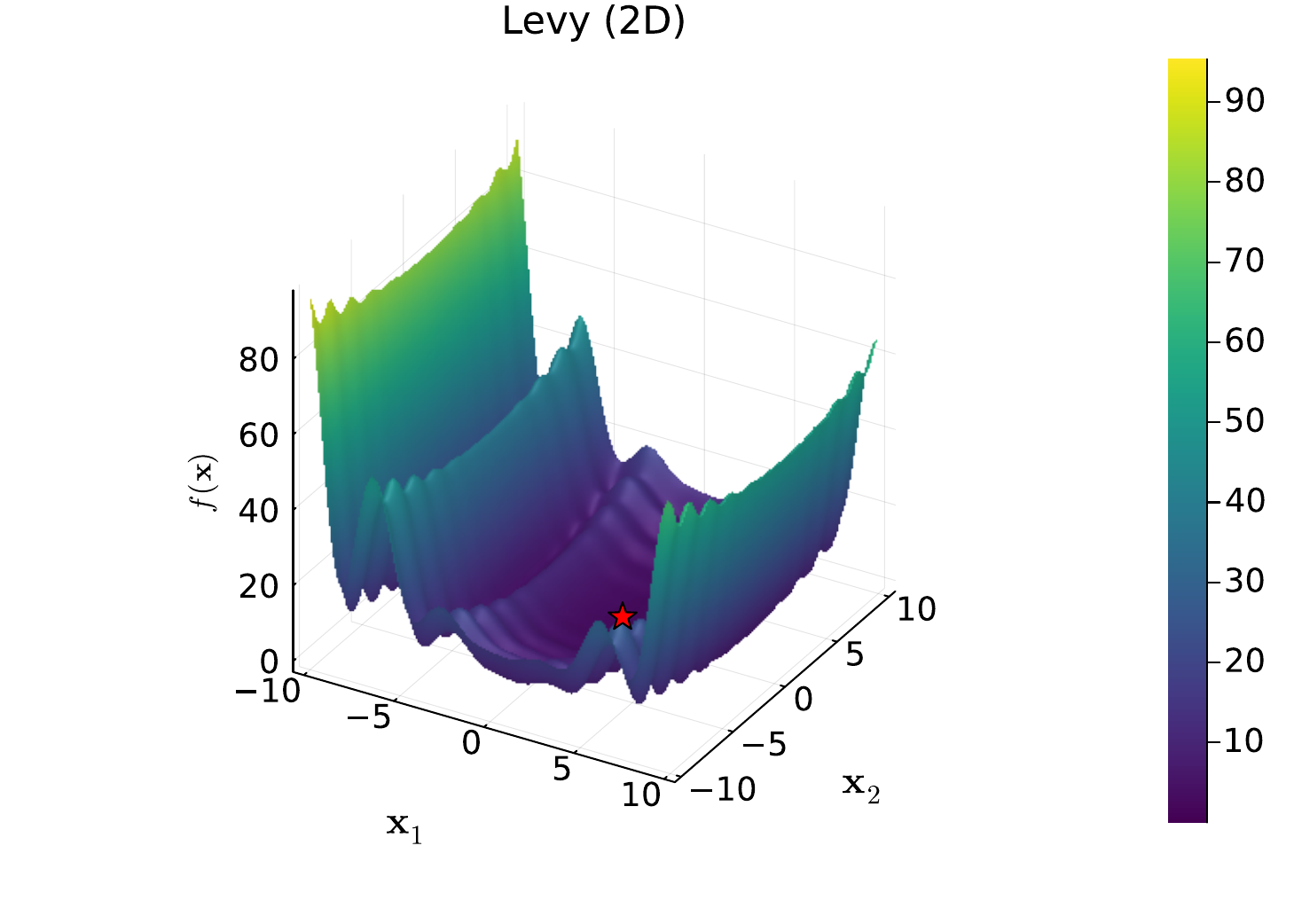}
\includegraphics[width=0.32\textwidth]{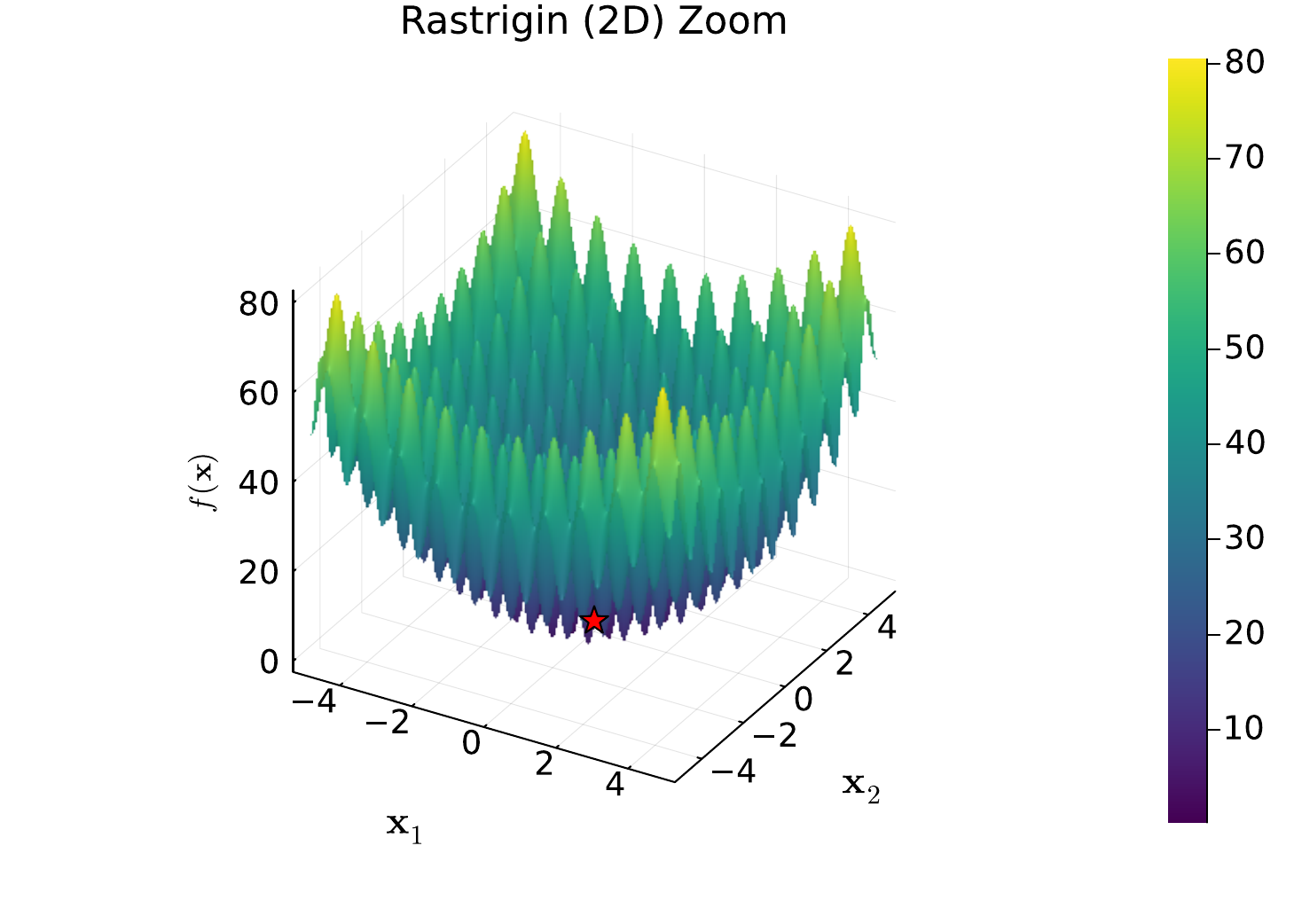}
\includegraphics[width=0.32\textwidth]{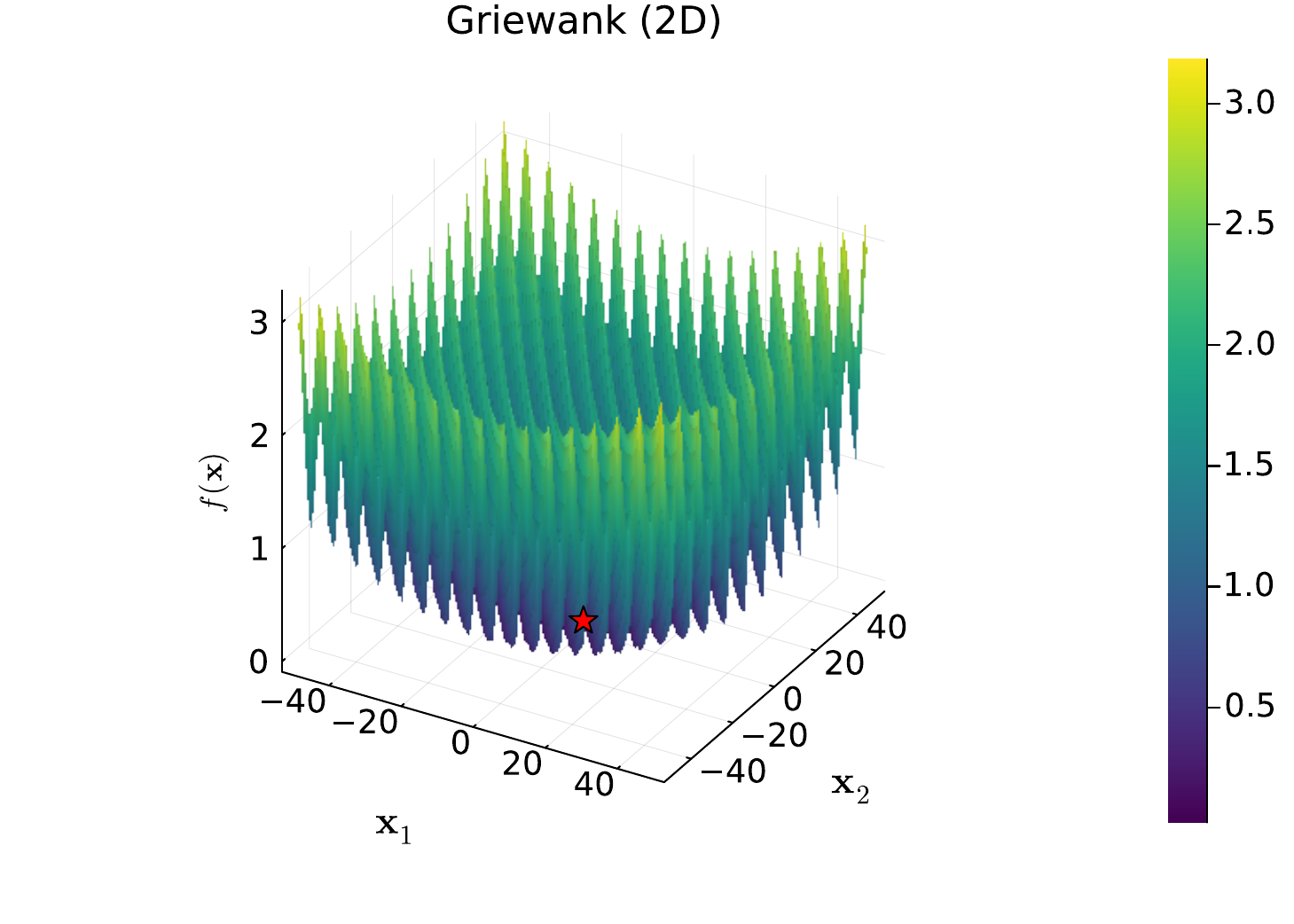}
\caption{Contour and surface plots of the 2D synthetic test functions. The red
stars show the global minima. Surface plots are used for highly multimodal
functions for clarity purposes, and the Rastrigin one is zoomed in a smaller domain $[-5,5]^2$.}
\label{fig:2D_contours}
\end{figure*}

\subsection{PDE-constrained optimization problem}\label[appendix]{appendix:pde_constrained_problem}

We now provide the problem description and specifications of the PDE-constrained optimization problem studied in \Cref{sec:experiments}.

\textbf{Problem setting.} Let $\mathcal{D} \subset \mathbb{R}^d$ be a spatial domain. 
We consider the optimization problem
\begin{equation}\label{eq:pde_problem}
\begin{aligned}
\min_{\mathbf{c} \in \mathcal{D}} \quad & \tilde{J}(y,\mathbf{c}):= \frac{1}{2} \|y-y_d\|^2_{L^2(\mathcal{D})}&\\
&\textrm{such that } y \in H_0^1(\mathcal{D}) \textrm{ satisfies} \\ 
& \int_{\mathcal{D}} \kappa(\mathbf{x})\nabla y(\mathbf{x}) \cdot \nabla v(\mathbf{x})d\mathbf{x} = \int_{\mathcal{D}}u(\mathbf{x},\mathbf{c})v(\mathbf{x})d\mathbf{x} \quad \forall v \in H^1_0(\mathcal{D}),
\end{aligned}
\end{equation}
where $u(\mathbf{x},\mathbf{c}) := \alpha\exp(-\beta
\|\mathbf{x}-\mathbf{c}\|^2_2)$ with $\alpha,\beta >0$, $\mathbf{c}$ represents
the location of the source term, $y_d : \mathcal{D} \to \mathbb{R}$ is the
desired state, and $\kappa : \mathcal{D} \to \mathbb{R}$ is the diffusion
coefficient. 
We define $J(\mathbf{c}):=\tilde{J}(S(\mathbf{c}),\mathbf{c})$ as the reduced cost function, with $S:\mathcal{D}\rightarrow H^1_0(\mathcal{D})$ being the solution map $\mathbf{c} \mapsto y(\mathbf{c})$ associated with the PDE constraint.
Using the adjoint-state method, the gradient $\nabla J$ can be computed at the additional cost of solving a single adjoint PDE \cite{hinze2008optimization}.
The adjoint state variable $p \in H^1_0(\mathcal{D})$ is obtained by solving 
\begin{equation*}
    \int_{\mathcal{D}} \kappa(\mathbf{x})\nabla p(\mathbf{x}) \cdot \nabla v(\mathbf{x})d\mathbf{x} = \int_{\mathcal{D}} (y(\mathbf{x})-y_d(\mathbf{x}))v(\mathbf{x})d\mathbf{x}, \quad \forall v \in H^1_0(\mathcal{D}).
\end{equation*}
Using a Lagrangian approach (see \citet{hinze2008optimization} for an overview), the gradient reads
\begin{equation*}
    (\nabla J(\mathbf{c}))_i =  - \left(\nabla_{\mathbf{c}} u(\cdot,\mathbf{c})_{i},p \right)_{L^2(\mathcal{D})},
\end{equation*}
with $\nabla_{\mathbf{c}} u(\cdot,\mathbf{c})=2\alpha \beta \exp(-\beta\|\cdot-\mathbf{c}\|^2_2)(\cdot-\mathbf{c})$.
As a result, evaluating the objective function and its full gradient incurs a cost comparable to two PDE solves, independently of the dimension $d$ of the control variable.
This makes PDE-constrained optimization problems with adjoint-based gradients
particularly well suited for LAGO, which explicitly leverages gradient
information throughout its optimization steps, either only for local steps
(LAGO-BO) or also for its gradient-enhanced GP in LAGO-grad.

We now discuss the different parameters used in this problem. As $J$ is very smooth, we decided to use a Matérn $7/2$ kernel whenever possible (BO, gradBO and LAGO(-BO/grad)).

\textbf{Finite Element Method framework.} For this study, we meshed the physical domain $\mathcal{D} := [0, 1]^2$ with $50 \times 50$ squares, further split into two to obtain triangular elements, and used $P^1$ finite elements.

\textbf{Functions used in the system.}
The diffusion coefficient is given as 
$$\kappa(\mathbf{x}) := \exp\bigl(\exp(-1.125)(\cos(\pi \mathbf{x}_1)\sin(\pi \mathbf{x}_2)
                                               +\cos(2\pi \mathbf{x}_1)\sin(\pi \mathbf{x}_2)
                                               +\cos(2\pi \mathbf{x}_1)\sin(2\pi \mathbf{x}_2)\bigl),$$
we set $\alpha = 5 \times 10^{2}$ and $\beta = 5 \times 10^{3}$ for the source term $u$, and the desired state $y_d$ is 
$$y_d(\mathbf{x}) :=  \exp(2\mathbf{x}_1 + 2\mathbf{x}_2)\sin(4\pi \mathbf{x}_1)\sin(4\pi\mathbf{x}_2).$$

The landscape of the cost function of \eqref{eq:pde_problem} is given in
\Cref{fig:cost_function}, which illustrates the presence of local minima and a
global minimum located near the top-right of the search domain. All the
quantities described above are shown in \Cref{fig:pde_constrained_setting}, as
well as the obtained state variable $y$ when we place the source term at
$\mathbf{c}^\star \approx (0.89, 0.89)$.
\begin{figure*}
\centering
\includegraphics[width=0.6\textwidth]{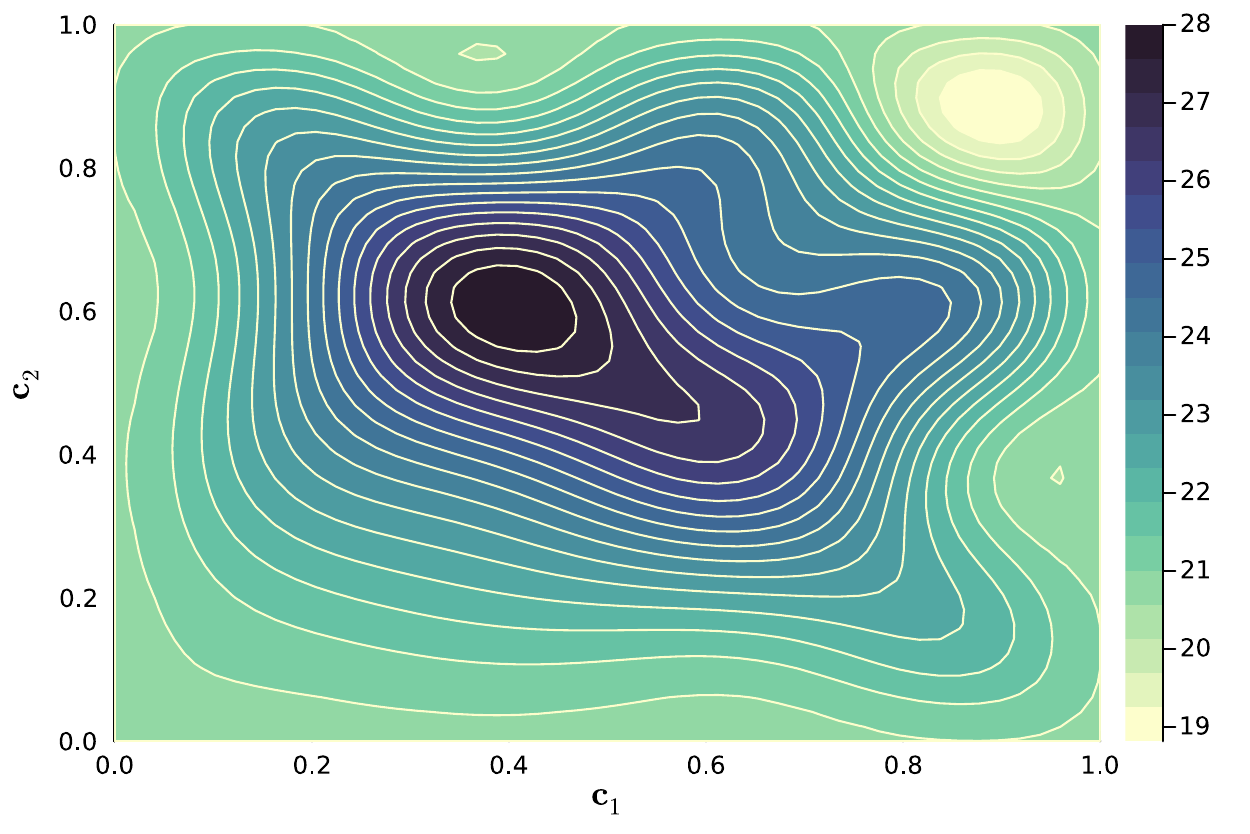}
\caption{Landscape of the reduced cost function $J$ for the
PDE-constrained optimization problem.
The objective exhibits multiple local minima and a distinct global minimizer,
highlighting the need for global exploration combined with local refinement.}
\label{fig:cost_function}
\end{figure*}

\begin{figure*}[ht]
\centering
\includegraphics[width=0.48\textwidth]{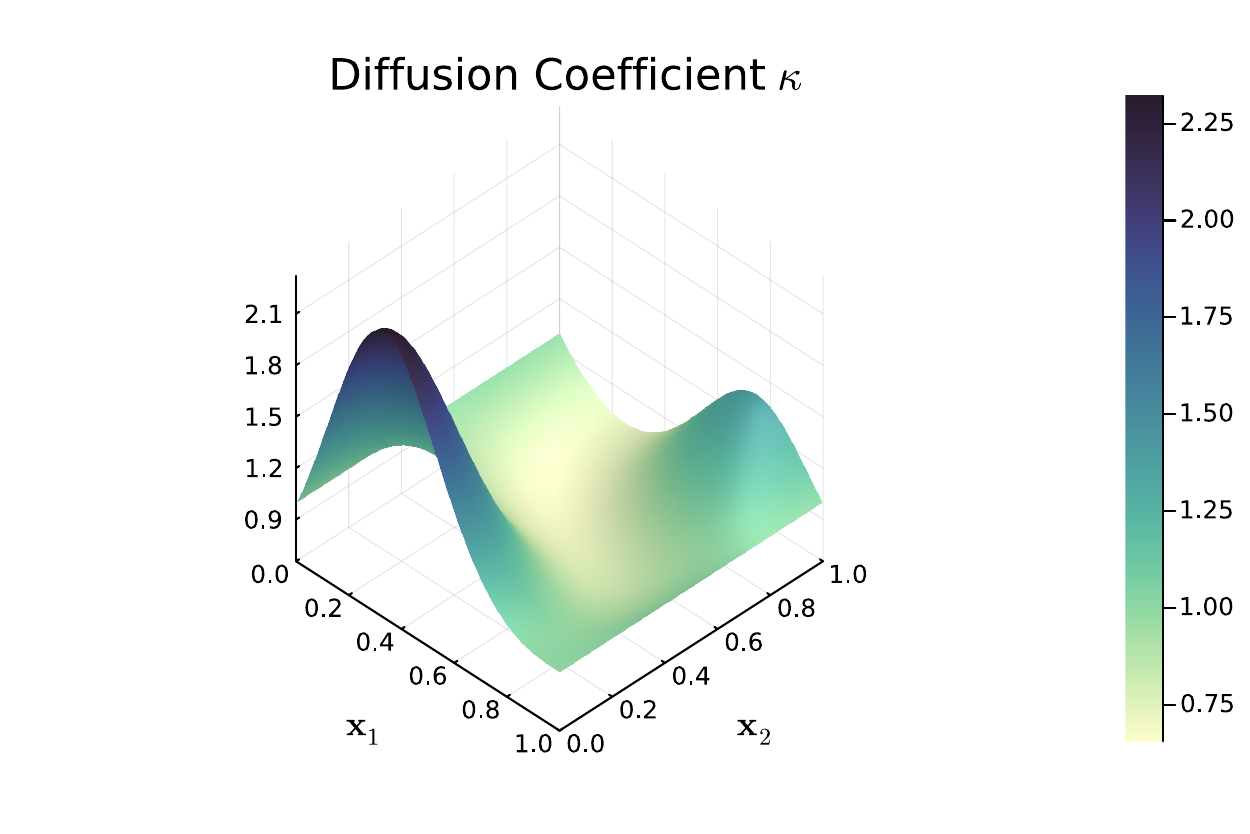}
\includegraphics[width=0.48\textwidth]{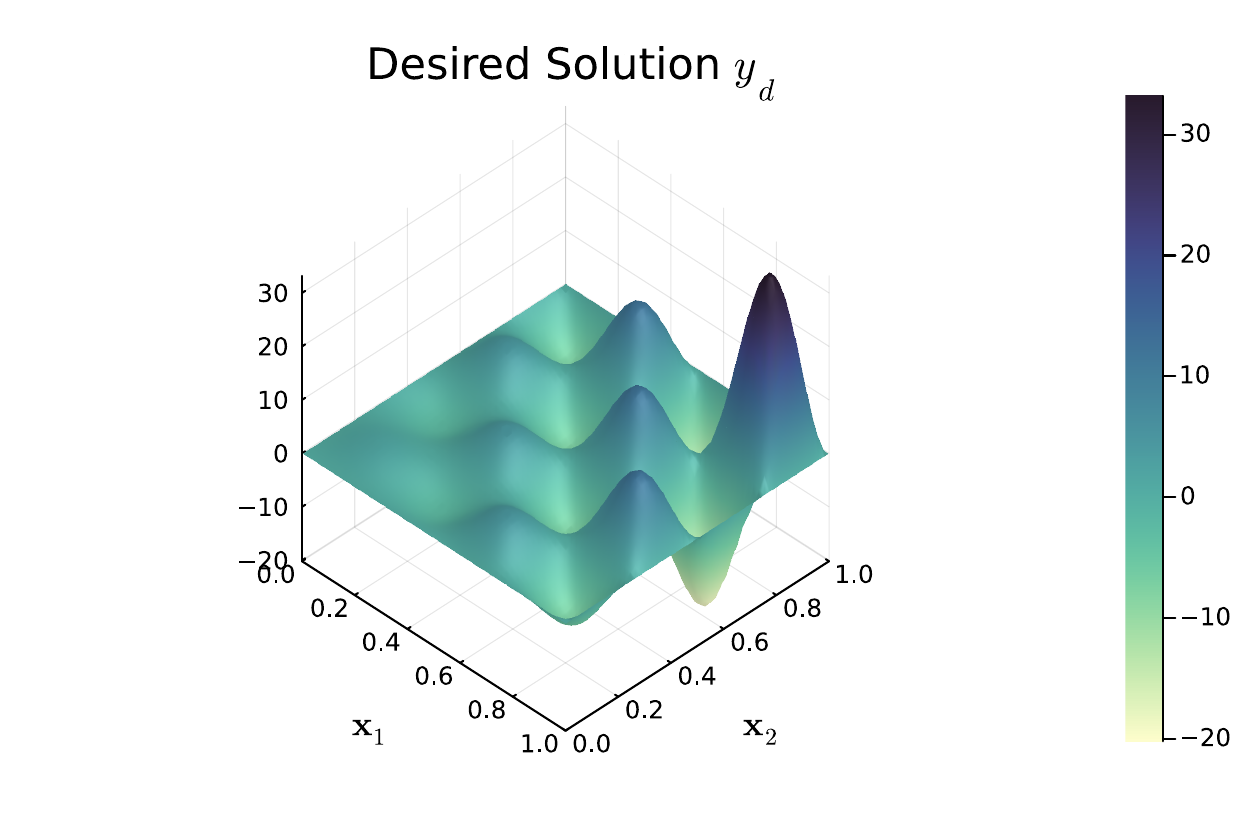}\\
\includegraphics[width=0.48\textwidth]{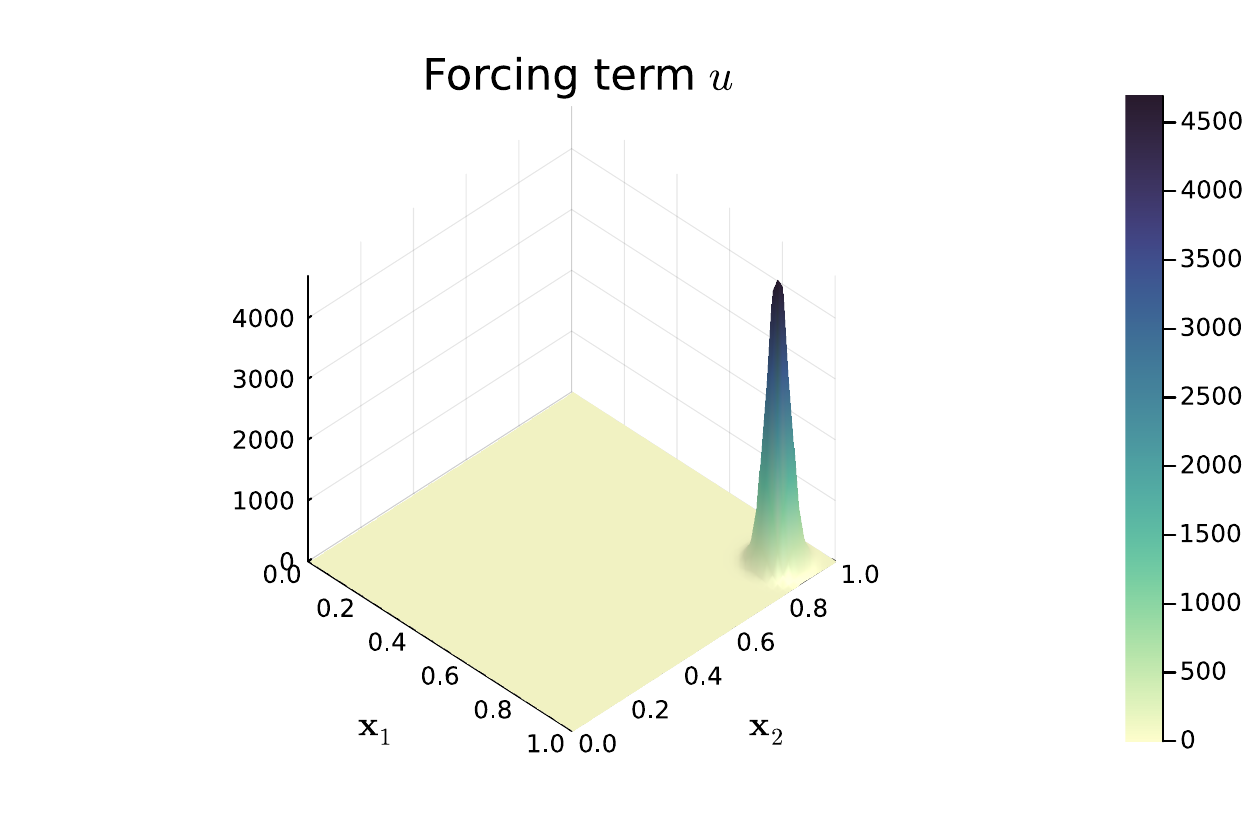}
\includegraphics[width=0.48\textwidth]{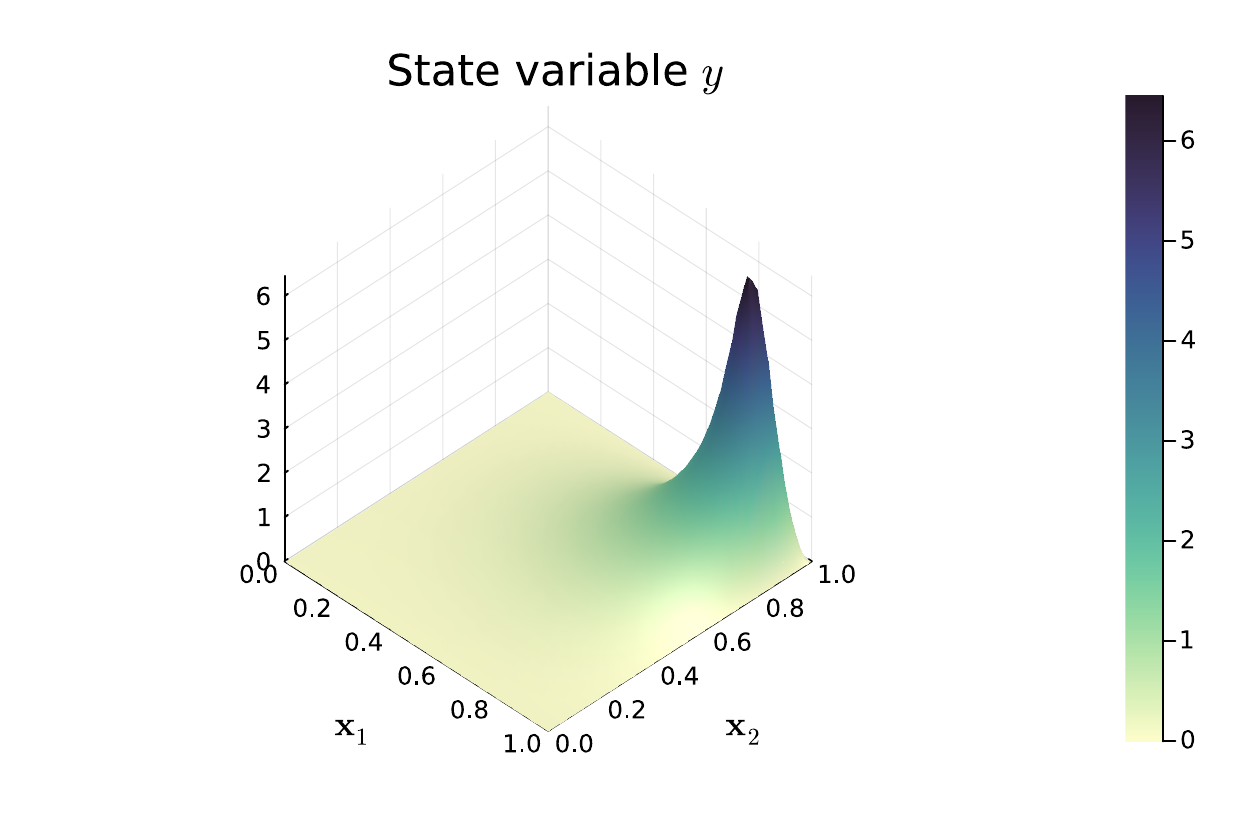}
\caption{Diffusion coefficient $\kappa$, desired state $y_d$, optimal forcing
term $u$, and obtained state variable $y$. The forcing term is approximately
placed at the highest point of the desired state $y_d$, as the resulting state
$y$ will minimize the $L_2$ distance with $y_d$.}
\label{fig:pde_constrained_setting}
\end{figure*}

\section{Sensitivity analysis of hyperparameters}\label[appendix]{appendix:sensitivity}

We discuss here the sensitivity of LAGO to its main hyperparameters ($\nu, \varepsilon_T, \gamma$). The
parameter $\gamma$ is analyzed in
\Cref{appendix:sensitivity_gamma}, the early stopping threshold
$\varepsilon_T$ in \Cref{appendix:sensitivity_eps}, and the parameter $\nu$
is discussed in the main text (\hyperlink{RQ_2}{\textbf{RQ2}}).

\subsection{Analysis of local-global trade-off}\label[appendix]{appendix:sensitivity_gamma}

We illustrate in \Cref{fig:gamma_sensitivity} the effect of the parameter
$\gamma$ in \eqref{eq:global_condition} on the Lévy function. Increasing
$\gamma$ favors local refinement over global exploration. We observe that
larger values of $\gamma$ lead to improved convergence on this benchmark,
indicating that earlier or more frequent local steps can be beneficial in this
setting. Overall, the performance remains stable across a range of $\gamma$ values,
highlighting the robustness of LAGO with respect to this parameter.

\begin{figure*}
\centering
\includegraphics[width=0.6\textwidth]{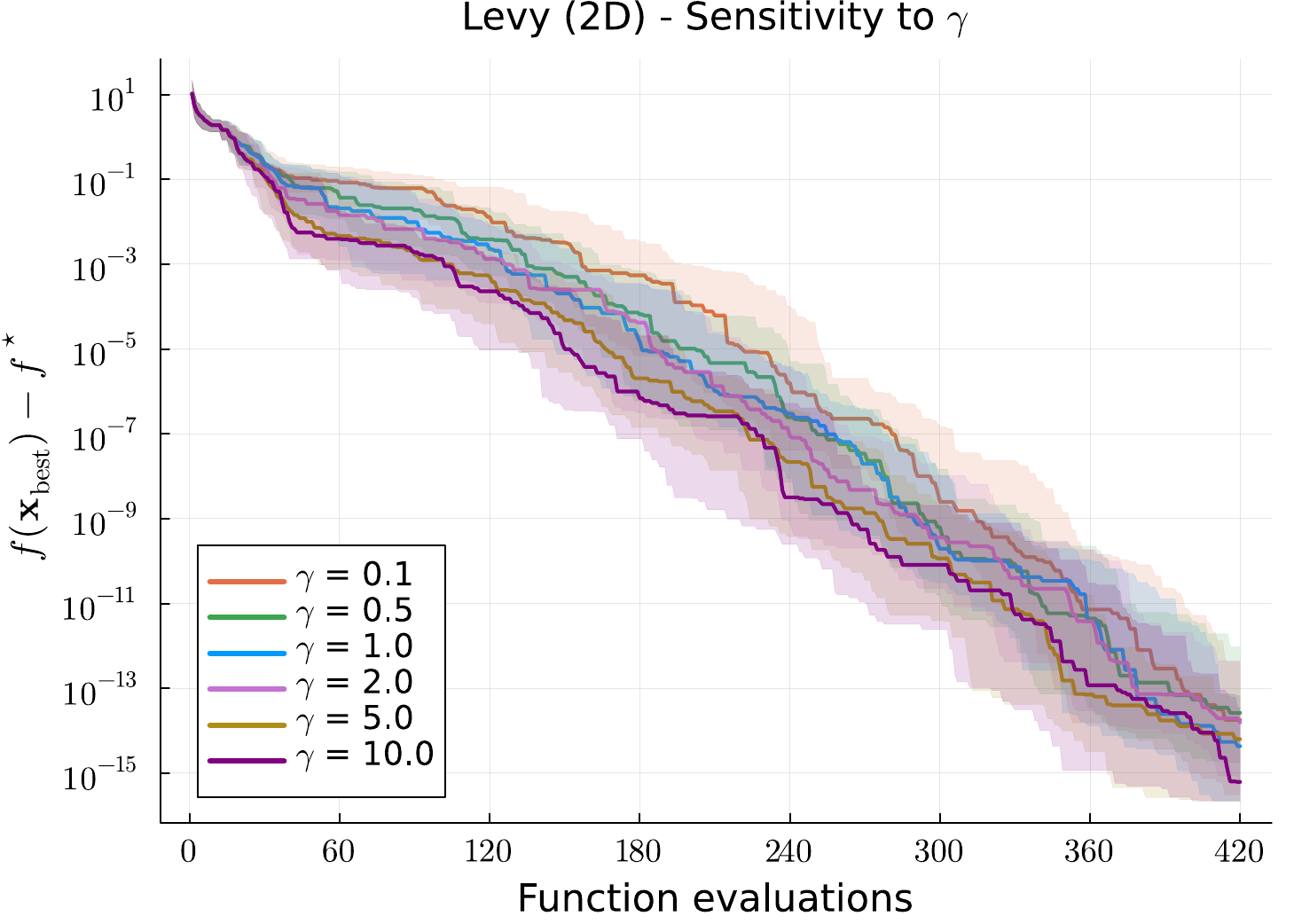}
\caption{$\boldsymbol{\gamma}$\textbf{-ablation.} Performance of LAGO (Median $\pm$ IQR) on the
Lévy function for different values of $\gamma$ in \eqref{eq:global_condition}
(default $\gamma=1$). Larger values of $\gamma$, which place greater emphasis
on local refinement, lead to lower error on this synthetic benchmark.}
\label{fig:gamma_sensitivity}
\end{figure*}

\subsection{Analysis of early stopping threshold}\label[appendix]{appendix:sensitivity_eps}

The tolerance $\varepsilon_T$ can be interpreted as a proxy to the remaining improvement one is willing to ignore.
We report in \Cref{table:epsilonT} the number of function evaluations performed by LAGO before the termination criterion is met on the Branin function, averaged over $50$ runs.
As expected, smaller values of $\varepsilon_T$ lead to stricter stopping criteria and therefore higher evaluation counts.

We further illustrate in \Cref{fig:epsilonT} the evolution of the objective
value across iterations for the same runs. Since termination occurs at
different iterations depending on the seed, trajectories are padded after
termination to match the longest run across seeds and values of $\varepsilon_T$.

This behavior confirms that $\varepsilon_T$ provides a direct control over the
trade-off between solution accuracy and evaluation cost.

\begin{figure*}[ht]
\centering
\begin{minipage}{0.48\textwidth}
\centering
\captionof{table}{Average number of function evaluations as a function of $\varepsilon_T$,
reported as mean $\pm$ standard deviation over $50$ runs for the Branin function. Smaller $\varepsilon_T$ leads to
stricter stopping criteria and higher evaluation counts.}
\label{table:epsilonT}
\begin{tabular}{ccc}
\toprule
$\varepsilon_T$ & Avg. evaluations & std \\
\midrule
$10^{-14}$ & 104.22 & 9.46 \\
$10^{-12}$ & 101.02 & 9.65 \\
$10^{-10}$ & 96.58  & 8.99 \\
$10^{-8}$  & 91.20  & 9.07 \\
$10^{-6}$  & 86.50  & 7.58 \\
$10^{-4}$  & 77.70  & 8.33 \\
$10^{-2}$  & 59.74  & 6.69 \\
$10^{-1}$  & 44.70  & 8.80 \\
\bottomrule
\end{tabular}
\end{minipage}
\hfill
\begin{minipage}{0.48\textwidth}
\centering
\includegraphics[width=\textwidth]{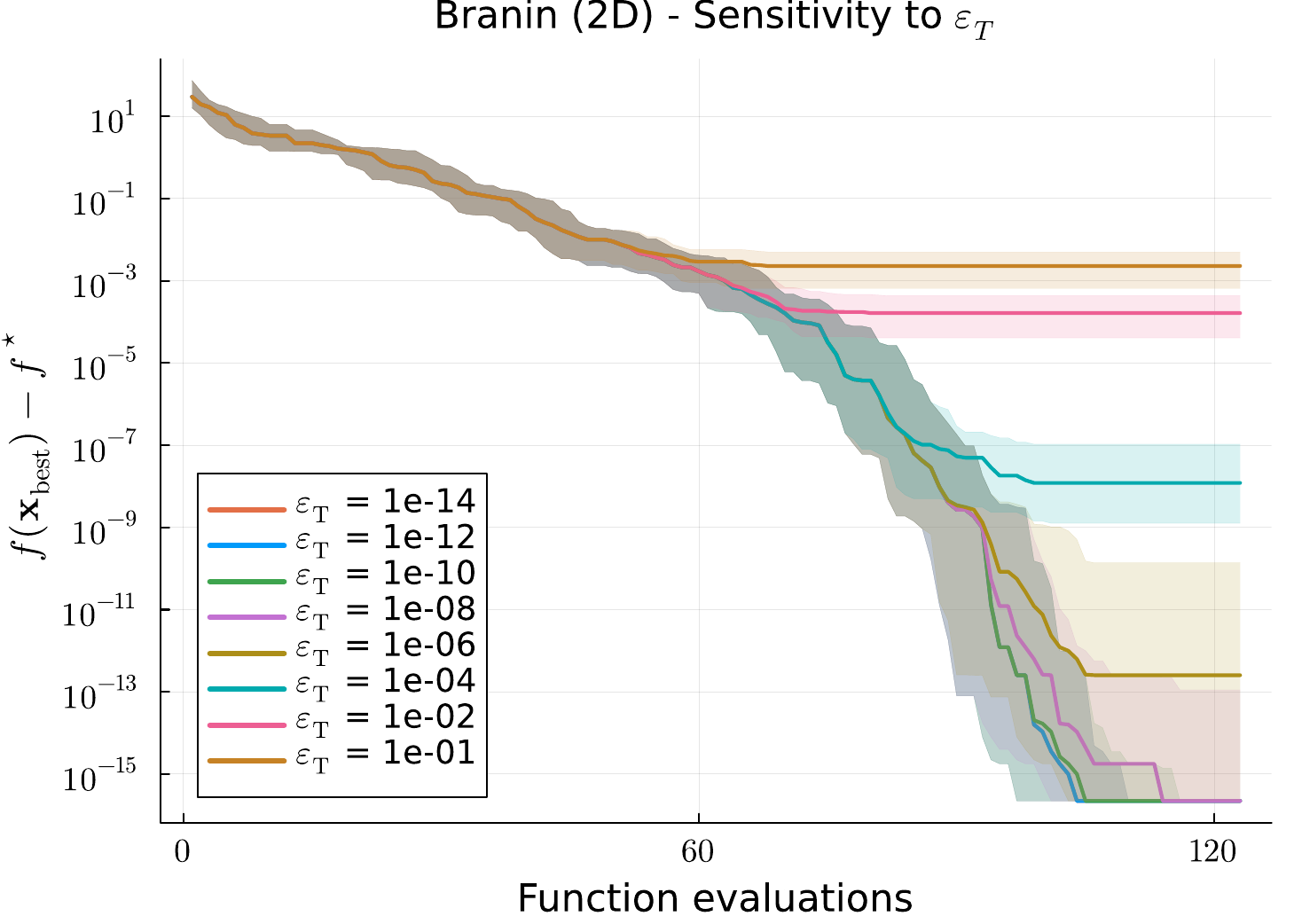}
\captionof{figure}{$\boldsymbol{\varepsilon_T}$\textbf{-ablation.} Effect of $\varepsilon_T$ on early stopping of LAGO on Branin. Smaller values enforce stricter stopping, and lead to lower regret.}
\label{fig:epsilonT}
\end{minipage}

\end{figure*}

\section{Local-global interleaving across benchmarks}\label[appendix]{appendix:interleaving}

The trajectories in \Cref{fig:local_global_interleaving} illustrate how LAGO
dynamically allocates evaluations between global exploration and local
refinement across different problem structures for a fixed random seed. Each trajectory highlights the
sequence of evaluations, distinguishing between global BO steps (blue) and local
trust region steps (green).

In multimodal settings such as Rastrigin, Lévy, Griewank and Styblinski-Tang in $5$-$10$D, LAGO predominantly selects global steps,
with only a small fraction of local refinements. This behavior reflects the
limited usefulness of local optimization when basin identification remains the
dominant challenge. In contrast, for convex or weakly multimodal problems
such as Branin, Rosenbrock, or low-dimensional Styblinski–Tang, LAGO increases the
frequency of local steps, exploiting gradient information to accelerate
convergence once a promising region has been identified.

This behavior is further quantified in \Cref{table:local_global_ratio}, which reports the proportion of
local steps across independent runs. The results confirm that LAGO consistently
adapts the balance between exploration and refinement to the underlying
landscape. In particular, the proportion of local steps remains low in
highly multimodal problems (typically below $5\%$), while increasing
significantly in settings where local refinement is beneficial.

Overall, these results provide empirical evidence that the selection criterion \eqref{eq:global_condition} effectively drives the interleaving mechanism, while the method
naturally adapts to the problem structure, reverting to BO-like behavior
when exploration dominates and leveraging local refinement when appropriate.

\begin{figure*}
\centering
\includegraphics[width=0.32\textwidth]{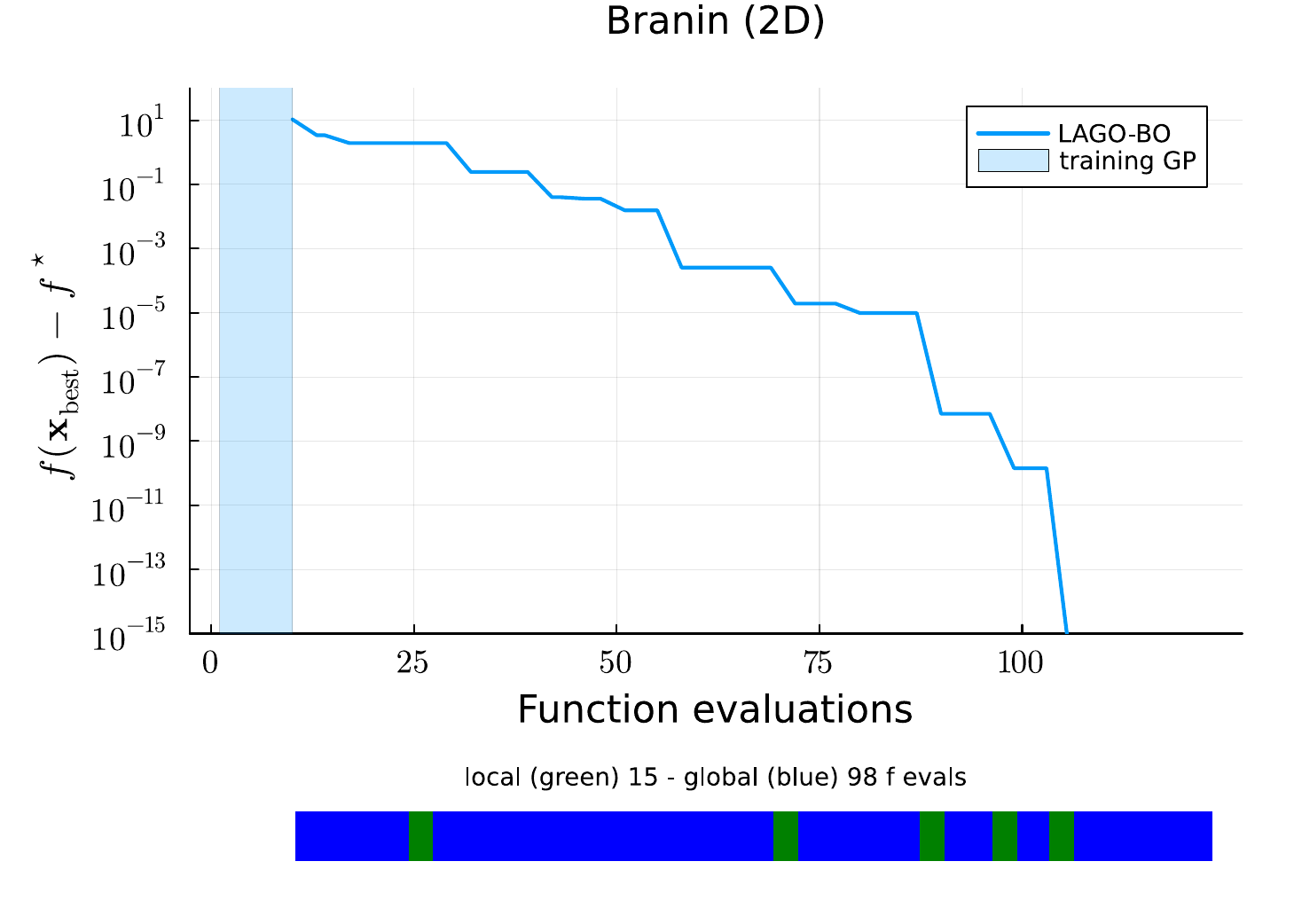}
\includegraphics[width=0.32\textwidth]{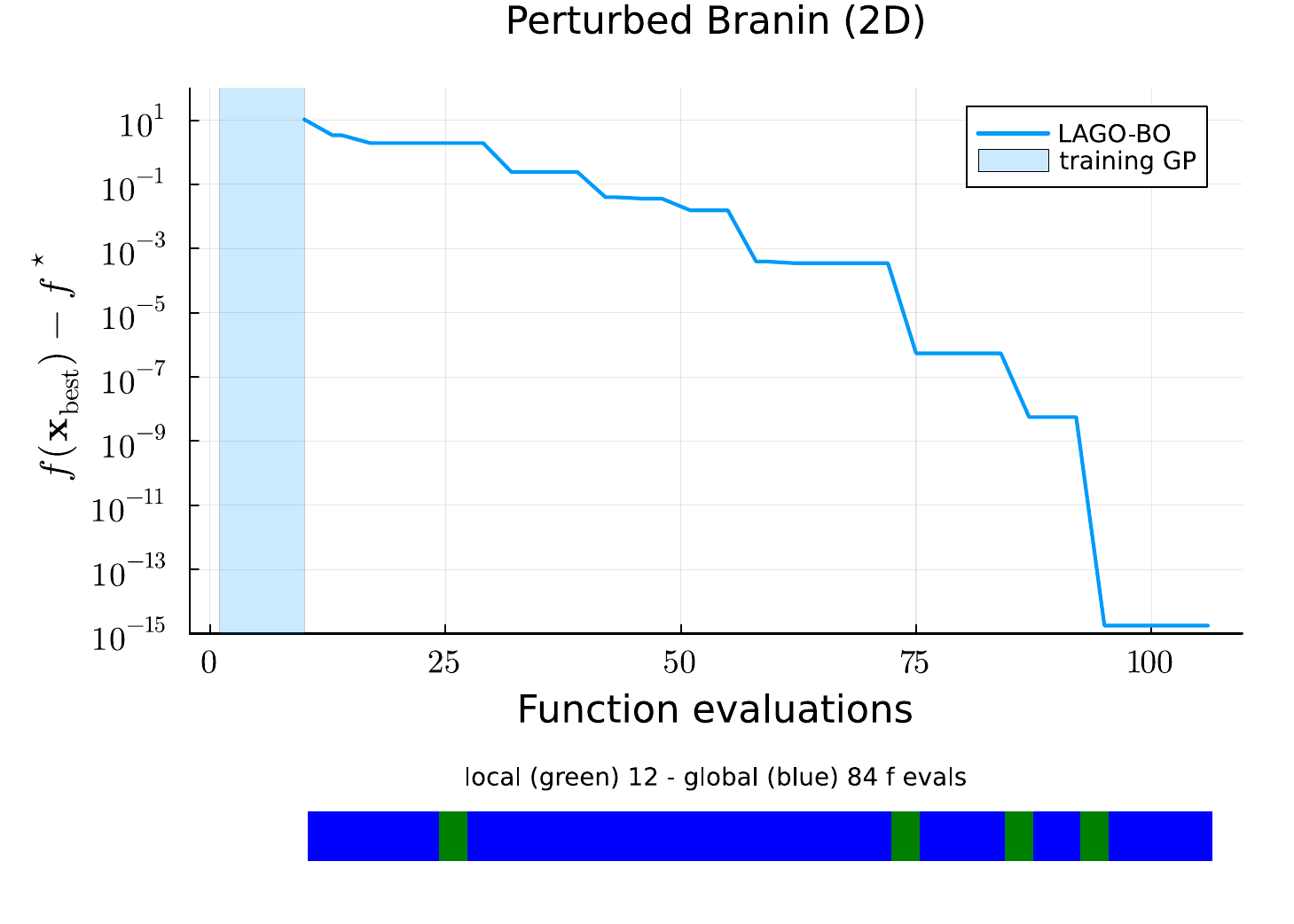}
\includegraphics[width=0.32\textwidth]{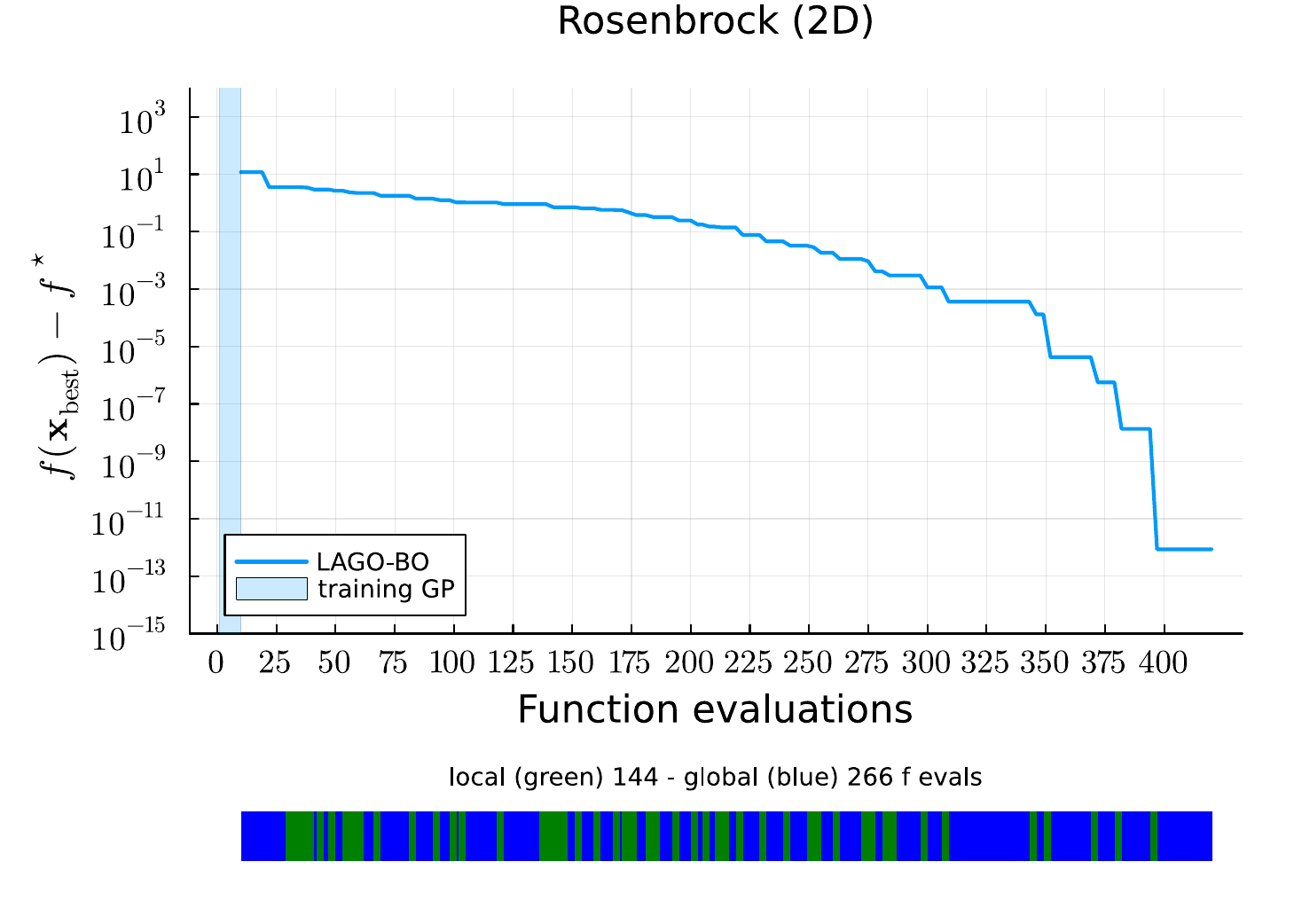}
\includegraphics[width=0.32\textwidth]{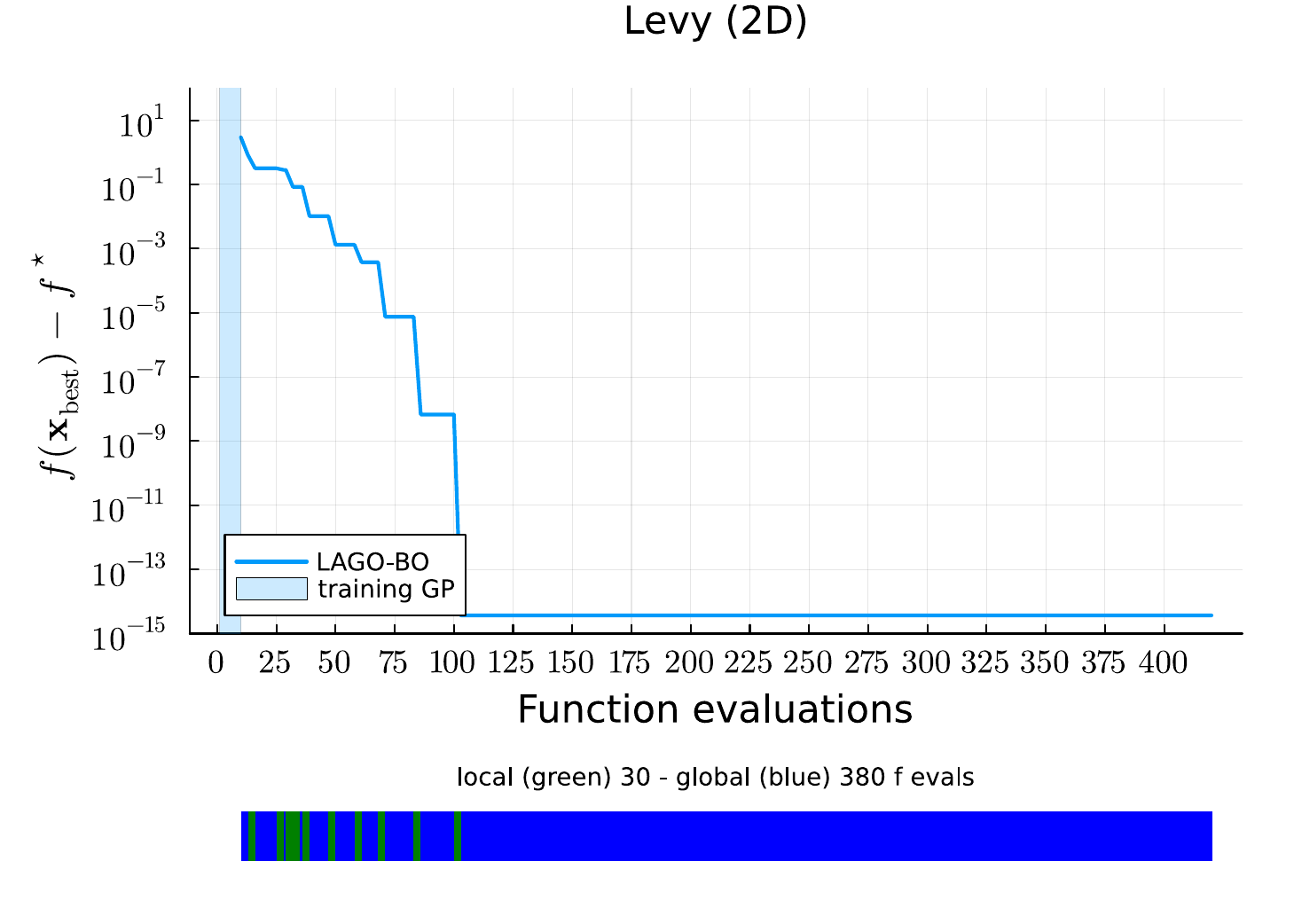}
\includegraphics[width=0.32\textwidth]{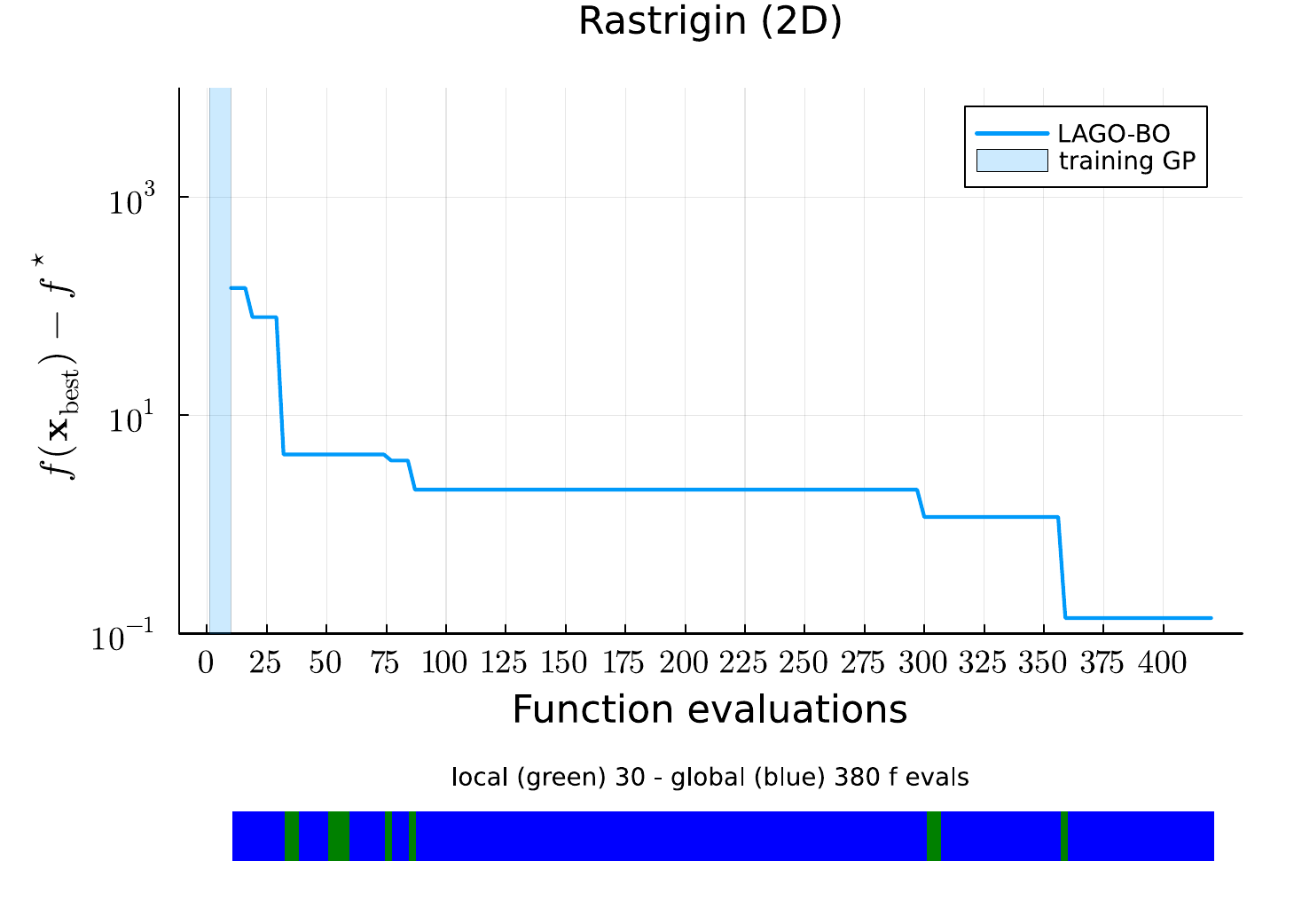}
\includegraphics[width=0.32\textwidth]{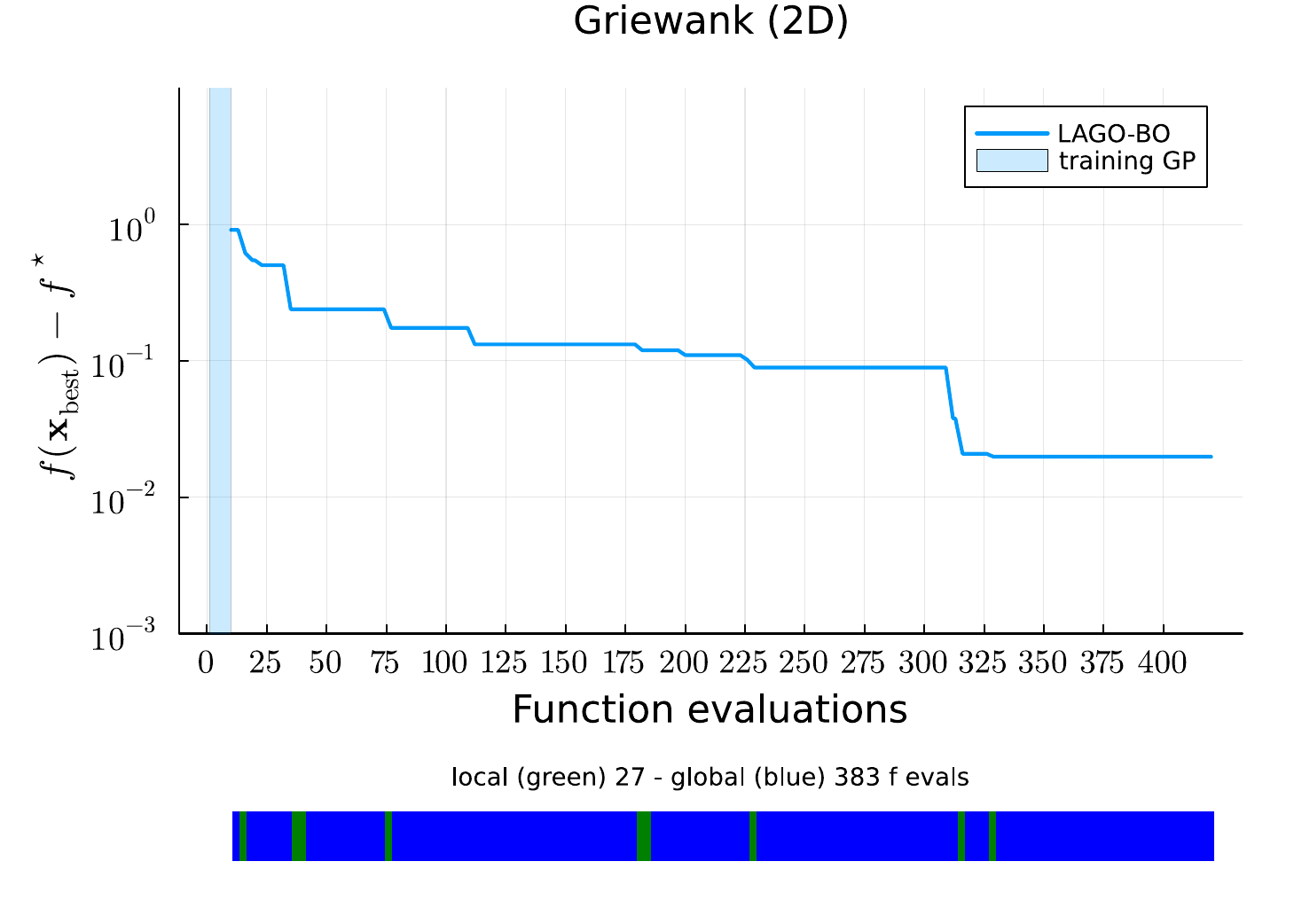}
\includegraphics[width=0.32\textwidth]{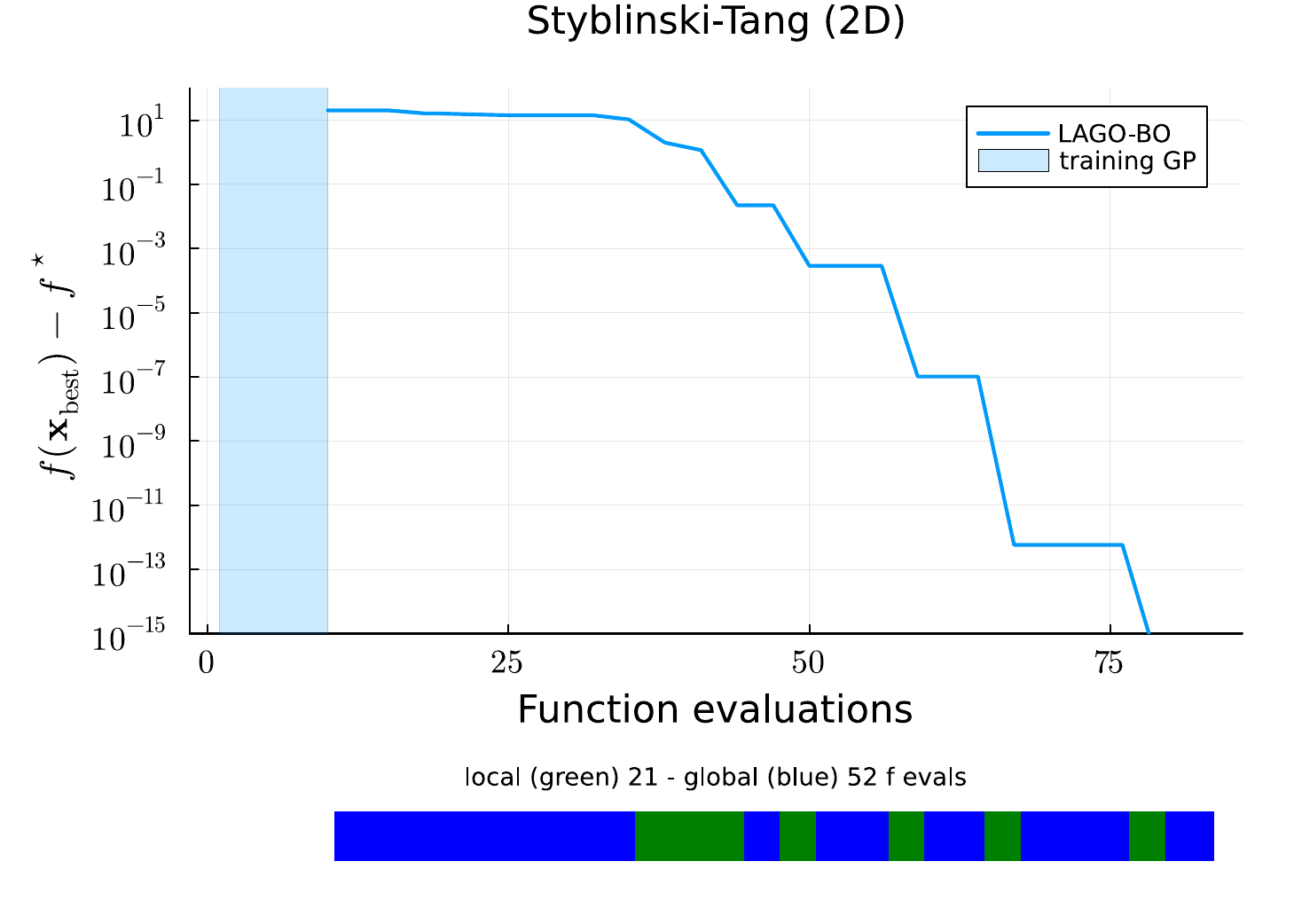}
\includegraphics[width=0.32\textwidth]{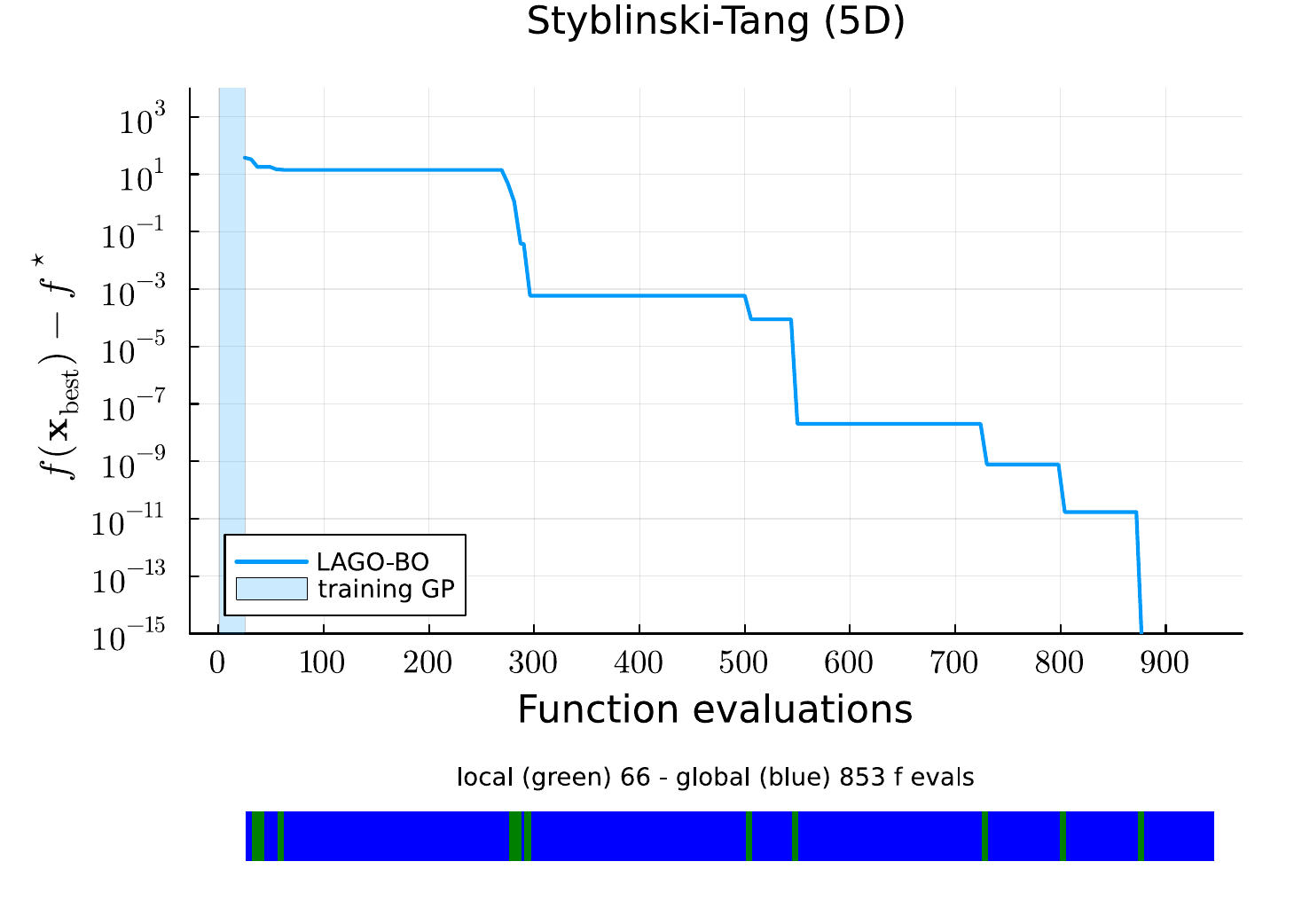}
\includegraphics[width=0.32\textwidth]{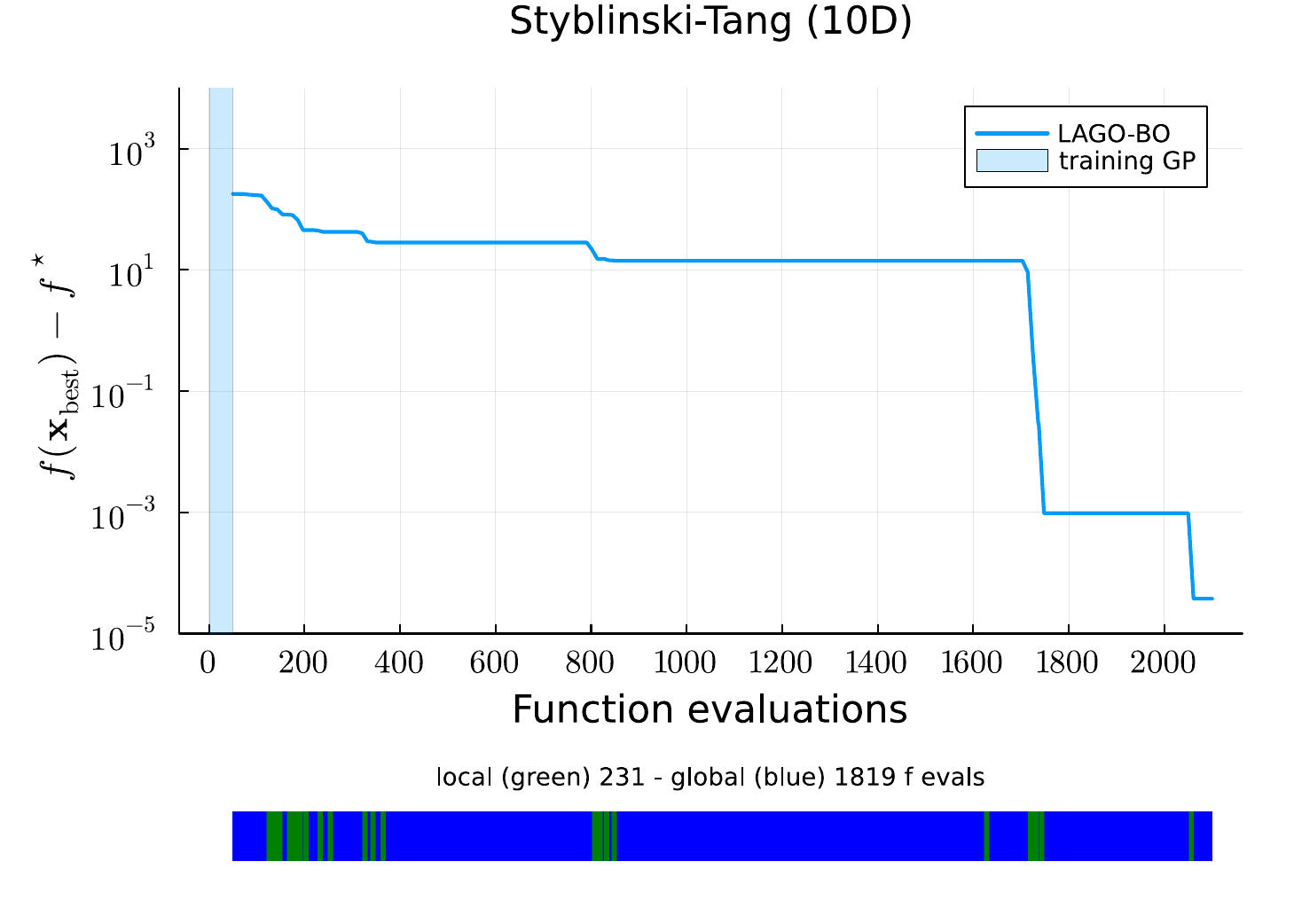}
\caption{\textbf{Local-global interleaving.} LAGO trajectories across benchmarks with local (green) and global (blue) steps. 
LAGO predominantly selects global steps in moderate and highly multimodal settings, while increasing the use of local refinement in convex or weakly multimodal regimes.
}
\label{fig:local_global_interleaving}
\end{figure*}

\begin{table*}
\begin{center}
\caption{Percentage of local steps versus global steps across benchmark problems (reported as mean $\pm$ standard deviation).}
\label{table:local_global_ratio}
\begin{tabular}{lc}
\toprule
Objective & Local steps (\%) \\
\midrule
\multicolumn{2}{l}{\textit{Mostly global behavior: local steps $<5\%$}} \\
Styblinski-Tang (10D)   & 1.0  $\pm$ 0.2 \\
Styblinski-Tang (5D)    & 1.5  $\pm$ 0.3 \\
Griewank (2D)           & 1.9  $\pm$ 0.9 \\
Lévy (2D)               & 2.1  $\pm$ 0.5 \\
Rastrigin (2D)          & 3.9  $\pm$ 2.0 \\
\midrule
\multicolumn{2}{l}{\textit{Increased local refinement: local steps $>5\%$}} \\
\addlinespace[0.25em]
Branin (2D)             & 6.3  $\pm$ 2.0 \\
Perturbed Branin (2D)   & 6.5  $\pm$ 2.0 \\
Styblinski-Tang (2D)    & 11.4 $\pm$ 3.8 \\
Rosenbrock (2D)         & 17.0 $\pm$ 6.2 \\
\bottomrule
\end{tabular}
\end{center}
\end{table*}

\section{Experimental design, LAGO and baselines specifications}\label[appendix]{appendix:experimental_setting}

\paragraph{Experimental design.} We run each algorithm with the same budget of function evaluation ($210d$), over $50$
different initial designs of experiments of size $5d$ using the Latin Hypercube Sampling method \cite{mckay1979comparison}. 
When possible, this initial design is given, but some exceptions apply (see below). 
If the optimization routine requires only one initial starting point, it is selected randomly from the $5d$ initial data points.
If multiple restarts occur, then the next restart point is selected within the remaining points of the initial dataset. 

For synthetic benchmarks, each full gradient evaluation is charged as $d$
function evaluations. This conservative accounting avoids treating
gradient information as essentially free in settings where no adjoint solver or
application-specific gradient-cost model is available. For the PDE-constrained
optimization problem, the adjoint method computes the full gradient with one additional PDE
solve \cite{hinze2008optimization}, so the gradient is charged at one function evaluation cost.

We now discuss the algorithms used in \Cref{sec:experiments} for reproducibility purposes. The code to perform the experiments is available in the supplementary material \cite{van_dieren_2026_20557676}.
\paragraph{LAGO} LAGO(-BO/grad) was implemented in an in-house Julia library. In our experiments, we set $\gamma = 1, \nu = 0.1$, unless stated otherwise.
The nugget is set to $10^{-9}$ and serves as a numerical guarantee for the global GP, the number of acquisition function values needed in Condition \hyperlink{termination_1}{(S$_1$)} is
$N=5$, with $\varepsilon_{\text{T}} = 10^{-12}$. Across all problems, we set $\varepsilon_{\text{step}} = 10^{-7}$.
The kernels chosen are a Matérn $5/2$ for LAGO-BO and a Matérn $7/2$ for LAGO-grad, which respectively have the lowest sufficient smoothness to compute the Hessian of the posterior mean, appearing at the initialization of the trust region.
We set the prior mean $m$ to the mean of the function evaluated at the initial data points. For LAGO-grad, the gradient prior mean is zero, by applying the gradient operator to $m$.
For the MLE hyperparameter tuning of the GP, we leverage the Julia library \texttt{AbstractBayesOpt.jl} \cite{eliott_van_dieren_2026_19607347}, under the version $0.1.5$,
and train the hyperparameters (lengthscale and magnitude) every $10$ iterations. This MLE optimization and prior mean definition are also used for BO and gradBO methods in this work.

The acquisition function is optimized using random sampling ($10^4$ points) outside the TR, followed by L-BFGS initialized from the top candidates ($100$), which follows the procedure in \texttt{AbstractBayesOpt.jl}.
To ensure feasibility, i.e. that the global candidate lies outside the TR, a
penalty based on the distance to the TR center is used to discourage solutions
inside the TR. If violated, we keep the next best candidate outside the TR.

The trust region subproblem \eqref{eq:tr_subproblem} is solved using a
subroutine of \texttt{Optim.jl} \cite{mogensen2018optim} (version $1.13.3$)
implementing the iterative method proposed in \citet[Section~4.3]{nocedal2006numerical}.
The SR1 algorithm has been rewritten, and follows \Cref{alg:TR_algorithm}. As for the trust region acceptance, we set $\eta = 5\times 10^{-4}$, and $r = 10^{-8}$ in \eqref{eq:condition_update_H}.
The initial trust region radius is set as $\min(\ell, \operatorname{diam}(\mathcal{X}))/2$.

\paragraph{BO} Classical BO was used from \texttt{AbstractBayesOpt.jl}, under the version $0.1.5$.
The noise nugget is set to $\sigma^2 = 10^{-9}$, and we use a Matérn $5/2$ kernel. 

For the BO-LOCAL techniques (ACQ, SPLIT) used in
\hyperlink{RQ_1}{\textbf{RQ1}}, the SR1-TR algorithm with the same parameters
as in LAGO are used. The stopping criteria for the SR1-TR to have converged are
\emph{either} $\|\mathbf{s}_k\| < 10^{-7}$, or 
$I_t < 10^{-12}$, which mimics the criteria in LAGO.

\paragraph{gradBO} Gradient-enhanced BO was also taken from the \texttt{AbstractBayesOpt.jl} library, under the version $0.1.5$.
The noise nugget is set to $\sigma^2 = 10^{-9}$, for both the function and its gradient, and we use a Matérn $5/2$ kernel.

\paragraph{L-BFGS} The L-BFGS routine is from \texttt{Optim.jl} ($1.13.3$) with out of the box parameters, and no preconditioners.

\paragraph{TuRBO} The TuRBO algorithm was taken from the original repository of the authors of \citet{eriksson2019scalable}. TuRBO-5 was also tested but gave similar results to TuRBO-1, and was hence not included in the plots.
Two modifications have been made to the original code:
\begin{enumerate}
    \item For TuRBO-1, we initialize the design of experiments with the same initial design as the other algorithms, but this was not possible for TuRBO-5 where all trust regions have their own initial datasets.
    \item The nugget MLE interval has been reduced to $[10^{-9}; 10^{-8}]$ as we deal with noise-free measurements.
\end{enumerate}

\paragraph{TREGO} The TREGO routine followed the tutorial provided in the \texttt{Trieste} \cite{trieste2023} documentation, as it was also provided as a legitimate source in \citet{diouane2023trego}.
We faced some stability issues with $\sigma^2 = 10^{-9}$, which we increased to $10^{-7}$ for all test cases. 
For TREGO, numerical instabilities were observed on the Rosenbrock function despite tuning the nugget parameter; results for this case have seen their nugget increased until success ($\sigma^2 = 10^{-3}$). 

\paragraph{LABCAT} The LABCAT algorithm was taken from the repository of the authors of \citet{visser2025labcat}, using the Python interface.
The initialization step prevents us to provide an original dataset using the Python framework. 
We therefore left the original dataset being initialized using a Latin Hypercube Sampling scheme as prescribed in the original paper \citep{visser2025labcat}.
As a stopping criterion, we put a tolerance on the range of output values of $10^{-12}$.
If this stopping criterion is met, we restart the optimization loop with the remaining budget, until exhausted.

\paragraph{BADS} BADS is using the \texttt{PyBADS} \citep{singh2024pybads} package, with default parameters and multiple restarts.

\paragraph{BLOSSOM} BLOSSOM was taken from the repository of the authors of
\citet{mcleod2018optimization}, and the global regret threshold was set to
$0.1$, as it was seen as the best performing in \citet{mcleod2018optimization}
on the synthetic benchmarks.

\paragraph{Julia-Python interface.} LAGO, BO, gradBO and L-BFGS scripts were written and run in Julia $1.12.2$. 
The remaining algorithms were run in Python $3.11.5$, except for BLOSSOM which was run on a Python $3.6.13$ instance, due to the requirement restrictions of the authors of \citet{mcleod2018optimization}. 
All evaluated points and objective values were stored in NumPy arrays \citep{harris2020array} and imported into Julia for post-processing and visualization.

All experiments were run with fixed random seeds, and the DOE were shared and initialized with a master seed. 

\paragraph{Licenses.} The implementation uses publicly available software under their respective
licenses, including \texttt{AbstractBayesOpt.jl} (MIT), \texttt{Optim.jl} (MIT), \texttt{Trieste}
(Apache-2.0), \texttt{PyBADS} (BSD-3-Clause), \texttt{LABCAT} (MIT), \texttt{TuRBO}
(non-commercial research license from Uber), and \texttt{BLOSSOM} (AGPL-3.0).

\paragraph{Compute resources.}
Experiments on low-dimensional problems ($d \leq 5$) were conducted on a local
workstation equipped with an AMD Ryzen $9$ $9950$X ($16$ cores) and $92$\,GB of RAM.
For more computationally demanding settings, such as certain baselines (e.g., BO, BLOSSOM)
and LAGO-BO on higher-dimensional problems such as Styblinski-Tang in $10$D, experiments were executed
on a shared CPU cluster with Intel Xeon Platinum processors.

The computational cost varies significantly across methods and problem
dimensions. For the proposed LAGO framework, a single run (one seed) typically
completes within a few minutes in low-dimensional settings, and remains under
30 minutes for dimensions up to $d = 5$. In higher-dimensional settings (e.g.,
$d = 10$), runtimes increase to roughly a day due to (i) harder optimization of the acquisition function, and (ii) the larger evaluation budget
($210d$), requiring the use of a cluster to parallelize over the seeds. 
Some baseline methods require significantly longer runtimes, with
certain configurations for $d \geq 5$ taking multiple days per seed.

Overall, the experiments do not require specialized hardware (e.g., GPUs) and
can be reproduced on CPU-based systems, although total runtime depends
on the method and problem dimensionality.

\end{document}